\documentclass{article}

\usepackage{Arxiv}

\usepackage{tabularx}
\usepackage{graphicx}  %support the \includegraphics command and options
\usepackage{subfigure} % make it possible to include more than one captioned figure/table in a single float

\usepackage{array}  

\usepackage{amssymb}
\usepackage{amsthm}
\usepackage{amsmath} 
\usepackage{comment}
\usepackage{xr}
\usepackage{algorithm}
\usepackage[noend]{algpseudocode}
\usepackage{xcolor}
\usepackage{natbib}
\usepackage{bbm}

\newtheorem{theorem}{Theorem}
\newtheorem{assumption}{Assumption}
\newtheorem{lemma}{lemma}

\newtheorem{proposition}{proposition}
\newtheorem{Corollary}{Corollary}

\newtheorem{remark}{Remark}

\newcommand{\RN}[1]{%
  \textup{\lowercase\expandafter{\romannumeral#1}}%
}

\usepackage[utf8]{inputenc} % allow utf-8 input
\usepackage[T1]{fontenc}    % use 8-bit T1 fonts
\usepackage{hyperref}       % hyperlinks
\usepackage{url}            % simple URL typesetting
\usepackage{booktabs}       % professional-quality tables
\usepackage{amsfonts}       % blackboard math symbols
\usepackage{nicefrac}       % compact symbols for 1/2, etc.
\usepackage{microtype}      % microtypography
\usepackage{lipsum}
\usepackage{fancyhdr}       % header
\usepackage{graphicx}       % graphics
\graphicspath{{media/}}     % organize your images and other figures under media/ folder

\title{Transfer Learning for High-dimensional Reduced Rank Time Series Models
}

\author{
  Mingliang Ma \\
  University of Florida \\
  \texttt{mml@mail.ustc.edu.cn} \\
   \And
  Abolfazl Safikhani \\
  George Mason University\\
  \texttt{asafikha@gmu.edu} \\
}

\begin{document}
\maketitle

\begin{abstract}

The objective of transfer learning is to enhance estimation and inference in a target data by leveraging knowledge gained from additional sources. Recent studies have explored transfer learning for independent observations in complex, high-dimensional models assuming sparsity, yet research on time series models remains limited. Our focus is on transfer learning for sequences of observations with temporal dependencies and a more intricate model parameter structure. Specifically, we investigate the vector autoregressive model (VAR), a widely recognized model for time series data, where the transition matrix can be deconstructed into a combination of a sparse matrix and a low-rank one. We propose a new transfer learning algorithm tailored for estimating high-dimensional VAR models characterized by low-rank and sparse structures. Additionally, we present a novel approach for selecting informative observations from auxiliary datasets. Theoretical guarantees are established, encompassing model parameter consistency, informative set selection, and the asymptotic distribution of estimators under mild conditions. The latter facilitates the construction of entry-wise confidence intervals for model parameters. Finally, we demonstrate the empirical efficacy of our methodologies through both simulated and real-world datasets.

% The objective of transfer learning is to enhance estimation and inference in a target data by leveraging knowledge gained from additional sources. Recent works have investigated the transfer learning problem for independent observations in high-dimensional models based on the sparsity assumption while results on time series models are rather limited. We study the transfer learning problem for temporally dependent observations with a more complex model parameter structure. Specifically, we focus on vector autoregressive model (VAR), a well-known time series model, whose transition matrix can be decomposed into a sparse matrix plus a low-rank one. We propose a novel transfer learning algorithm tailored for estimating high-dimensional VAR models with both low-rank and sparse structures. Further, a new method to select informative observations from auxiliary data sets is proposed. Theoretical guarantees are provided including model parameter consistency, informative set selection, and asymptotic distribution of proposed estimators under mild conditions. The latter allows for the construction of entry-wise confidence intervals for model parameters. Finally, the empirical performance of the proposed methodologies are demonstrated through both simulated data and a real data set.
 
\end{abstract}

\section{INTRODUCTION}\label{sec1}

\vspace{-0.35cm}

% There is a mature literature on imposing low rank plus sparse for VAR model \cite{basu2019low}. This work considers both group sparsity and entry-wise sparsity for sparse component with the usage of group lasso and traditional lasso penalty. The constrain on low-rank component

In many applications, Vector Autoregressive (VAR) model provides a principled framework for a wide range of tasks, including analyzing speech signal \citep{juang1985mixture,shannon2012autoregressive}, investigating causality between economics variables \citep{granger1969investigating}, reconstructing gene
regulatory interactions \citep{michailidis2013autoregressive}, extracting classifiable features for neuroscience data \citep{anderson1998multivariate}, and finding connectivity between brain regions \citep{van2010exploring}. A simple form of VAR model is $X_t = B X_{t-1} + \epsilon_t$ where $B$ is a $p \times p$ transition matrix and $\epsilon_t$ is the error term at time $t$. Sparse transition matrices are among popular choices considered in high-dimension regime (the dimension of variables is significantly greater than the number of observations). To get a sparse estimation, there are plenty of studies using penalty for the transition matrix, such as the popular $\ell_1$ penalty (lasso), group lasso type penalties
employed in \cite{melnyk2016estimating} and non-convex penalties akin to
a square-root lasso \citep{jiang2018sparse}. A low-rank transition matrix is assumed in hidden factor model \citep{bai2003inferential}. In this scenario, $B$ can be written as the product of two rank-r ($r \ll p$) matrices $U$, $V$, i.e $B = UV^{'}$ so that the resulting
model specification of the original $p$ time series is expressed as linear combinations $Z_t = V^{'}X_t$ of the original ones and $U$ specifies the dependence between $X_t$ and $Z_t$; namely $X_t = U Z_{t-1} + \epsilon_t$. $Z_t$ is in fact the hidden factor with dimension $r$, driving the evolution of the process (as a simple example, $X_t$ can be the GDP of each country, and the hidden factors can reflect the general state of the economics at a continent scale). Recent works have generalized the mentioned model by assuming a low-rank
plus sparse structure for the transition matrix, i.e. $B = L+S$ where $L$ is the low-rank part and $S$ is the sparse component \citep{basu2019low,bai2020multiple}. This decomposition is a natural assumption for dynamic imaging since the $L$ and $S$ components can represent the background and the dynamic foreground, respectively \citep{otazo2015low}. To include the low-rank structure, the estimation is achieved by imposing a nuclear penalty. However the algorithm requires relatively large amount of data for training and testing purposes to guarantee a consistent estimation, i.e. number of observations $N$ should have at least the same order as the dimension $p$. An important question then arises: when insufficient data is provided, is there a way to estimate $L$ and $S$ with high accuracy? A nature idea of solving the data shortage is leveraging knowledge from additional sources whose observations have similar behavior as original ones. This is where transfer learning comes in. There
has been a growing body of literature on transfer learning under high-dimension regime with sparse transition matrices. For example, \cite{cai2021transfer} studies transfer learning in the context of nonparametric classification, while \cite{tian2022transfer} provides theoretical analysis of transfer learning algorithm, a transferable source detection approach, as well as constructing confidence intervals for model parameters for generalized linear models. Also, a low-rank transfer learning algorithm is proposed by \cite{tian2023learning}, while they mainly focus on low-rank component instead of a decomposed structure and/or the model parameters are $p$-dimensional vectors as opposed to squared matrices of dimension $p$ in our paper. Other related works include investigating transfer learning in large-scale Gaussian
graphical models with false discovery rate control \citep{li2022transfer}, and leveraging big data information via weighted estimator for logistic regression \citep{zheng2019data}; see also related works on hypothesis transfer and meta-learning \citep{kuzborskij2013stability,wang2016nonparametric,kuzborskij2017fast,aghbalou2023hypothesis,lin2024smoothness,tripuraneni2021provable}. To the best of our knowledge, theoretical analysis of transfer learning for VAR models has not been investigated in the literature. The main goal of this paper is to bridge this gap. To  that end, in this work, we focus on VAR models with low-rank plus sparse structure for the transition matrix and propose a transfer learning algorithm to improve the estimation and inference for target model.

Suppose that we have $K+1$ groups of observations, $X_t^{(i)} = B_i^{'} X_{t-1}^{(i-1)} + \epsilon_t^{(i)}, 0\leq i \leq K$ while the first one is the target group. Hence, $B_0$ is the transition matrix of the target model while $B_i, i \geq 1$ is the transition matrix of the $i$-th auxiliary model. We assume that $B_i$ can be decomposed as $L + S_i$. In other words, all models have a common low-rank component, which can be interpreted as a background information, while $S_i$ varies across different models. For example, if the $i$-th group is collected in chronological order, $S_i$ captures a dynamic evolution across time and our model can be applied in a dynamic imaging problem as mentioned before (see more details in Section~\ref{sec5}).

The main goal of this paper is estimating sparse component of the target model $S_0$ with high accuracy while also estimating the shared low-rank structure $L$. To address this problem, we propose a transfer learning algorithm by using observations from those informative models whose $S_i$ is close to $S_0$. To be more specific, our algorithm comprises of two main steps. First, since $L$ is a common part over all groups, we merge all observations from target and auxiliary sets to estimate the low-rank component denoted by $\widehat{L}$. With more observations, our theory verifies that the first step derives a more accurate estimation of $L$ compared with the result of estimating $L$ with only the target data (Theorem~\ref{th1}). In the second step, we remove the effect of the low-rank component by defining $Y_t^{(i)} := X_t^{(i)} - \widehat{L}^{'}X_{t-1}^{(i)}$ and then apply a transfer learning algorithm for the remaining sparse part. It should be noticed that different from merging all collected data, observations from informative models are merged in the transfer learning step. Taking into account observations from non-informative models whose sparse component $S_i$ deviates far away from $S_0$ could damage the estimation performance, which is a phenomenon often called as negative transfer \citep{torrey2010transfer}. Since it is crucial to have a correct informative set, we also propose a novel algorithm to select informative groups from all auxiliary groups. Under some mild condition, it is shown that informative groups and non-informative groups can be separated perfectly with high probability using the proposed algorithm (Theorem~\ref{th2}). Finally, inference for model parameters is performed by adding an additional debiasing step.

In summary, the main contributions of this work include: (1) propose a new algorithm to perform transfer learning for VAR models with low-rank plus sparse structure with theoretical guarantees; (2) develop a novel algorithm for selecting informative sets from all auxiliary groups; (3) constructing confidence intervals for model parameters. These developments came with addressing important challenges due to the complex model structure and temporal dependence. More specifically, (a) due to existence of a shared low rank component, an additional step to the transfer learning algorithm had to be added to estimate that shared piece consistently; (b) inclusion of temporal dependence makes the theoretical development more complicated. We need to design new (A1) Restricted Strong Convexity and (A2) Deviation Bound Conditions. The original version of such conditions are written for single group data while we deal with multiple groups in transfer learning scenarios. As such, these conditions had to be adjusted to the multi-group cases where models from different groups are not same but only similar. Further, these two conditions had to be verified for the proposed model. This is done successfully in the paper by listing appropriate sufficient conditions under which they are satisfied. Certain type of Hanson–Wright inequality is applied to prove that (A1) and (A2) hold for VAR models; (c) certain parts of the algorithm need to be adapted to respect the temporal dependence. For example, the last step on debiasing the estimate to get confidence intervals had to be adjusted appropriately using an ``online debiasing" method so that Martingale-type CLT can be used to derive the asymptotic distribution. A simple debiasing step is not appropriate here since the usual CLT will not work for such models. It is worth noting that developing inferential framework is not only for theoretical purposes, but can also be utilized in practice, i.e. in real data analysis. For example, in Figure~\ref{fig3}, we illustrate how the inferential framework helps in finding pixels with potential changes in the surveillance video data, i.e. locating potential root cause of changes in the video data by checking significantly different from zero parameter estimates. Using invalid confidence intervals can potentially ruin this analysis and either select additional unchanged pixels by mistake or fail to select important/changed pixels (see more details in Section~\ref{sec5}).

The remainder of the paper is organized as follows. In Section~\ref{sec2}, we present the modeling framework, introduce some prior knowledge for VAR models and background for transfer learning algorithm. In Section~\ref{sec3}, we introduce the two-step transfer learning algorithm and establish its theoretical properties, introduce an strategy of selecting informative set as well as making inference for sparse component. Section~\ref{sec4} presents our simulation results while a real data is analyzed using our algorithm in Section~\ref{sec5}. Finally, some concluding remarks are summarized in Section~\ref{sec6}.

\textbf{Notation}:  For a $p \times p$ matrix $A$, $\Lambda_{min}(A)$, $\Lambda_{max}(A)$ denote the smallest and largest eigenvalues of $A$, respectively. $\Vert A\Vert_2$ denotes the operator norm for matrix $A$, i.e.  
$\Vert A\Vert_2=\sqrt{\Lambda_{max}(A^{'}A)}$. $\Vert A\Vert_{\infty}$ denotes the infinity norm $\Vert A\Vert_{\infty} = \mathrm{max}_{i,j}|A_{ij}|$. $\Vert A \Vert_F$ denotes the Frobenius norm, i.e $\Vert A \Vert_F = \sqrt{\sum_{i,j} A_{ij}^2}$. $\Vert A \Vert_{*}$  denote the nuclear norm, i.e. $\sum_{j=1}^p \sigma_j (A)$, the sum of the singular values of a matrix. $A^{*}$ denotes the conjugate transpose of a matrix $A$. $\Vert A\Vert_0$ denotes the number of non-zero entry in $A$. $\Vert A \Vert_1$ denotes the $\ell_1$ norm of $A$, i.e $\Vert A \Vert_1 = \sum_{i,j} |A_{i,j}|$. For a vector $u\in\mathbb{R}^p$, 
$\Vert u\Vert_2=\sqrt{\sum u_i^2}$ and $\Vert u\Vert_1=\sum |u_i|$.
For two sequences $(a_t)_{t\geq 1}$ and $(b_t)_{t\geq 1}$, we write $a_t \lesssim b_t$ if there exists a constant $c \geq 1$ such that $a_t \leq c b_t$ for all $t$. If both $a_t \lesssim b_t$ and $b_t \gtrsim a_t$, we write $a_t \asymp b_t$. Also, $a_t = o (b_t)$ implies $a_t/b_t \rightarrow 0$ as $t \rightarrow \infty$; $a_t = O (b_t)$ implies $a_t/b_t < \infty$; $a_t = \Omega (b_t)$ implies $b_t/a_t \rightarrow 0$. For two real numbers $a$ and $b$, $a\vee b$ denotes $\mathrm{max}\{a,b\}$ and $a\wedge b$ denotes $\mathrm{min}\{a,b\} $. Finally, let $e_j$ be a vector such that its $j$-th element is 1 and all other elements
are zero. 

\vspace{-0.35cm}

\section{MODEL FORMULATION}\label{sec2}
\vspace{-0.25cm}
% \subsection{Model Formulation}
% \vspace{-0.25cm}
As mentioned, we focus on transfer learning for the VAR model. To that end, assume that we observe samples from the target
model and K other auxiliary models. Each model has the following expression

\vspace{-0.75cm}

% \epsilon_t^{(i)} \sim N(0, \Sigma^{(i)})

\begin{equation}\label{sec1_eq1}
\begin{split}
&X_t^{(i)} = B_i^{'} X_{t-1}^{(i)} + \epsilon_t^{(i)},\ B_i = L + S_i, \ rank(L) = r,
\end{split}
\end{equation}
where $X_t^{(i)}$ is the p dimensional vector of observed time series at time $t$ for the $i$-th group. Observations from different models are assumed to be independent with each other. $L$ represents the low-rank matrix while $S_i$ represents the sparse matrix and the transition matrix $B_i$ is
low-rank plus sparse. We assume that the number of none-zero entries in $S_i$ is $s$ with $s \ll p^2$. We further assume that the rank of low-rank component $L$ is far less than $p$, $r \ll p$. The length of $i$-th group is defined by $n_i$. Let $N:= \sum_{i=0}^Kn_i$ be the overall sample size. We can rewrite \eqref{sec1_eq1} as

\vspace{-0.35cm}

\begin{equation}\label{sec1_eq2}
\underbrace{
\begin{pmatrix}
(X_{n_i}^{(i)})^{'}\\
\vdots \\
(X_1^{(i)})^{'}
\end{pmatrix}
}_{\mathcal{Y}_i}
=
\underbrace{
\begin{pmatrix}
(X_{n_i-1}^{(i)})^{'}  \\
\vdots            \\
(X_{0}^{(i)})^{'}  
\end{pmatrix}
}_{\mathcal{X}_i}
B_i
+
\underbrace{
\begin{pmatrix}
(\epsilon_{n_i}^{(i)})^{'} \\
\vdots          \\
(\epsilon_1^{(i)})^{'}
\end{pmatrix}
}_{\mathcal{E}_i}.
\end{equation}

\vspace{-0.35cm}

This model is a  generalization of standard sparse VAR model and is able to deal with the setting where there is an invariant cross-autocorrelation structure $L$ across target groups and auxiliary groups. $S_i$ captures the additional  cross-sectional autocorrelation
structure for each group. \cite{basu2019low} consider the low-rank plus structured sparse model in one-group case and propose an algorithm to estimate $L$ and $S$ accurately. Our goal is to improve the estimation accuracy for sparse component of the target group $S_0$ and the shared low-rank $L$ provided that more information from auxiliary groups are available.

In the context of transfer learning, to improve the estimation accuracy, we need to select useful observations that have similar behavior as target data. Those observations from auxiliary models that have similar behavior as target data are named as informative observations and their corresponding models are named as informative groups. In this work, similarity is characterized by the difference between the sparse component of transition matrices $S_i$, i.e $\delta^{k} := S_k - S_0$. A small $\delta^{k}$ implies a high level of similarity. When $\delta^{k}$ is relatively small, taking into account observations from $k-$th group via transfer learning could improve the estimation accuracy of transition matrices. On the contrary, incorporating information from non-informative groups will damage the transfer learning performance, which is known as negative transfer \citep{zhang2022survey}. Informative groups are mathematically defined as $\mathcal{A} = \{ k\in\{1,2,\cdots,K\} : \Vert \delta^{(k)} \Vert_1 \leq h\}$ where $h$ is some positive number. We also define $\mathcal{A}_0 := \{0\} \cup\mathcal{A}$ to simplify the notation. We use $n_\mathcal{A}$ to denote the sample size of informative sets, i.e $n_\mathcal{A} := \sum_{i \in \mathcal{A}}n_i$, and similarly $n_{\mathcal{A}_0} := \sum_{i \in \mathcal{A}_0}n_i$. We also define $\mathcal{X}_{\mathcal{A}_0}$ as the design matrix constructed by $\mathcal{X}_i$, $i \in \mathcal{A}_0$.

\section{ESTIMATION PROCEDURE AND THEORETICAL RESULT}\label{sec3}
In this section, we introduce the proposed transfer learning algorithm and present theoretical results. As an overview,
our estimation procedure comprises of two steps. First, we estimate low-rank matrix given all observations. Since all models share one low-rank matrix, considering all observations could improve estimation accuracy. Then we focus on informative set and apply transfer learning to estimate sparse matrix of the target model. Our algorithm is shown in Algorithm~\ref{algorithm1}.

\begin{algorithm}
\caption{: Transfer learning for sparse component}\label{algorithm1}
\hspace*{\algorithmicindent} \textbf{Input} : observations from target model and auxiliary model $\{X_t^{(i)}\}, i=0,1,\cdots K$; penalty parameters $\lambda_\beta$,$\lambda_\delta$; informative set $\mathcal{A}$ and some $\theta>0$.\\
\hspace*{\algorithmicindent} \textbf{Output} : coefficient estimator for the target model $\widehat{\beta}$.

\begin{algorithmic}[0]
\State\textbf{Step 1} 
Let $\Omega := \{ L \in \mathbb{R}^{p\times p}: \Vert L \Vert_{\infty} \leq \theta \}$
\begin{equation}\label{alg_eq1}
\begin{split}
\widehat{L},\widehat{S}_1,\cdots,\widehat{S}_K 
&= 
\mathop{\mathrm{argmax}}_{\substack{{L, S_1,\cdots,S_K}\\ L \in\Omega}} \sum_i \frac{1}{N}\Vert \mathcal{Y}_i- \mathcal{X}_i(L+S_i) \Vert_F^2\\
&+
\lambda \Vert L \Vert_{*} + \sum_i \frac{1}{\sqrt{N}}\mu_i\Vert S_i\Vert_1
\end{split}
\end{equation}
\State\textbf{Step 2} 
\begin{equation}\label{alg_eq2}
    \tilde{S} = \mathop{\mathrm{argmin}}_{S \in \mathbb{R}^{p^2}} \sum_{i \in \mathcal{A}}\frac{1}{2 n_{\mathcal{A}_0}}\Vert \mathcal{Y}_i - \mathcal{X}_i(\widehat{L} + S) \Vert_F^2 + \lambda_\beta \Vert S \Vert_1
\end{equation}

% \State\textbf{Step 3}

$\widehat{S}_{tran} := \tilde{S} - \tilde{\delta}$, where
\begin{equation}\label{alg_eq3}
\tilde{\delta} :=\mathop{\mathrm{argmin}}_{\delta\in \mathbb{R}^{p^2}}\{ \frac{1}{2n_0}\Vert \mathcal{Y}_0 - (\widehat{L} + \tilde{S} + \delta) \mathcal{X}_0\Vert_F^2 + \lambda_{\delta}\Vert \delta \Vert_1 \}
\end{equation}\\    
\end{algorithmic}
\end{algorithm}

\vspace{-0.35cm}

\subsection{Step 1: Estimating The Low-rank Component}
The first step is a low-rank plus sparse decomposition problem takes the form of \eqref{alg_eq1}
\begin{comment}
\begin{equation}\label{sec3_eq1}
\begin{split}
\widehat{L}, \widehat{S}_0, \cdots, \widehat{S}_K
&=
\mathop{\mathrm{argmax}}_{\substack{{L, S_1,\cdots,S_K}\\ L \in\Omega}} l(L, S_0,\cdots, S_K)\\
l(L, S_0,\cdots, S_K)
&=
\sum_{i = 0}^K \frac{1}{N}\Vert \mathcal{Y}_i- \mathcal{X}_i(L+S_i) \Vert_F^2
+
\lambda \Vert L \Vert_{*} + \sum_i \frac{1}{\sqrt{N}}\mu_i\Vert S_i\Vert_1
\end{split}
\end{equation}
\end{comment}
, where $\lambda$ and $\mu_i$ are non-negative tuning parameters
controlling the regularizations of low-rank and sparse parts. The parameters $\theta$ controls the degree of
non-identifiability of decomposition of the low-rank and sparse matrices. For example, if the sparse component $\{S_i\}_{1\leq i\leq K}$ is also low-rank and low-rank component $L$ is sparse itself, there will be multiple choices of decomposition $L+S_i$ without imposing any further constraints. Larger values of $\theta$ provide sparser estimates of sparse component and allow both sparse and low-rank components to be absorbed in $\widehat{L}$. A smaller value of $\theta$, on the other hand, tends to produce a matrix $L$ with smaller rank and pushes both low-rank and sparse components to be absorbed in $\{\widehat{S}_i\}_{1\leq i\leq K}$. We refer to \cite{agarwal2012noisy} for more details about this identifiability issue. In the low-rank plus sparse regime, consistent estimation relies on the following assumption:

\vspace{-0.35cm}

\begin{itemize}
\item[(A1)] Restricted Strong Convexity (RSC):
There exist $\alpha > 0$ and $\tau \geq \tau^{'} > 0$ such that for all $ \Delta \in \mathbb{R}^{p \times p}$.
\begin{equation*}
\begin{split}
&\frac{1}{2N}\sum_i \Vert \mathcal{X}_i\Delta \Vert_F^2
\geq
\alpha\Vert \Delta \Vert_F^2 - \tau^{'} \Phi^2(\Delta),\\
&\frac{1}{2n_i} \Vert \mathcal{X}_i\Delta \Vert_F^2
\geq
\alpha\Vert \Delta \Vert_F^2 - \tau \Vert \Delta \Vert_1^2
\end{split}
\end{equation*}
where $\Phi(\Delta) = \mathop{\mathrm{\inf}}_{L + S = \Delta} \{ \Vert L \Vert_{*}  + \frac{\mu}{\lambda} \Vert S \Vert_1\}$, $\mu = \mathrm{max}\{\mu_0,\cdots,\mu_K\}$ and $\tau =O (\frac{\log p}{\mathrm{max}_i  n_i}).$
% \vspace{-0.35cm}
\item[(A2)] Deviation Bound Condition (DBC): There exists a constant
$\phi$ depending on the model parameters $B_0,\cdots,B_K$ and $\Sigma_0,\cdots,\Sigma_K$ such that
\begin{equation*}
\begin{split}
&\Vert \frac{1}{N}\sum_{i=0}^K \mathcal{X}_i^{'}\mathcal{E}_i
\Vert_2 \leq \phi \sqrt{\frac{p}{N}}\\
&\mathop{\mathrm{max}}_{0\leq i\leq K} \frac{1}{N}\Vert \mathcal{X}_i^{'}\mathcal{E}_i \Vert_{\infty} \leq \phi \sqrt{\frac{\log p}{N}}
\end{split}    
\end{equation*}
\end{itemize}

\vspace{-0.35cm}

RSC and DBC are basic assumptions for low-rank plus sparse models \citep{basu2019low}. We show that all stable VAR models satisfy these assumptions with high probability in Proposition~\ref{prop1} in the Appendix. Applying the above deviation bounds, we obtain the consistency result for both sparse and low-rank parts.

\begin{theorem}\label{th1}
Suppose that the low-rank matrix $L$ has rank at most r, while the sparse matrix $S_i$ has at most s nonzero entries for $i\in \{0,1,\cdots,K\}$. Assume that $p = O(N)$ and $\log p =O(n_i) $. Let $\mu_i = 2c_0\phi\sqrt{\frac{\log p}{N}} +\theta$, $\lambda = 2c_0\phi \sqrt{\frac{p}{N}} $ and $\theta = o(\sqrt{\frac{p}{N}})$. Under Conditions (A1) and (A2), the estimator of \eqref{alg_eq1} satisfies $\Vert L - \widehat{L}\Vert_F^2  +\sum_{i=0}^{K} \frac{n_i}{N}\Vert S_i - \widehat{S}_i  \Vert_F^2 \lesssim s\frac{\log p}{N} + r\frac{p}{N}.$
\end{theorem}

\begin{remark}
The convergence rate is a combination of low-rank component and sparse component. The result implies that the upper bound on low-rank component $\Vert L - \widehat{L}\Vert_F^2$ is $\frac{s\log p + rp}{N}$ and the upper bound on sparse component $\Vert S_i - \widehat{S}_i\Vert_F^2$ is $\frac{s\log p + rp}{n_i}$. When there is no auxiliary observations, the upper bound becomes $\Vert L - \widehat{L}\Vert_F^2  + \Vert S_0 - \widehat{S}_0  \Vert_F^2 \lesssim s\frac{\log p}{n_0} + r\frac{p}{n_0}$. In this case, the upper bound on $\Vert L - \widehat{L}\Vert_F^2$ is $\frac{s\log p + rp}{n_0}$ and the upper bound on $\Vert S_0 - \widehat{S}_0\Vert_F^2$ is $\frac{s\log p + rp}{n_0}$. Comparing the conclusions of these two scenarios, we can see that auxiliary observations help improve the estimation accuracy of $L$ but not for $S_0$. An additional step is required to improve the estimation for the sparse components, see Theorem~\ref{th2} and the discussion after the theorem for more details.
\end{remark}

Theorem \ref{th1} provides estimation consistency for the first step. Since all auxiliary models have a common low-rank component, merging all observations is helpful for estimating $L$. We next show that a better estimator of $S_0$ could be obtained from a transfer learning algorithm utilizing the better estimator $\widehat{L}$ we found in the first step.

\vspace{-0.35cm}

\subsection{Step 2: Transfer Learning for Sparse Component}

% We denote the estimator of $S_0$ by $\widehat{S}_{tran}$,

% \begin{align}
% &\tilde{S} = \mathop{\mathrm{argmin}}_{S \in \mathbb{R}^{p\times p}} \sum_{i \in \mathcal{A}_0}\frac{1}{2 n_{\mathcal{A}_0}}\Vert \mathcal{Y}_i - \mathcal{X}_i(\widehat{L} + S) \Vert_F^2 + \mu \Vert S \Vert_1 \label{sec3_eq2}\\
% &\widehat{S}_{tran} = \tilde{S} - \tilde{\delta}, \quad
% \tilde{\delta} =\mathop{\mathrm{argmin}}_{\delta\in \mathbb{R}^{p\times p}}\{ \frac{1}{2n_0}\Vert \mathcal{Y}_0 - (\widehat{L} + \tilde{S} + \delta) \mathcal{X}_0\Vert_F^2 + \lambda_{\delta}\Vert \delta \Vert_1 \}\label{sec3_eq3}
% \end{align}

In the second step, we estimate sparse component of the target model $S_0$ via the transfer learning method summarized in equations \eqref{alg_eq2} and \eqref{alg_eq3}. $\tilde{S}$ is an intermediate estimator in transfer learning method calculated by merging target observations and informative observations as source data. This estimator will slightly deviates from $S_0$ due to the usage of informative observations. We show that $\tilde{S}$ converge to $\Bar{S} := (\sum_{i \in \mathcal{A}_0} \Gamma_i)^{-1}(\sum_{i\in \mathcal{A}_0} \Gamma_i S_i),
\quad \Gamma_i := Cov(X_1^{(i)}, X_1^{(i)})$ in the Appendix. To get a consist estimator for $S_0$, we need to debias $\tilde{S}$ further, as shown in \eqref{alg_eq3}. Next, we introduce the form of Restricted Eigenvalue and Deviation Bound Condition we need in the analysis of this high-dimensional transfer learning problem.

\vspace{-0.35cm}

\begin{itemize}
\item [(B1)] Restricted Eigenvalue(RE):
\begin{equation*}
\begin{split}
&\alpha_2^{'} \Vert \Delta \Vert_F^2 + \tau_{n_{\mathcal{A}_0}} \Vert \Delta\Vert_1^2
    \geq
    \frac{1}{n_{\mathcal{A}_0}}\Vert \mathcal{X}_{\mathcal{A}_0}\Delta \Vert_F^2\\
    &\geq
    \alpha_2 \Vert \Delta \Vert_F^2 - \tau_{n_{\mathcal{A}_0}} \Vert \Delta\Vert_1^2    
\end{split}
\end{equation*}   
, where $\alpha > 0$, $\alpha^{'} >0$ and $\tau_{n_{\mathcal{A}_0}}= O(\frac{\log p}{n_{\mathcal{A}_0}}).$
\item [(B2)] Deviation Bound Condition:
\begin{equation*}
\frac{1}{n_{\mathcal{A}_0}}\Vert \sum_{i \in \mathcal{A}}\mathcal{X}_{i}^{'}\mathcal{E}_i \Vert_{\mathrm{max}}
\leq
\phi_{\mathcal{A}_0}\sqrt{\frac{\log p}{n_{\mathcal{A}_0}}}    
\end{equation*}
, where $\phi_{\mathcal{A}_0}$ is a constant depending on $\{B_i\}_{i \in \mathcal{A}_0}$ and $\{\Sigma_i\}_{i \in \mathcal{A}_0}.$
\end{itemize}

\vspace{-0.35cm}

Proposition~\ref{prop2} in the Appendix shows that (B1) and (B2) are satisfied with high probability in the high-dimensional regime.

\begin{theorem}\label{th2} 
Assume that $\Vert S_0\Vert_0 \leq s$ and $\Vert S_i - S_0 \Vert_1 \leq h$ for all $i \in \mathcal{A}$. We take $\mu = 2(c_3 + c_\Sigma)(1 \vee h^2)\sqrt{\frac{\log p}{n_{\mathcal{A}_0}}}$, $\lambda_\delta = c\sqrt{\frac{\log p}{n_0}}$. Assume that $\frac{n_{\mathcal{A}_0}(pr + s\log p)}{N n_0} = o(1)$ and $\frac{n_0(pr + s\log p)}{N \log p} = o(1)$. Under the condition (B1) and (B2), the estimator of \eqref{alg_eq3} satisfies
\begin{equation*}
\begin{split}
\Vert \widehat{S}_{tran} - S_0 \Vert_F^2
&\lesssim
h\sqrt{\frac{\log p}{n_0} } \wedge h^2 + (1 \vee h^4) \frac{s \log p}{n_{\mathcal{A}_0}}\\
&+ \frac{n_{\mathcal{A}_0}(pr + s\log p)^2 }{n_0 N^2}
\end{split}
\end{equation*}
with high probability.
\end{theorem}

Theorem \ref{th2} provides the convergence rate of $S_0$. This consistency rate underscores the non-trivial nature of our method and theoretical developments, as it deviates from existing rates in the literature \citep{li2022transfer,tian2022transfer,li2020transfer}. This distinctive consistency rate offers valuable insights into how the similarity between target and informative groups —quantified by $h$— affects the overall estimation error. Specifically, it elucidates the interplay among the similarity metric $h$, dimensionality $p$, sample sizes of the target and informative groups ($n_0$ and $n_{\mathcal{A}_0}$), rank $r$, and sparsity level $s$. Further, the upper bound on estimation error consists of two parts. First part, $h\sqrt{\frac{\log p}{n_0} } \wedge h^2 + (1 \vee h^4) \frac{s \log p}{n_{\mathcal{A}_0}}$, which we call the transfer learning error, is coming from transfer learning steps. This rate is the same as the rate of traditional transfer learning algorithm \citep{li2020transfer} when no low-rank component is present in the model. Second part is the last term representing the error due to the estimation error of $L$ in our first step. When we estimate $L$ with high accuracy given enough observations (i.e $N \gtrsim  \frac{n_{\mathcal{A}_0}(pr + s\log p)}{\sqrt{sn_0\log p}} $), the third term will be dominated by the transfer learning error. When the informative set $\mathcal{A}$ is empty ($h=0$ and $n_{\mathcal{A}_0} = n_0$), transfer learning error becomes $ \frac{s\log p}{n_0}$, which is the same as the rate of traditional lasso method \citep{basu2015regularized}. As we can see, using extra information received from informative groups improves the estimation accuracy when $h = o(s\sqrt{\frac{\log p}{n_0}})$.

\vspace{-0.35cm}

\subsection{Selecting Informative Set}
Algorithm \ref{algorithm1} is based on a known informative set, while informative set is typically unknown in practice. Misclassifying non-informative observations as informative observations does harm to the performance of transfer learning. Therefore, we need to select useful observations before applying Algorithm \ref{algorithm1}. The goal of this section is to determine informative models from all auxiliary models. This algorithm is inspired by \cite{tian2022transfer}.

The basic idea of selecting informative set comes from cross validation. We evenly split target data into two groups $X_{\mathcal{I}}^{(0)}$ and $X_{\mathcal{I}^c}^{(0)}$ where $\mathcal{I}$ plays the training set role and $\mathcal{I}^c$ as testing set. For each $k$, we estimate transition matrices for the $k$-th auxiliary model based on observations from both $k$-th group and $\mathcal{I}$. Then, we compute squared residual on the testing set, i.e. $ R^{(k)} = \Vert Y_{\mathcal{I}^c}^{(0)}- X_{\mathcal{I}^c}^{(0)} \widehat{\beta}^{(k)} \Vert_2^2$.
A lower test error $R^{(k)}$ implies a closer transition matrix $S_k$ to $S_0$, and thus the ones with lower $R^{(k)}$ is selected as informative set. The proposed algorithm is summarized in Algorithm~\ref{algorithm2} in the Appendix. Next, we make some additional assumptions before presenting the theoretical properties of $R^{(k)}$.

\begin{assumption}\label{assum10}
There exists some constant $M > 0$, such that, $\mathrm{sup}_k\Vert S_k - S_0 \Vert_1 \leq M.$
\end{assumption}
\begin{assumption}\label{assum11}
For $k\in\mathcal{A}$, $\Vert \delta^{(k)} \Vert_2^2 = O(\sqrt{\frac{\log (p)}{n_0/2}}) $; For $k \in \mathcal{A}^c$, $\Vert \delta^{(k)} \Vert_2^2 = \Omega(\sqrt{\frac{\log (p)}{n_0/2}}) $.
\end{assumption}

\begin{theorem}\label{th3}
    Suppose $\Vert S_0 \Vert_0 \leq s$, $\frac{s\log(p)}{n_0}=o(1)$ and $K = o(p^2)$. Taking $\lambda_k=C_1(1\vee M^2)\sqrt{\frac{\log(p)}{n_k+n_0/2}}$, where $C_1$ depends on $\mathcal{M}$ and $\mathfrak{m}$. Under Assumption \ref{assum10}, We have 
    \begin{equation}\label{th3_eq0}
 \Vert 
\delta^{(k)} \Vert_2^2 - \sqrt{\frac{\log(p)}{n_0/2}}       \lesssim R^{(k)} - R_{1}^{(0)} \lesssim \Vert 
\delta^{(k)} \Vert_2^2 + \sqrt{\frac{\log(p)}{n_0/2}},
    \end{equation}
with high probability. $R_1^{(0)}$ is defined in Algorithm \ref{algorithm2}. Further, suppose that Assumption \ref{assum11} holds. Then, we have $\mathbb{P}\{ \mathrm{sup}_{k\in\mathcal{A}} R^{(k)} < \mathrm{inf}_{k\in\mathcal{A}^c} R^{(k)} \}\rightarrow 1.$
\end{theorem}

Theorem \ref{th3} implies that informative groups and non-informative groups can be perfectly separated based on $R^{(k)} - R_1^{(0)}$. Since $R^{(k)} - R_1^{(0)}$ can be treated as testing error, models with lower $R^{(k)} - R_1^{(0)}$ will be preferred for transfer learning. Basically, we can set a threshold and select models with the value of $R^{(k)} - R_1^{(0)}$ below the threshold as informative set. In Algorithm \ref{algorithm2}, we use $| R_1^{(0)} - R_2^{(0)}|$ as a threshold to select models. As we can see in \eqref{th3_eq0}, $| R_1^{(0)} - R_2^{(0)}|$ is lower than $\sqrt{\frac{\log p}{n_0}}$ with high probability, which implies any model in $\mathcal{A}^c$ will be excluded from the estimated informative set by our proposed algorithm.

\vspace{-0.35cm}

\subsection{Inference for Sparse Component}

We propose an additional debiasing step to help with inference. The explicit form of debiased estimator is $   \widehat{S}^{on} = \widehat{S}_{tran} + \frac{1}{n_0}\sum_{i=1}^{n_0}M_i X_i^{(0)}(X_{i+1}^{(0)} - X_i^{(0)}(\widehat{L} + \widehat{S}_{tran}))^{'}$, where $M_i$ is called the debiasing matrix and needs to be estimated by target model. If observations are i.i.d, setting $M_1=M_2\cdots=M_{n_0}$ is an effective way to debias $\widehat{L}$ \citep{javanmard2014confidence}. However, for VAR models, the existence of dependency destroys the asymptotic normality. To fix this problem, we follow the online procedure in \cite{deshpande2021online} in estimating $M_i$ by past observations, $\{X_t\}_{t< i}$, which makes $M_i$ predictable. This online algorithm works for any estimator $\widehat{S}$ as long as $\Vert \widehat{S} -  S\Vert_1 = o(s\sqrt{\frac{\log p}{n_0}})$. For the proposed transfer learning estimator, such a rate holds when $h=o(s\sqrt{\frac{\log p}{n_0}})$. Details are deferred to the Appendix due to page limits.

% The approach of constructing confidence intervals for $\widehat{S}_{tran}$ with high dimensional VAR model is proposed by \cite{deshpande2021online}. To simplify the notation, we use $\widehat{S}$ to denote $\widehat{S}_{tran}$. \cite{deshpande2021online} provides an online algorithm and makes conclusions based on lasso estimator. Actually the online algorithm also works for any estimator $\widehat{S}$ as long as $\Vert \widehat{S} -  S\Vert_1 = o(s\sqrt{\frac{\log p}{n_0}})$. For transfer learning estimator, such $\ell_1$ boundness condition holds when $h=o(s\sqrt{\frac{\log p}{n_0}})$. For completeness, we briefly introduce how the online algorithm is applied here in supplement file. 

\vspace{-0.35cm}

\section{SIMULATION RESULTS}\label{sec4}
\vspace{-0.35cm}

\begin{comment}
In this section, we conduct two simulations. In the first simulation, we show the accuracy of recovering low rank
matrix in the first step of algorithm \ref{algorithm1} with increasing dimension. In the second simulation, we compare the performance of proposed transfer learning algorithms with the lasso method. Transfer learning algorithms include the Oracle Trans-Lasso (algorithm \ref{algorithm1} with known $\mathcal{A}$), Trans-Lasso (algorithm \ref{algorithm1} with $\mathcal{A}$ selected by algorithm \ref{algorithm2}), and naive Trans-lasso (algorithm \ref{algorithm1} with $\mathcal{A}=\{1,2,\cdots,K\}$). The lasso method is applied only for the target data. We focuses on both the estimation performance and the inference performance of oracle transfer learning algorithms with lasso method. All simulations are repeated 200 times.
\end{comment}

% Consider the VAR model
% \begin{center}
% Target model:
% $
% X^{(0)}_t = 
% (L + A_0)
% X^{(0)}_{t-1}
% +
% \epsilon^{
% (0)}_{t}
% $;
% Auxiliary model:
% $
% X^{(k)}_t = 
% (L + A_k)
% X^{(k)}_{t-1}
% +
% \epsilon^{(k)}_{t}, \ 1\leq k\leq K,
% $
    
% \end{center}

In this section, we compare the performance of the proposed transfer learning algorithms with the lasso method (see also additional simulation studies on improvement for recovering the low-rank part utilizing the proposed transfer learning algorithm as well as reporting on computation time in the appendix). Transfer learning algorithms include the Oracle Trans-Lasso (Algorithm \ref{algorithm1} with known $\mathcal{A}$), Trans-Lasso (Algorithm \ref{algorithm1} with $\mathcal{A}$ selected by Algorithm \ref{algorithm2}), and naive Trans-lasso (Algorithm \ref{algorithm1} with $\mathcal{A}=\{1,2,\cdots,K\}$). The lasso method is applied only for the target data. Discuss on hyperparameter selection and their sensitivity analysis are summarized in the Appendix due to space consideration.

% \textit{Hyperparameter selection.} Using cross-validation to select hyperparameters (tuning parameters) in our model is difficult due to the time series nature of the data. Also, it would be computationally demanding since each private sparse component corresponds to one hyperparameter. It seems difficult to compute all cases when dealing with large groups. There are three types of hyperparameters in the proposed algorithms. The first group is related to lasso penalty terms. For those, we use suggestions in the literature \cite{li2020transfer}  for tuning parameter selection. Specifically, we set $\lambda_\beta = \sqrt{2\log p/(n_0+n_{\mathcal{A}_0})}$ and $\lambda_\delta = \sqrt{2\log p/(n_0)}$ . Second, there will be a tuning parameter for the low rank penalty. For that, we set $\mu = \tau * \sqrt{np}$. We performed several simulations with different $\tau$. Empirically speaking, our algorithm reaches a good result when $\tau \in (0.1, 2)$. We set $\tau = 0.1$ in all numerical analyses. Finally, there will be an additional tuning parameter for the selection algorithm which is the constant $c$. We performed sensitivity analysis and given $c \in (0.01, 0.5)$, the algorithm always selects helpful groups for the next transfer learning step. We set $c = 0.01$ in all numerical analyses.

We focus on both estimation and inference performances of oracle transfer learning algorithms with the lasso method. All simulations are repeated 200 times. In this simulation setting, the entries of $S_0$ are generated
independently from a Bernoulli distribution with success probability 
$q = 0.02$, multiplied by $b\cdot \mathrm{uniform}(\{+1, -1\})$ with $b = 0.25$, i.e. $b\cdot\mathrm{Bernoulli}(q)\cdot \mathrm{uniform}(\{+1, -1\})$. $L$ is generated by $L = U  D  V^{'}$, where $D := diag(0.2, r)$ is a diagonal matrix. Rank of $L$, r is set to be 8, We set $p = 100$, $n_0 = 200$, $n_1 = n_2= \cdots = n_K = 100$, and $K = 10$. Note that in this setting, the total number of parameters in the target model is $p^2 = 10,000$ while there are only $200$ time points. Thus, this can be regarded as the high-dimensional case. Let $\mathcal{A}$ denote the informative set. We define $\mathcal{J}$ as the set of non-zero entries in $S_0$, and $\mathcal{J}^c$ as the set of zero entries. For the transition matrices of auxiliary models $S_k$, we construct them by modifying entries of $S_0$ in $\mathcal{J}$ and $\mathcal{J}^c$ separately. For a given $k$, let $H_k$ be a random subset of $\mathcal{J}$ such that $|H_k| = \gamma |\mathcal{J}|$, and $G_k$ be a random subset of $\mathcal{J}^c$ such that $|G_k| = \gamma |\mathcal{J}^c| $, where $\gamma = \gamma_1$ if $k\in\mathcal{A}$, and $\gamma = \gamma_2$ otherwise, and it ranges from $0$ to $1$.
If $(i,j)\in H_k$, we set $S^{(k)}_{ij} = - S^{(0)}_{ij} $. If $(i,j) \in G_k$, we set $S^{(k)}_{ij} = S^{(0)}_{ij} + \eta_{ij}, \text{where}\ \eta_{ij}\sim \text{uniform}(-0.1,0.1) $. The two terms $\gamma_1$ and $\gamma_2$ are the percentage of changes we make for entries of $A_0$. We set $\gamma_1 = 0.04$, $\gamma_2 = 0.4$ and $|\mathcal{A}| = \{0,2,4,\cdots,10\}$.

% We use suggestions in \cite{li2020transfer} and \cite{deshpande2021online} for tuning parameter selection. Specifically, for Oracle Trans-lasso, we set $\lambda_\beta = \sqrt{2\log p/(n_0+n_{\mathcal{A}_0})}$ and $\lambda_\delta = \sqrt{2\log p/(n_0)}$ while for Trans-Lasso, we set $c=0.01$ in Algorithm \ref{algorithm2}. For the lasso method, we set $\lambda = \sqrt{2\log p/(n_0)}$. For inference part we use $r_0 = \sqrt{n_0}$, $r_i = 2^{i}$ and $\mu_j = \sqrt{\frac{\log p}{2m_j}}$ for $i, j\geq 1$.

The estimation error is shown in Figure~\ref{simu2_fig1} in the Appendix. We present the absolute estimation error, i.e $\Vert S_0 - \widehat{S}\Vert_1$, for lasso, Oracle Trans-lasso, naive Trans-lasso and Trans-lasso with an increasing samples in informative set. As excepted, Trans-lasso has better estimation error than lasso method when we consider enough informative samples. The similar behavior of Oracle Trans-lasso and lasso implies that Algorithm \ref{algorithm2} selects informative set accurately. As for naive Trans-lasso, it outputs the worst estimation when the size of informative set $|\mathcal{A}|$ is small and reach the same accuracy level when $|\mathcal{A}|$ increases. This is because naive Trans-lasso takes non-informative set into consideration, which damages the estimation performance.

In addition, we construct entry-wise confidence interval for sparse matrix based on Trans-lasso and lasso separately. To make comparison for inference performance, we consider four metrics: True Positive Rate (TPR), False Positive Rate (FPR), coverage rate of confidence intervals and average confidence interval length (Avg CI length). Figure \ref{simu2_fig2} summarizes the results for all methods
at significance level $\alpha = 0.05$. As we can see, Trans-lasso shows a comparable result with lasso method when $|\mathcal{A}| = 0$ (all auxiliary sets are non-informative). This is because transfer learning with no informative set  is equivalent to lasso method. As more informative sets are provided, Trans-lasso performs better in terms of FPR, TPR and coverage rate. The coverage rate gradually goes up to $0.95$ as $|\mathcal{A}|$ increases, which is consistent with our significant level $\alpha = 0.05$. Since both lasso and transfer learning use the same method to generate conditional variance $V_n$, the length of confidence intervals for lasso and transfer learning are at the same level all the time.

\begin{figure}[h]
    \begin{tabular}{ccc}
        \centering
        \subfigure[FPR(naive)]{
        \begin{minipage}[t]{0.24\linewidth}
        \centering
        \includegraphics[width=1.0in]{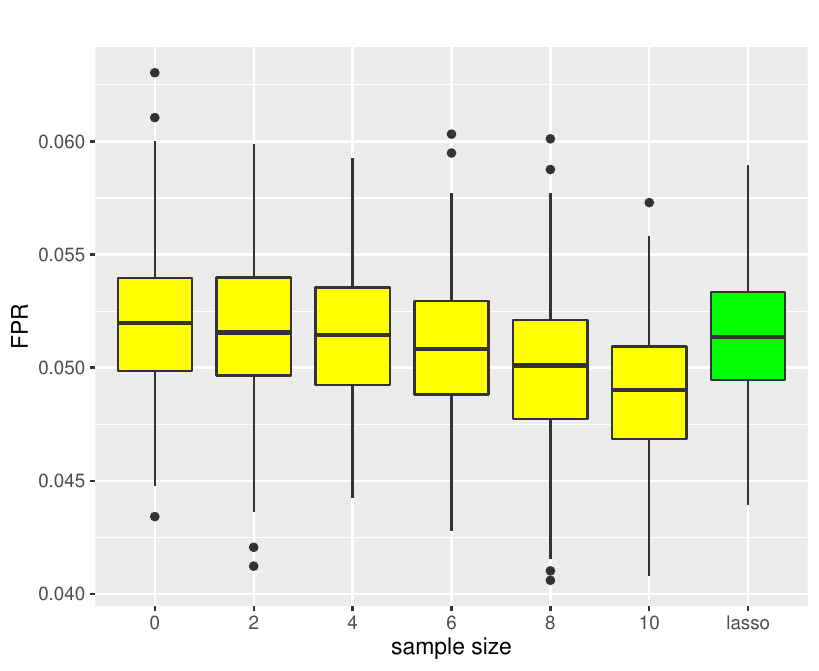}
        %\caption{fig1}
        \end{minipage}%
        }&\subfigure[FPR(Oracle-Trans)]{
        \begin{minipage}[t]{0.24\linewidth}
        \centering
        \includegraphics[width=1.0in]{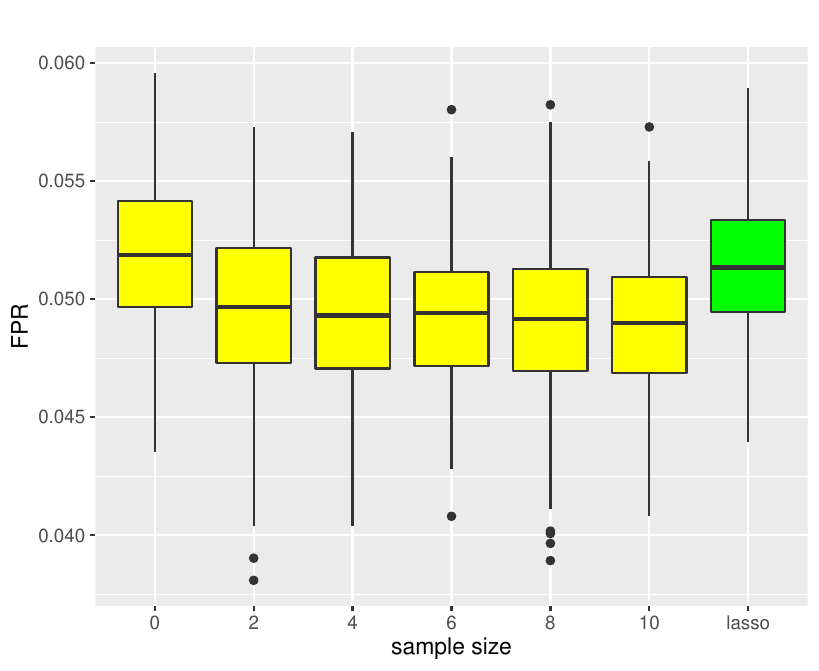}
        %\caption{fig1}
        \end{minipage}%
        }&\subfigure[FPR(Trans-Lasso)]{
        \begin{minipage}[t]{0.24\linewidth}
        \centering
        \includegraphics[width=1.0in]{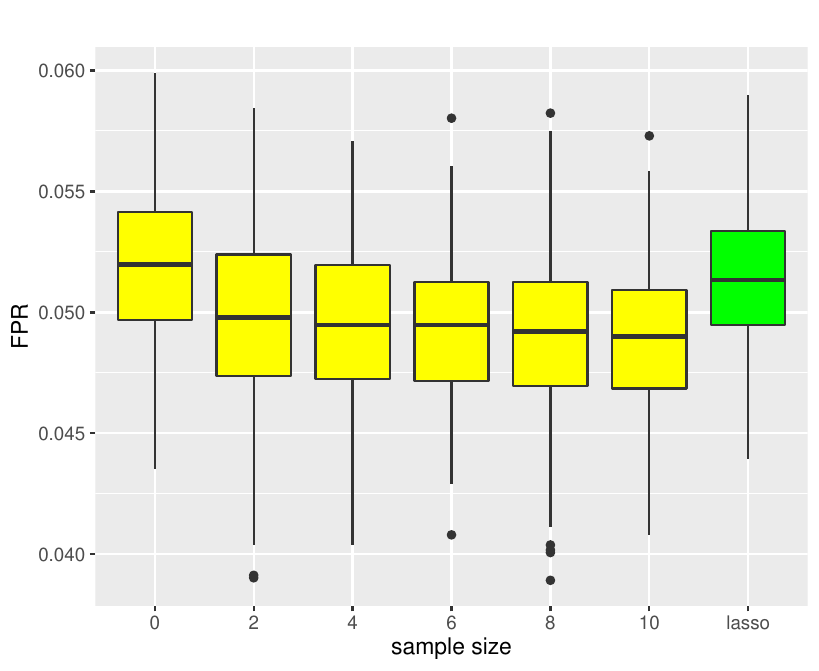}
        %\caption{fig1}
        \end{minipage}%
        }
        \\
        \subfigure[TPR(naive)]{
        \begin{minipage}[t]{0.24\linewidth}
        \centering
        \includegraphics[width=1.0in]{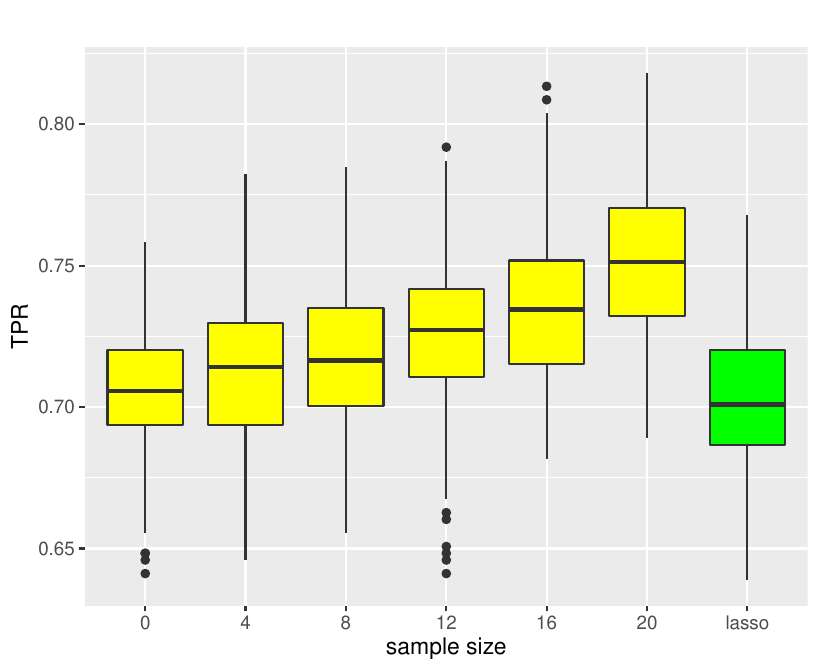}
        %\caption{fig1}
        \end{minipage}%
        }&\subfigure[TPR(Oracle-Trans)]{
        \begin{minipage}[t]{0.24\linewidth}
        \centering
        \includegraphics[width=1.0in]{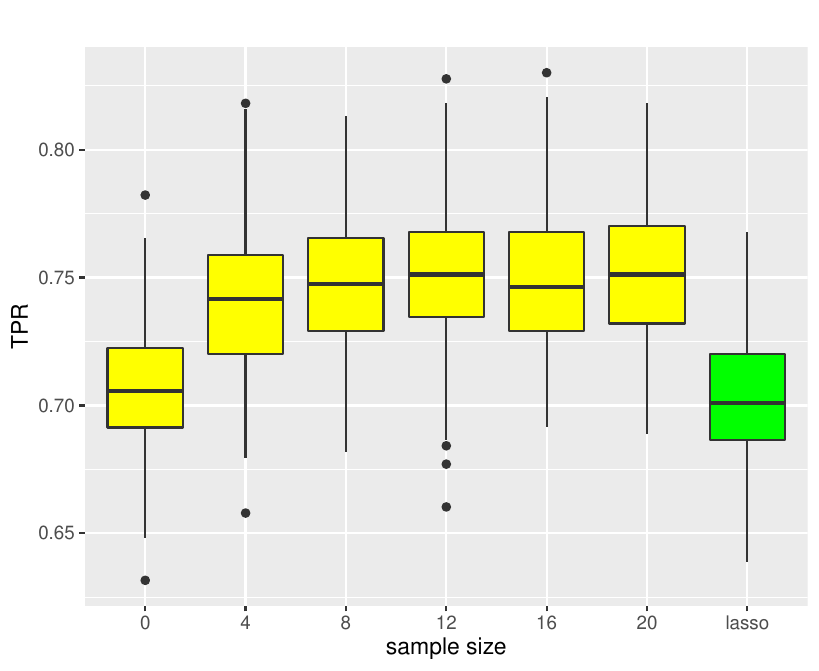}
        %\caption{fig1}
        \end{minipage}%
        }&\subfigure[TPR(Trans-Lasso)]{
        \begin{minipage}[t]{0.24\linewidth}
        \centering
        \includegraphics[width=1.0in]{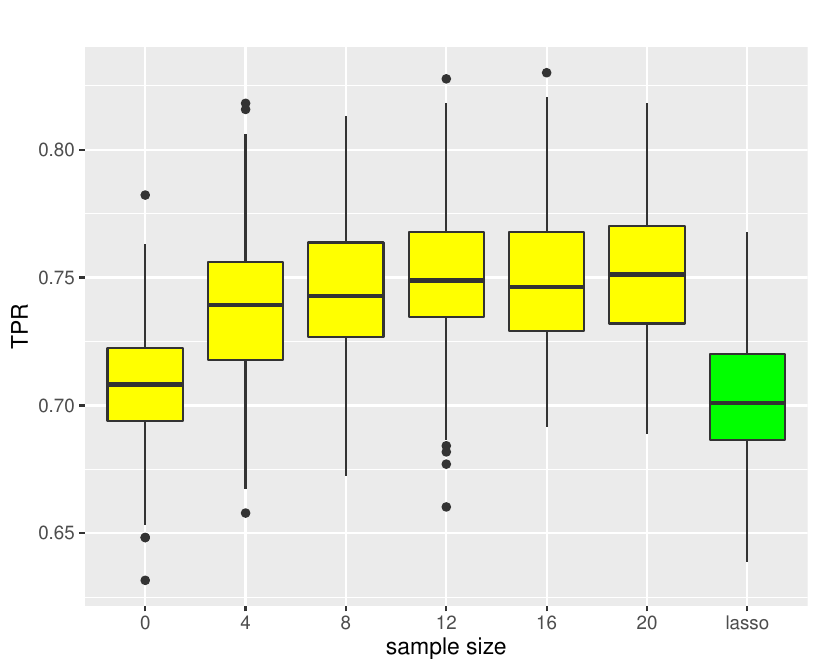}
        %\caption{fig1}
        \end{minipage}%
        }
        \\
        \subfigure[Avg CI length(naive)]{
        \begin{minipage}[t]{0.24\linewidth}
        \centering
        \includegraphics[width=1.0in]{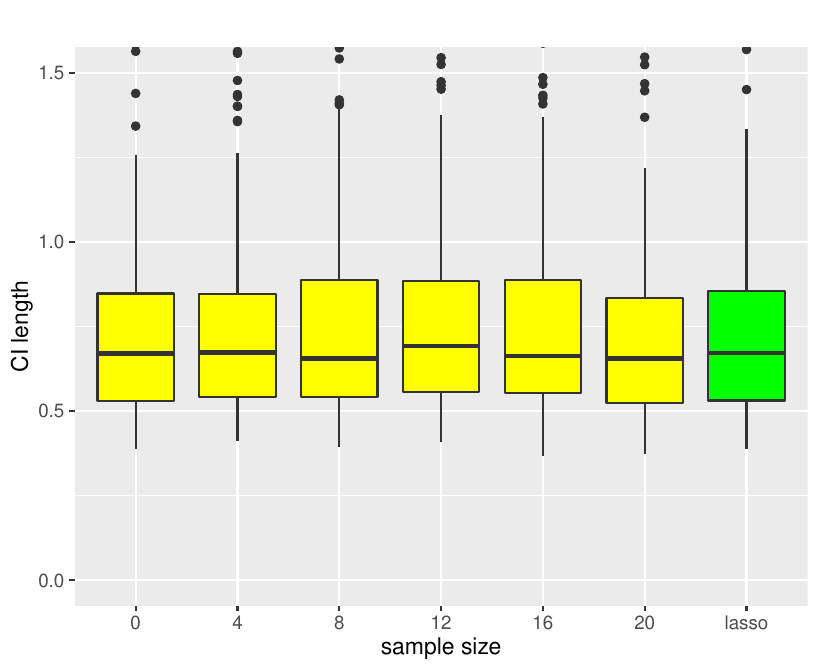}
        %\caption{fig1}
        \end{minipage}%
        }&\subfigure[Avg CI length(Oracle-Trans)]{
        \begin{minipage}[t]{0.24\linewidth}
        \centering
        \includegraphics[width=1.0in]{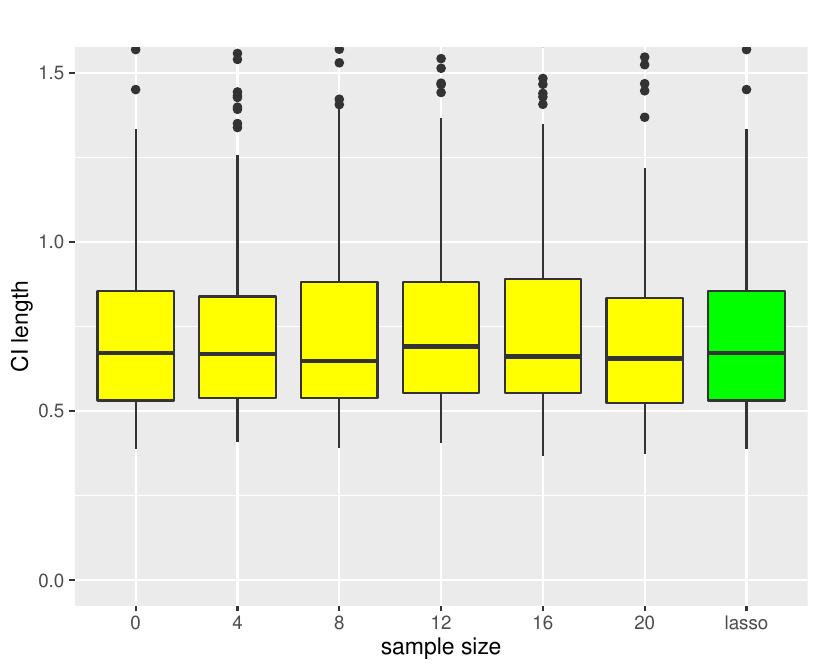}
        %\caption{fig1}
        \end{minipage}%
        }&\subfigure[Avg CI length(Trans-Lasso)]{
        \begin{minipage}[t]{0.24\linewidth}
        \centering
        \includegraphics[width=1.0in]{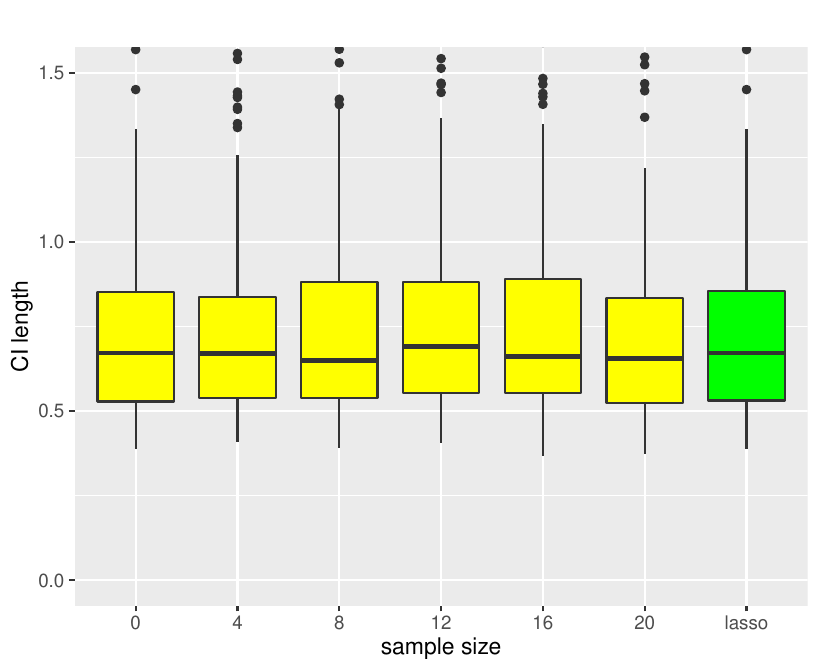}
        %\caption{fig1}
        \end{minipage}%
        }
        \\
        \subfigure[Coverage rate(naive)]{
        \begin{minipage}[t]{0.24\linewidth}
        \centering
        \includegraphics[width=1.0in]{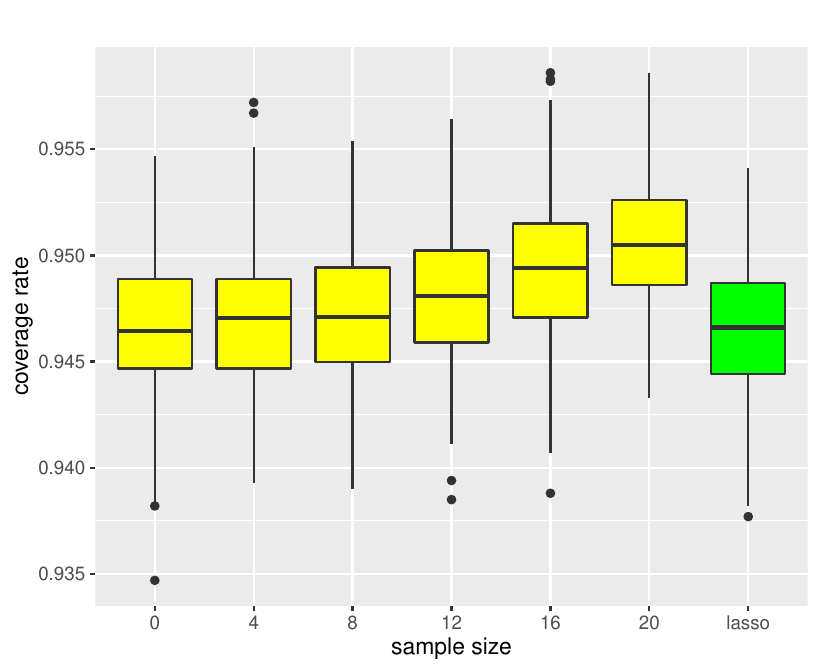}
        %\caption{fig1}
        \end{minipage}%
        }&\subfigure[Coverage rate(Oracle-Trans)]{
        \begin{minipage}[t]{0.24\linewidth}
        \centering
        \includegraphics[width=1.0in]{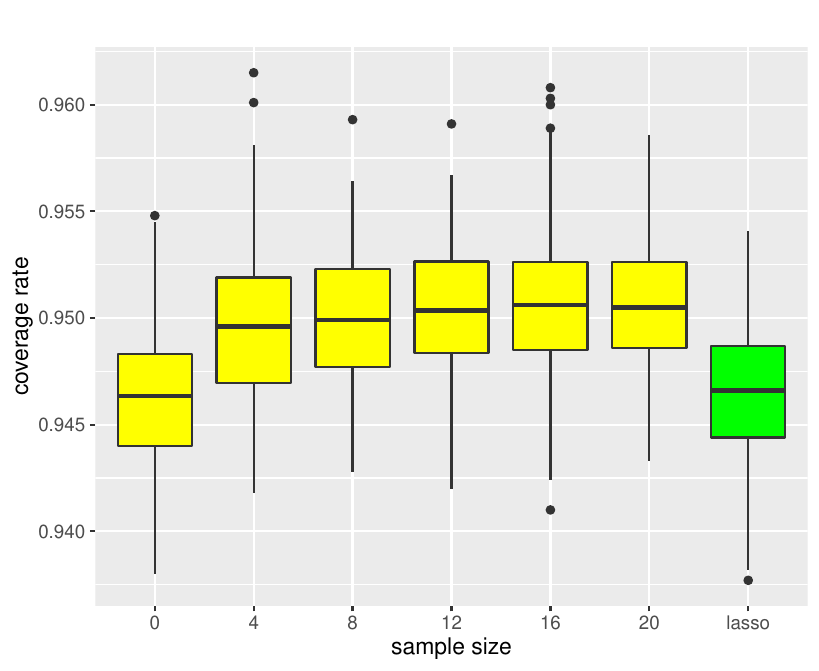}
        %\caption{fig1}
        \end{minipage}%
        }&\subfigure[Coverage rate(Trans-Lasso)]{
        \begin{minipage}[t]{0.24\linewidth}
        \centering
        \includegraphics[width=1.0in]{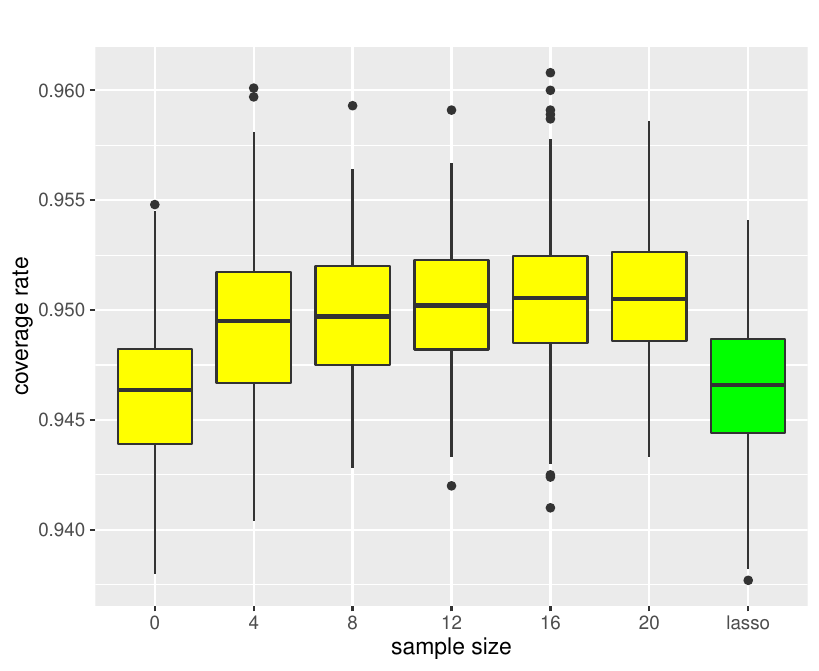}
        %\caption{fig1}
        \end{minipage}%
        }
    \end{tabular}
\caption{FPR (False Positive Rate), TPR (True Positive
Rate), Coverage Rate and Average Length of Confidence Intervals (Avg CI Length) at significance level $\alpha = 0.05$. The $x$ axis is the size of informative set $| \mathcal{A}|$.}
\label{simu2_fig2}
\end{figure}

\vspace{-0.35cm}

\section{A REAL DATA EXAMPLE}\label{sec5}
\vspace{-0.35cm}
The proposed transfer learning algorithm is applied to a surveillance video data set obtained from the CAVIAR project\footnote{http://homepages.inf.ed.ac.uk/rbf/CAVIARDATA1/}. A number of video clips record different actions by people in diverse settings, including walking alone, meeting with others, entering and exiting a room, etc. The data set we analyze is ``Two other people meet and walk together". The original data set has $837$ images in total and the resolution of each image is half-resolution PAL standard (384 × 288 pixels, 25 frames per second). Before applying our proposed algorithm, we re-sized the original images from 384 × 288 pixels to
32 × 24 pixels and used a gray-scaled scheme instead of the
original colored image to accelerate computations. We then digitalize and vectorize 32 × 24 image to be a 768 dimension vector. Therefore, the resulting time series process has $n = 837$ time points and $p = 768$ features.

The whole video data can be divided into 12 segments depending on human activities in the video (see more details in \cite{bai2020multiple}). The background (mainly the lobby) seems fixed during the time while certain human movements occur during the video. When our model is applied to this data, the low-rank part will capture the background (lobby) while the sparse components will capture the additional human movements during the video clip. To ensure the proposed model is a good fit, we compared it with several additional parameterizations (including low-rank only, sparse only, etc.) and concluded that the proposed algorithm coupled with low rank plus sparse model parameters outperforms all other competing models. More specifically, we consider five scenarios in total: Trans-lasso(L+S), Trans-lasso(S), lasso(L+S), lasso(S) and Low-rank. Trans-lasso(L+S) and lasso(L+S) refer to methods that we model VAR with low-rank plus sparse structure, with and without transfer learning, respectively. Trans-lasso(S) and lasso(S) refer to methods that we model VAR with sparse components only, with and without transfer learning, respectively. Low-rank method implies that we model VAR with only low-rank component for each segment. Results of this comparison are summarized in Table~\ref{table1} in the Appendix. As seen from this table, the proposed modeling framework with the help from the proposed transfer learning algorithm achieves the best prediction error overall.

First row of Figure \ref{fig3} shows the start time point for four of these segments/movements: 1st segment, 4th segment, 6th segment and 9th segment. We apply Algorithm \ref{algorithm1} to estimate the low-rank component (dimension is $768 \times 768$) for each segment. Since non-changing low-rank component corresponding to the stationary background of the space surveyed and the changing sparse component corresponds to movement of people in and out of the space in the evolving foreground, the sparse component can imply the position of people in the lobby. To visualize the information contained in sparse component, we (1) construct entry-wise $95\%$ confidence interval of the sparse estimator $\widehat{S}$, (2) count the number of significant entries in each row, i.e $V := (v_1,\cdots, v_{768}), \ v_i = \#\{j:\widehat{S}_{ij} \mbox{is significant}\} $, (3) map vector $V$ back to a $32 \times 24$ matrix $M$. Figure \ref{fig3e} -\ref{fig3h} shows the heatmap of $M$. As we can see, the dark region of the heatmap perfectly matches the position of people in original image. Results for other segments are summarized in the Appendix. Further, Table~\ref{table1} summarizes out-of-sample mean squared prediction error (each segment is split such that its first $2/3$ observations are used as training and the remaining parts as testing data) obtained from lasso and Trans-lasso (T-lasso) algorithms which clearly illustrates the great reduction of prediction error when similar images are used in the estimation procedure for each segment.

\begin{table}[h]
\caption{Mean Squared Prediction Error for Each Segment; Standard Errors Are Shown in Parentheses.}\label{table1}
    \centering
    \begin{tabular}{c|c|c|c|}
    \hline
    & seg1 & seg2 & seg3 \\
    \hline
T-lasso & 7.205(0.015) & 0.133(0.003) & 2.177(0.012)  \\
lasso & 8.660(0.017) & 0.241(0.004) & 4.484(0.017)  \\
\hline
    & seg4 & seg5 & seg6 \\
    \hline
T-lasso  & 0.468(0.005) & 0.092(0.002) & 0.916(0.008) \\
lasso  & 1.928(0.011) & 0.450(0.005) & 2.686(0.013) \\
 \hline
 & seg7 & seg8 & seg9 \\
 \hline
T-lasso & 0.372(0.005) & 0.233(0.004) & 0.146(0.003) \\
lasso & 1.653(0.010) & 1.235(0.009) & 0.523(0.006)  \\
\hline
 & seg10 & seg11 & seg12\\
 \hline
T-lasso & 0.094(0.002) & 0.071(0.002) & 0.139(0.002)\\
lasso  & 0.189(0.004) & 0.102(0.004) & 0.153(0.002) \\
\hline
    \end{tabular}
\end{table}

\begin{comment}
\begin{table}[]
    \centering
    \begin{tabular}{c|c|c|c|}
    \hline
        & naive & Trans-lasso & lasso\\
        \hline
segment1 & 7.214 & 7.214  & 8.627 \\
segment2 & 0.118 & 0.131 & 0.241 \\
segment3 & 1.983 & 1.983 & 4.485 \\
segment4 & 0.465 & 0.465  & 1.952 \\
segment5 & 0.087 & 0.091 & 0.449  \\
segment6 & 0.896 & 0.896 &  2.682 \\
segment7 & 0.376 & 0.376 & 1.700  \\
segment8 & 0.233 & 0.240 & 1.237 \\
segment9 & 0.141 & 0.148 & 0.519 \\
segment10 & 0.094 & 0.094 & 0.188 \\
segment11 & 0.069 & 0.070 & 0.102 \\
segment12 & 0.139 & 0.138 & 0.152 \\
 \hline
    \end{tabular}
    \caption{Prediction mean square error.}
    \label{table1}
\end{table}
\end{comment}

\begin{figure}[h]
    \begin{tabular}{cc}
        \centering
        \subfigure[T = 0]{
        \begin{minipage}[t]{0.4\linewidth}
        \centering
        \includegraphics[width=1.4in]{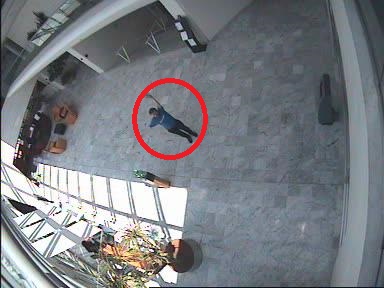}
        %\caption{fig1}
        \label{fig3a}
        \end{minipage}%
        }&\subfigure[T=231]{
        \begin{minipage}[t]{0.4\linewidth}
        \centering
        \includegraphics[width=1.4in]{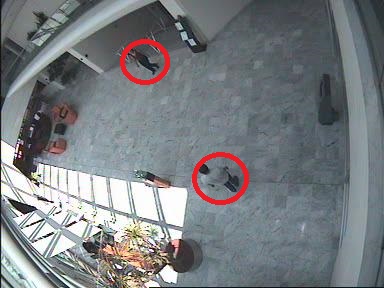}
        %\caption{fig1}
        \label{fig3b}
        \end{minipage}%
        }
        \\
        \subfigure[T = 347]{
        \begin{minipage}[t]{0.4\linewidth}
        \centering
        \includegraphics[width=1.4in]{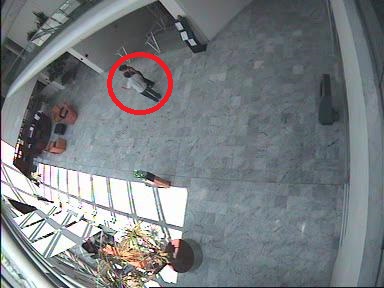}
        %\caption{fig1}
        \label{fig3c}
        \end{minipage}%
        }&\subfigure[T = 521]{
        \begin{minipage}[t]{0.4\linewidth}
        \centering
        \includegraphics[width=1.4in]{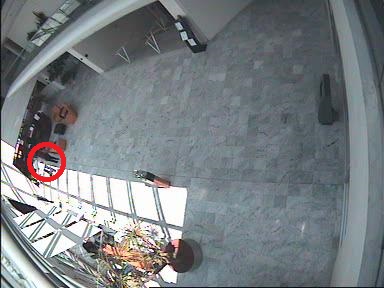}
        %\caption{fig1}
        \label{fig3d}
        \end{minipage}%
        }
        \\
        \subfigure[T = 0]{
        \begin{minipage}[t]{0.4\linewidth}
        \centering
        \includegraphics[width=1.4in]{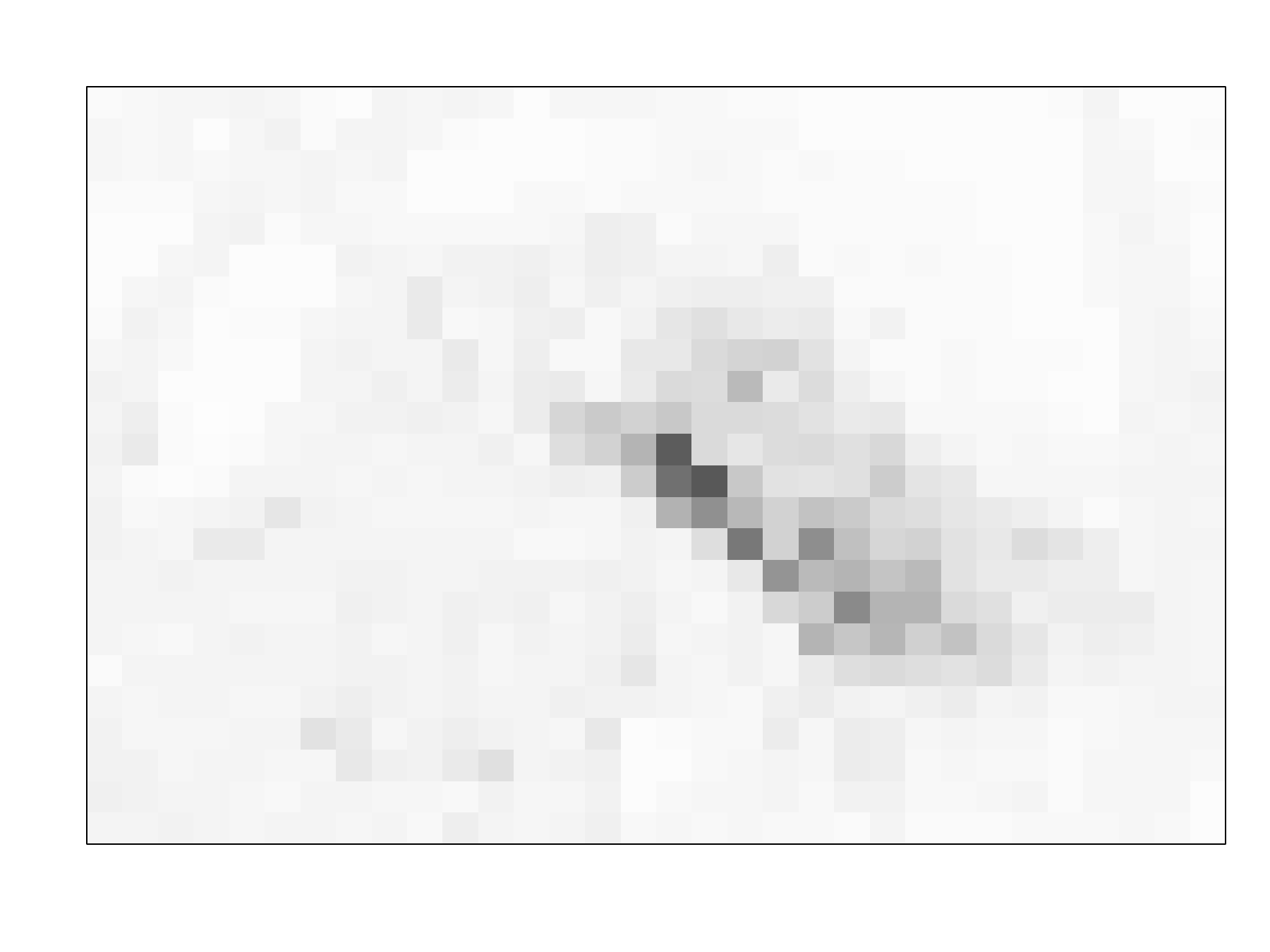}
        %\caption{fig1}
        \label{fig3e}
        \end{minipage}%
        }&\subfigure[T=231]{
        \begin{minipage}[t]{0.4\linewidth}
        \centering
        \includegraphics[width=1.4in]{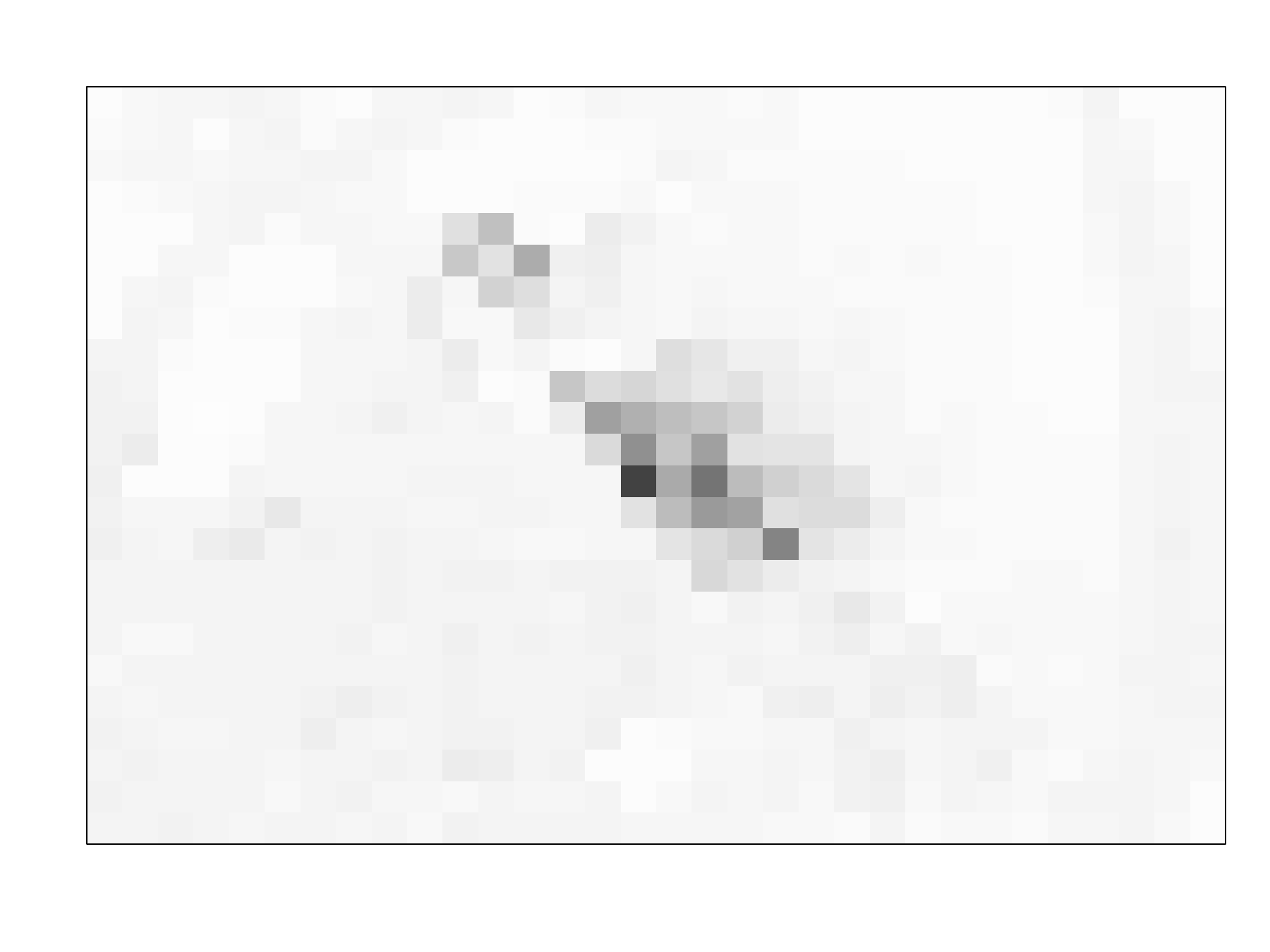}
        %\caption{fig1}
        \label{fig3f}
        \end{minipage}%
        }
        \\
        \subfigure[T = 347]{
        \begin{minipage}[t]{0.4\linewidth}
        \centering
        \includegraphics[width=1.4in]{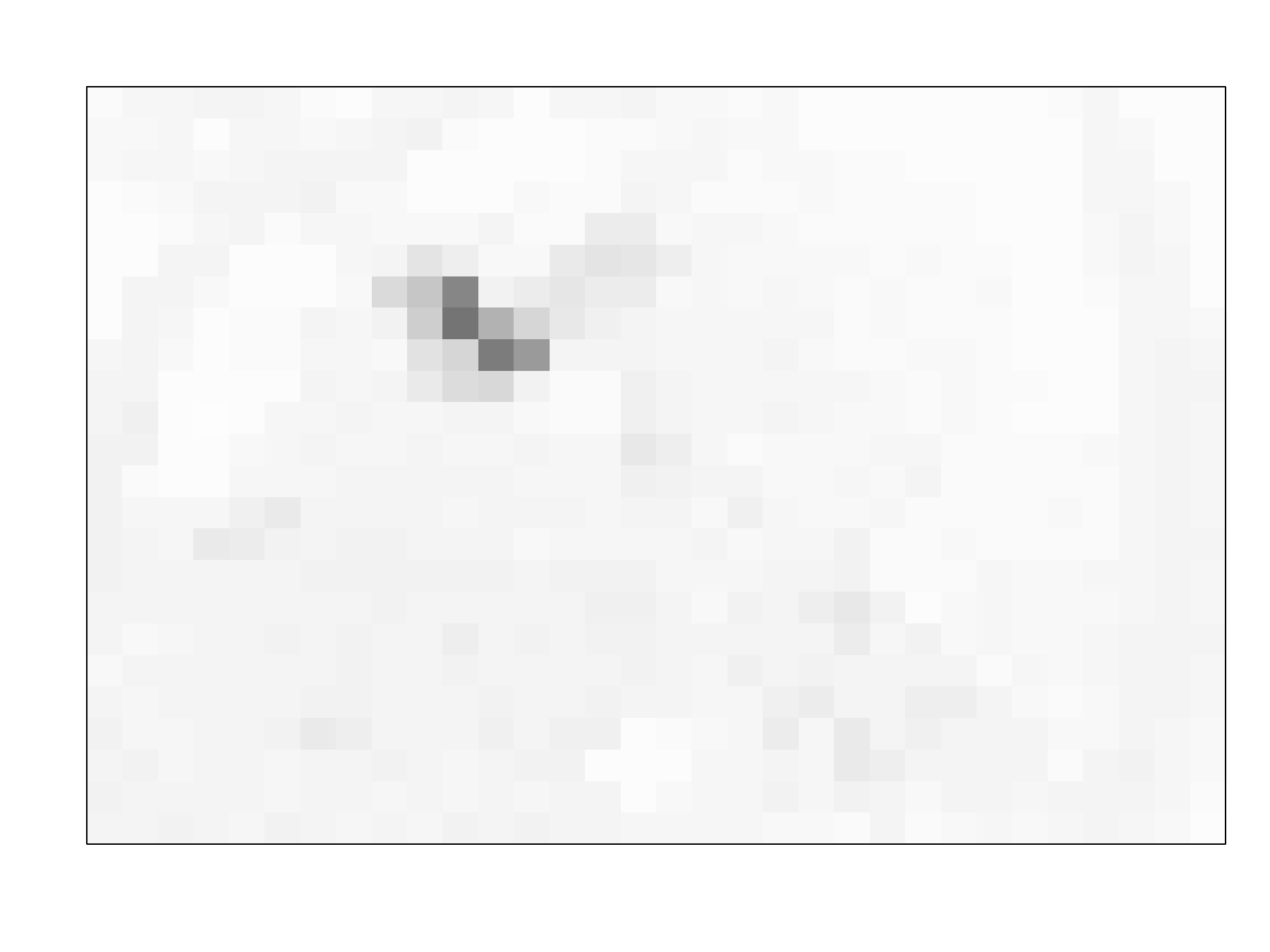}
        %\caption{fig1}
        \label{fig3g}
        \end{minipage}%
        }&\subfigure[T = 521]{
        \begin{minipage}[t]{0.4\linewidth}
        \centering
        \includegraphics[width=1.4in]{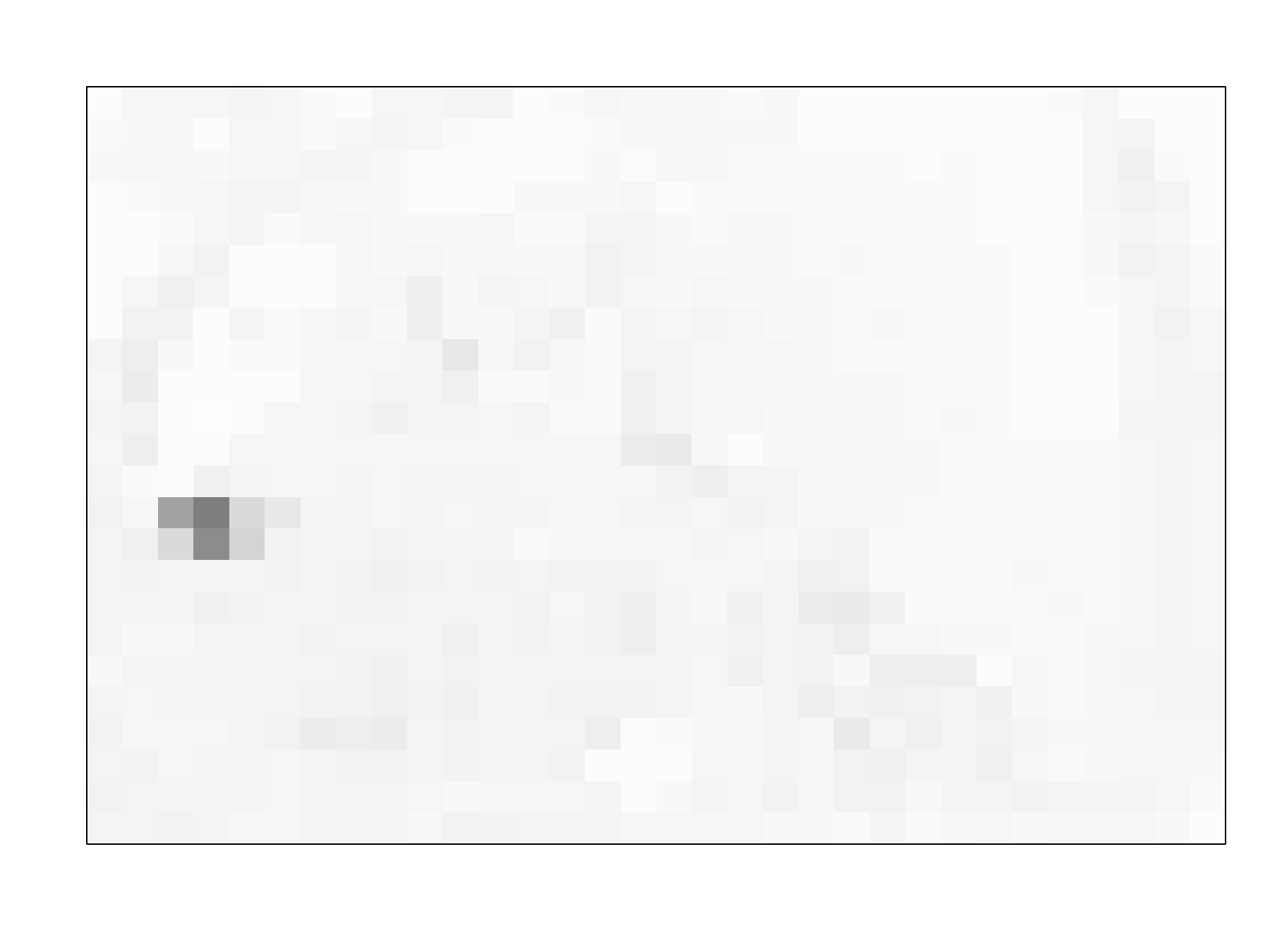}
        %\caption{fig1}
        \label{fig3h}
        \end{minipage}%
        }
    \end{tabular}
\caption{\ref{fig3a}-\ref{fig3d} are the view of footage for 1st segment, 4th segment, 6th segment and 9th segment respectively. \ref{fig3e}-\ref{fig3h} are corresponding heatmaps of sparse estimators.}
\label{fig3}
\end{figure}

\section{DISCUSSION}\label{sec6}
In this paper, we propose an step-wise algorithm for implementing transfer learning for VAR models with low-rank plus sparse structure. Theoretical results confirm that our transfer learning algorithm can improve the estimation accuracy for the low-rank and sparse components given known informative set. We also provide an approach based on prediction error for properly selecting the informative sets when it is unknown. Numerical experiments and real data applications support the theoretical findings. In our model, all auxiliary models have a common low-rank component. How to relax to the case where auxiliary models have different but similar low-rank components is an interesting future research direction. Measuring the similarity according to the column space of low-rank components \citep{tian2023learning} can be a feasible approach. Another limitation of our work is considering a VAR model with single lag. Extensions to VAR model of general lag, i.e. VAR(d) models (utilizing techniques in \cite{lutkepohl2005new}) is of interest.

% \newpage
\bibliographystyle{apalike}
\bibliography{main}

\section*{Appendix}
% the \\ insures the section title is centered below the phrase: AppendixA
In this section, some prior information about VAR models are summarized in Section~\ref{sec:prior}, some useful lemmas with their proofs are provided in Section~\ref{sec:1} while some propositions with their proofs are summarized in Section~\ref{sec:2}. Proof of main theorems are stated in Section~\ref{sec:3} while additional details on the proposed algorithms are described in Section~\ref{sec:4}. Further, some additional details on numerical studies and a new simulation study are explained in Section~\ref{sec:5}. Finally, computer information is summarized in Section~\ref{sec:6}. 

\section{Prior Knowledge for VAR Model in High-dimensions}\label{sec:prior}

% \begin{equation*}
% f_X(\theta)
% =
% \frac{1}{2\pi} \sum_{-\infty}^{\infty} \Gamma_X(h) e^{-ih \theta}
% \end{equation*}

For a $p$-dimensional centered, covariance-stationary process $\{X_t\}_{t \in \mathbb{Z}}$ with autocovariance function $\Gamma_X (h) = \mathrm{Cov}(X_t, X_{t+h})$, its spectral density is defined as $f_X(\theta)
=
\frac{1}{2\pi} \sum_{-\infty}^{\infty} \Gamma_X(h) e^{-ih \theta}$. For the VAR model \eqref{sec1_eq1} in the main file, the spectral density has the closed form $
f_{X^{(k)}}(\theta)
=
\frac{1}{2\pi}(\mathcal{B}_k^{-1}(e^{i\theta}))\Sigma_k (\mathcal{B}_k^{-1}(e^{i\theta}))^{*}
$ where $\mathcal{B}_k(z) = I_p - B_k^{'} z$ is the characteristic polynomial and $\Sigma_k$ is the covariance matrix of the error term. To introduce some useful properties for VAR model, we need the following quantities

\vspace{-0.35cm}

\begin{equation*}
\begin{split}
&\mathcal{M}(f_{X^{(k)}}) := \mathop{\mathrm{sup}}_{\theta \in [-\pi, \pi]} \Lambda_{\mathrm{max}}(f_{X^{(k)}}(\theta))\\
&\mathfrak{m}(f_{X^{(k)}}) := \mathop{\mathrm{sup}}_{\theta \in [-\pi, \pi]} \Lambda_{\mathrm{min}}(f_{X^{(k)}}(\theta))\\
&\mu_{\mathrm{max}}(\mathcal{B}_k) := \mathop{\mathrm{max}}_{|z| = 1}\Lambda_{\mathrm{max}}(\mathcal{B}_k^{*}(z) \mathcal{B}_k(z))\\
&\mu_{\mathrm{min}}(\mathcal{B}_k) := \mathop{\mathrm{min}}_{|z| = 1}\Lambda_{\mathrm{max}}(\mathcal{B}_k^{*}(z) \mathcal{B}_k(z)).
\end{split}
\end{equation*}

\vspace{-0.35cm}

Stability is always a basic assumption in time series model to ensure consistent estimation. \cite{basu2015regularized} provide a new measure of stability described by $\mathcal{M}(f_X)$ and shows that a larger $\mathcal{M}(f_X)$ implies a  less stable process. $\mathcal{M}(f_{X})$ and $ \mathfrak{m}(f_{X})$ capture the dependence among the univariate components of the vector-valued time series and help quantify dependence among the columns of the design matrix in our analysis. For VAR model, the boundness of $\mathcal{M}(f_{X})$ and $ \mathfrak{m}(f_{X})$ are related to $\mu_{\mathrm{max}}(\mathcal{B})$ and $\mu_{\mathrm{min}}(\mathcal{B})$:
$
\mathcal{M}(f_{X^{(k)}}) \leq \frac{1}{2\pi}\frac{\Lambda_{\mathrm{max}}(\Sigma_k)}{\mu_{\mathrm{min}}(\mathcal{B}_k)},\ 
\mathfrak{m}(f_{X^{(k)}}) \geq \frac{1}{2\pi}\frac{\Lambda_{\mathrm{min}}(\Sigma_k)}{\mu_{\mathrm{max}}(\mathcal{B}_k)}.
$

\noindent
\textit{Notations.} For a matrix $A$, its transpose is denoted by $A^{'}$ while $\text{vec}(A)$ is the vectorized version of matrix $A$. $\Lambda_{min}(A)$, $\Lambda_{max}(A)$ denote the smallest and largest eigenvalues of $A$, respectively. Define $\Gamma_{X^{(k)}}(i-j) := \mathrm{Cov}(X_i^{(k)}, X_j^{(k)})$.
$\Upsilon_{n_k}^{X^{(k)}} = \mathrm{Cov}(\mathrm{vec}((\mathcal{X}^{(k)})^{'}),\mathrm{vec}((\mathcal{X}^{(k)})^{'})) $.
From Proposition 2.3 in \cite{basu2015regularized}, we know that
\begin{equation*}
\begin{split}
&2\pi \mathfrak{m}(f_{X^{(k)}}) \leq \Lambda_{min}(\Upsilon_{n_k}^{X^{(k)}}) \leq \Lambda_{max}(\Upsilon_{n_k}^{X^{(k)}}) \leq 2\pi \mathcal{M}(f_{X^{(k)}})\\
&2\pi \mathfrak{m}(f_{X^{(k)}}) \leq \Lambda_{min}(\Gamma_k) \leq \Lambda_{max}(\Gamma_k) \leq 2\pi \mathcal{M}(f_{X^{(k)}}).
\end{split}
\end{equation*}

Similarly, for $\{\epsilon_0^{(k)}, \cdots, \epsilon_{n_k}^{(k)}\}$, we have
\begin{equation*}
\begin{split}
&2\pi \mathfrak{m}(f_{\epsilon^{(k)}}) \leq \Lambda_{min}(\Upsilon_{n_k}^{\epsilon^{(k)}}) \leq \Lambda_{max}(\Upsilon_{n_k}^{\epsilon^{(k)}}) \leq 2\pi \mathcal{M}(f_{\epsilon^{(k)}})\\
&2\pi \mathfrak{m}(f_{\epsilon^{(k)}}) \leq \Lambda_{min}(\Sigma_k) \leq \Lambda_{max}(\Sigma_k) \leq 2\pi \mathcal{M}(f_{\epsilon^{(k)}}).
\end{split}
\end{equation*}

We define $\mathcal{M}_\epsilon := \mathrm{max}_k \mathcal{M}(f_{\epsilon^{(k)}})$ and $\mathfrak{m}_\epsilon := \mathrm{min}_k \mathfrak{m}(f_{\epsilon^{(k)}})$. For two matrices $A$ and $B$, the inner product of $A$ and $B$ is defined as $<A, B> := \sum_{i,j} (AB^{'})_{ij}$.

\section{Useful Lemmas with Proofs}\label{sec:1}

\begin{lemma}\label{lemma1}
Consider model \eqref{sec1_eq1}. Recall the following notation $2\pi\mathcal{M}=\mathop{\mathrm{max}}_k\Lambda_{max}(\Upsilon_{n_k}^{X^{(k)}})$, $\Gamma_k = Cov(X_0^{(k)}, X_0^{(k)})$. Suppose $v_0,\cdots,v_K\in \mathbb{R}^{p}$. We have that,

\begin{equation}\label{eq1}
  \mathbb{P}\left( \frac{1}{N} \left| \sum_{\substack{k=0,\cdots,K \\ i=1,\cdots,n_k}} v_{k}^{'}[(X_i^{(k)})(X_i^{(k)})^{'}-\Gamma_k]v_k \right| \geq 2\pi\mathcal{M} \mathop{\mathrm{max}}_k(\Vert v_k \Vert_2^2) t \right) \leq 2\mathop{\mathrm{exp}}[-cN\mathop{\mathrm{min}}\{t,t^2\}],
\end{equation}
\begin{equation}\label{eq2}
\begin{split}
    &\mathbb{P}\left( \frac{1}{N} \left| \sum_{\substack{k=0,\cdots,K \\ i=1,\cdots,n_k}} u_{k}^{'}[(X_i^{(k)})(X_i^{(k)})^{'}-\Gamma_k]v_k \right| \geq 6\pi\mathcal{M} (\mathop{\mathrm{max}}_k(\Vert v_k \Vert_2^2)+\mathop{\mathrm{max}}_k(\Vert u_k \Vert_2^2)) t \right) \\ 
    &\quad \quad \quad \quad \quad \quad
    \leq 6\mathop{\mathrm{exp}}[-cN\mathop{\mathrm{min}}\{t,t^2\}].
\end{split}
\end{equation}
\end{lemma}

\noindent
\textbf{Proof.} Let $V_{i}^{(k)} = (X_i^{(k)})^{'}v_k$ and $Q=\mathop{\mathrm{Var}}(V_{1}^{(0)},\cdots,V_{n_0}^{(0)},V_{1}^{(1)},\cdots,V_{n_2}^{(1)},\cdots,V_{1}^{(K)},\cdots,V_{n_K}^{(K)})$. The entry of Q is shown as follow:

\begin{equation*}
\begin{split}
    \mathrm{Cov}(V_i^{(k)},V_j^{(k)}) = v_k^{'}\Gamma_{\Bar{X}^{(k)}}(i-j)v_k; \quad \mathrm{Cov}(V_i^{(k_1)},V_j^{(k_2)}) = 0.
\end{split}
\end{equation*}

% \newpage

Define $Q^{(k)} = \mathrm{Var}(V_1^{(k)},\cdots,V_{n_k}^{(k)})$. We can see that $Q$ is a block diagonal matrix,
\begin{equation*}
    Q =
    \begin{pmatrix}
    Q^{(0)}   &   0   & \cdots & 0  \\
    0     &  Q^{(1)}  & \cdots & 0  \\
    \vdots &\vdots &\ddots &\vdots\\
    0     & \cdots & \cdots & Q^{(K)}
    \end{pmatrix}.
\end{equation*}

%\newpage

For $Q^{(k)}$ and any $\Vert w \Vert_2 = 1$, we have
\begin{equation}\label{lemma1_eq1}
\begin{split}
w^{'}Q^{(k)} w = \sum_{r=1}^{n_k} \sum_{s=1}^{n_k} w_rw_s Q_{rs}^{(k)}
&= \sum_{r=1}^{n_k} \sum_{s=1}^{n_k} w_rw_s (v^{(k)})^{'} \Gamma_{\Bar{X}^{(k)}}(r-s) v^{(k)}  \\
&=(w\otimes v^{(k)})^{'}\Upsilon_{n_k}^{X^{(k)}} (w\otimes v^{(k)})\\
&\leq \Lambda_{max}(\Upsilon_{n_k}^{X^{(k)}})\Vert v_k\Vert_2^2
\leq
2\pi \mathcal{M}(f_{X^{(k)}})\Vert v_k\Vert_2^2.
\end{split}
\end{equation}

Since $\Vert Q \Vert_2 \leq \mathrm{max}_k(\Vert Q^{(k)}\Vert_2)$, we obtain $\Vert Q \Vert_2 \leq 2\pi\mathcal{M}\mathrm{max}_k( \Vert v_k\Vert_2^2)$. By the Hanson–Wright inequality \citep{rudelson2013hanson}
\begin{equation}\label{lemma1_eq2}
      \mathbb{P}\left( \frac{1}{N} \left| \sum_{\substack{k=0,\cdots,K \\ i=1,\cdots,n_k}} v_{k}^{'}[(X_i^{(k)})(X_i^{(k)})^{'}-\Gamma_k]v_k \right| \geq \eta \right) \leq 2\mathop{\mathrm{exp}}[-c\mathop{\mathrm{min}}\{\frac{ N^2\eta^2}{\Vert Q\Vert_F^2},\frac{N\eta}{\Vert Q \Vert_2}\}].
\end{equation}
Setting $\eta = \Vert Q \Vert_2 t$ and using $\Vert Q \Vert_F^2 \leq N\Vert Q\Vert_2$, we get \eqref{eq1}. 

To prove \eqref{eq2}, notice that 
\begin{equation*}
\begin{split}
     \frac{1}{N} \left| \sum_{\substack{k=1,\cdots,K \\ i=1,\cdots,n_k}} u_{k}^{'}[(X_i^{(k)})(X_i^{(k)})^{'}-\Gamma_k]v_k \right|
     &\leq
      \frac{1}{N} \left| \sum_{\substack{k=1,\cdots,K \\ i=1,\cdots,n_k}} u_{k}^{'}[(X_i^{(k)})(X_i^{(k)})^{'}-\Gamma_k]u_k \right|\\
      &+
       \frac{1}{N} \left| \sum_{\substack{k=1,\cdots,K \\ i=1,\cdots,n_k}} v_{k}^{'}[(X_i^{(k)})(X_i^{(k)})^{'}-\Gamma_k]v_k \right|\\
       &+
       \frac{1}{N} \left| \sum_{\substack{k=1,\cdots,K \\ i=1,\cdots,n_k}} (u_{k}^{'} + v_{k}^{'})[(X_i^{(k)})(X_i^{(k)})^{'}-\Gamma_k](u_k + v_k) \right|.
\end{split}
\end{equation*}
By applying \eqref{eq1} on each of three terms separately, we get \eqref{eq2}.
% \end{proof}

\begin{lemma}\label{lemma2}
Consider model \eqref{sec1_eq1}, we have
\begin{equation*}
    \mathbb{P}\left( \frac1N\Vert\sum_{k=0,\cdots,K} (\mathcal{X}^{(k)})^{'}\mathcal{E}^{(k)} \Vert_{\infty} \geq 6\pi(\mathcal{M}+\mathcal{M}_\epsilon)t \right)
    \leq
    6p^2\mathrm{exp}(-cN\mathrm{min}\{t,t^2\}).
\end{equation*}
\end{lemma}

\noindent
\textbf{Proof.}  Note that
\begin{equation*}
\begin{split}
  \frac1N\Vert\sum_{k=0,\cdots,K} (\mathcal{X}^{(k)})^{'}\mathcal{E}^{(K)}    \Vert_{\infty} 
  &= 
  \frac{1}{N}\Vert
  \sum_{\substack{k=0,\cdots,K \\ i=1,\cdots,n_k}} (X_{i}^{(k)})(\epsilon_i^{(k)})^{'} \Vert_{\infty}. 
\end{split}
\end{equation*}
Next we get the upper bound for 
\begin{equation*}
    \mathbb{P}\left( \frac1N \left|\sum_{\substack{k=0,\cdots,K \\ i=1,\cdots,n_k}} u^{'}X_{i}^{(k)}(\epsilon_i^{(k)})^{'}v \right|\geq t \right),
\end{equation*}
where $u,v\in\mathbb{R}^{p}$ and $\Vert u\Vert_2,\Vert v\Vert_2=1$. Since $\mathrm{Cov}((X_{i}^{(k)}){'}u,(\epsilon_i^{(k)})^{'}v)=0$, we have the following decomposition:
\begin{equation*}
\begin{split}
    &\frac1N \left(\sum_{\substack{k=0,\cdots,K \\ i=1,\cdots,n_k}} u^{'}X_{i}^{(k)}(\epsilon_i^{(k)})^{'}v \right)
    =
    \underbrace{\left[ \frac1N\sum_{\substack{k=0,\cdots,K \\ i=1,\cdots,n_k}}((X_{i}^{(k)}){'}u+(\epsilon_i^{(k)})^{'}v)^2
    -\mathrm{Var}((X_{i}^{(k)}){'}u+(\epsilon_i^{(k)})^{'}v)
    \right]}_{(a)}\\
    &\quad \quad-
    \underbrace{\left[ \frac1N\sum_{\substack{k=0,\cdots,K \\ i=1,\cdots,n_k}}((X_{i}^{(k)}){'}u)^2
    -\mathrm{Var}((X_{i}^{(k)}){'}u)
    \right]}_{(b)}
    -
    \underbrace{\left[ \frac1N\sum_{\substack{k=0,\cdots,K \\ i=1,\cdots,n_k}}((\epsilon_i^{(k)})^{'}v)^2
    -\mathrm{Var}((\epsilon_i^{(k)})^{'}v)
    \right]}_{(c)}.
\end{split}
\end{equation*}
We can apply \eqref{eq1} to obtain that
\begin{equation*}
      \mathbb{P}\left( \frac{1}{N} \left| \sum_{\substack{k=0,\cdots,K \\ i=1,\cdots,n_k}} u^{'}[(X_{i}^{(k)})(X_{i}^{(k)})^{'}-\Gamma_k]u \right| \geq 2\pi\mathcal{M}  t \right) \leq 2\mathop{\mathrm{exp}}[-cN\mathop{\mathrm{min}}\{t,t^2\}],
\end{equation*}
\begin{equation*}
      \mathbb{P}\left( \frac{1}{N} \left| \sum_{\substack{k=0,\cdots,K \\ i=1,\cdots,n_k}} v^{'}[(\epsilon_i^{(k)})(\epsilon_i^{(k)})^{'}-\Gamma_k]v \right| \geq 2\pi\mathcal{M}_\epsilon  t \right) \leq 2\mathop{\mathrm{exp}}[-cN\mathop{\mathrm{min}}\{t,t^2\}].
\end{equation*}
For $(a)$, similar to the proof of Lemma \ref{lemma1}, we let 
\begin{equation*}
\begin{split}
&Q = \mathrm{Var}((X_{1}^{(0)})^{'}u +(\epsilon_0^{(0)})^{'}v  ,\cdots,(X_{n_K}^{(K)})^{'}u+(\epsilon_{n_K}^{(K)})^{'}v),   \\
&Q_k = \mathrm{Var}((X_{1}^{(k)})^{'}u +(\epsilon_0^{(k)})^{'}v  ,\cdots,(X_{n_k}^{(k)})^{'}u+(\epsilon_{n_k}^{(k)})^{'}v),\quad k=1,\cdots,K.
\end{split}  
\end{equation*}
Since $Q$ is a block diagonal matrix
\begin{equation*}
    Q =
    \begin{pmatrix}
    Q_0   &   0   & \cdots & 0  \\
    0     &  Q_1 & \cdots & 0  \\
    \vdots &\vdots &\ddots &\vdots\\
    0     & \cdots & \cdots & Q_K
    \end{pmatrix}.
\end{equation*}
We have that $\Vert Q\Vert_2 \leq \mathrm{max}_k\{\Vert Q_k\Vert_2\} $. For every $k$, we define
\begin{equation*}
\begin{split}
    &Q_r^{'} = \mathrm{Var}((X_{1}^{(k)})^{'}u -(\epsilon_0^{(k)})^{'}v  ,\cdots,(X_{n_r}^{(k)})^{'}u-(\epsilon_{n_k}^{(k)})^{'}v)\\
    &Q_{k,1} = \mathrm{Var}((X_{1}^{(r)})^{'}u   ,\cdots,(X_{n_k-h}^{(k)})^{'}u)\\
    &Q_{k,2} = \mathrm{Var}((\epsilon_0^{(k)})^{'}v  ,\cdots,(\epsilon_{n_k}^{(k)})^{'}v).
\end{split}
\end{equation*}
Note that $Q_r, Q_r^{'}, Q_{r,1}, Q_{r,2}$ are positive definite. Since $Q_r + Q_r^{'} = 2Q_{r,1}+ 2Q_{r,2} $, we have that $\Vert Q_r\Vert_2\leq 2\Vert Q_{r,1}\Vert_2+ 2\Vert Q_{r,2}\Vert_2$. Using the same method as \eqref{lemma1_eq1}, we get $\Vert Q_{r,1} \Vert_2 \leq \Lambda_{max}(\Upsilon_{n_k}^{X^{(r)}})$, $\Vert Q_{r,2} \Vert_2 \leq \Lambda_{max}(\Sigma_{\epsilon^{(r)}})$. Therefore, $\Vert Q\Vert_2\leq \mathrm{max}_r\{2\Vert Q_{r,1}\Vert_2 + 2\Vert Q_{r,2}\Vert_2\} \leq 4\pi\mathcal{M}^{*} + 4\pi\mathcal{M}_\epsilon\leq 4\pi\mathcal{M} + 4\pi\mathcal{M}_\epsilon$. Applying Hanson–Wright inequality \citep{rudelson2013hanson} again, we obtain that
\begin{equation*}
    \mathbb{P}\left(\left| (a)
    \right| \geq 4\pi(\mathcal{M}+\mathcal{M}_\epsilon)t\right)\leq
    2\mathop{\mathrm{exp}}[-cN\mathop{\mathrm{min}}\{t,t^2\}].
\end{equation*}
By the probability inequalities derived for $(a),(b),(c)$, we have that
\begin{equation*}
   \mathbb{P}\left(\frac1N \left|\sum_{\substack{k=0,\cdots,K \\ i=1,\cdots,n_k}} u^{'}X_{i}^{(k)}(\epsilon_i^{(k)})^{'}v\right|\geq 6\pi(\mathcal{M}+\mathcal{M}_\epsilon)t \right) \leq 6\mathop{\mathrm{exp}}[-cN\mathop{\mathrm{min}}\{t,t^2\}].
\end{equation*}
Let $e_j$ be a vector such that its $j$-th element is 1 and all other elements are zero. Observe that
\begin{equation*}
    \frac1N\Vert\sum_{k=0,\cdots,K} (\mathcal{X}^{(k)})^{'}\mathcal{E}^{(K)}    \Vert_{\infty}
    =
    \mathop{\mathrm{max}}_{\substack{ 1\leq r,s\leq p}}\frac{1}{N}\vert
    \sum_{\substack{k=1,\cdots,K \\ i=1,\cdots,n_k}} e_r^{'}(X_{i}^{(k)})(\epsilon_i^{(k)})^{'}e_s
    \vert.
\end{equation*}
Taking a union bound over $r,s$ yields the final result. 
% \end{proof}

\begin{lemma}\label{lemma3}
Consider model \eqref{sec1_eq1}, we have
\begin{equation*}
    \mathbb{P}\left( \frac{1}{n_i}\Vert(\mathcal{X}^{(k)})^{'}\mathcal{X}^{(k)} -\Gamma_k \Vert_{\infty} \geq 2\pi\mathcal{M}t \right)
    \leq
    6p^2\mathrm{exp}(-cn_i\mathrm{min}\{t,t^2\}).
\end{equation*}
\end{lemma}

\noindent
\textbf{Proof.} Similar to lemma \ref{lemma2}, we can prove that for $u, v \in \mathbb{R}^p$
\begin{equation*}
   \mathbb{P}\left(\frac{1}{n_i} \left|\sum_{\substack{i=1,\cdots,n_k}} u^{'}(X_{i}^{(k)}(X_i^{(k)})^{'} - \Gamma_k)v\right|\geq 2\pi\mathcal{M}t \Vert u \Vert_2  \Vert v \Vert_2\right) \leq 2\mathop{\mathrm{exp}}[-cn_i\mathop{\mathrm{min}}\{t,t^2\}].
\end{equation*}
Observe that
\begin{equation*}
    \frac{1}{n_i}\Vert \mathcal{X}^{(k)})^{'}\mathcal{X}^{(k)} -\Gamma_k    \Vert_{\infty}
    =
    \mathop{\mathrm{max}}_{\substack{ 1\leq r,s\leq p}}\frac{1}{n_i}\vert
     e_r^{'}(\sum_{\substack{ i=1,\cdots,n_k}} X_{i}^{(k)}(X_i^{(k)})^{'} - \Gamma_k)e_s
    \vert.
\end{equation*}
Taking a union bound over $r,s$ yields the final result. 
% \end{proof}

\begin{lemma}\label{lemma4} Consider model \eqref{sec1_eq1}. Under the conditions of Theorem \ref{th1}, we have
\begin{equation}%\label{th_lemma1_eq3}
\Vert \frac{1}{n_{\mathcal{A}_0}} \sum_{i \in \mathcal{A}_0}(S_i - \Bar{S})\mathcal{X}_i^{'} \mathcal{X}_i \Vert_{\infty} \leq c_{\mathcal{M}} \sqrt{\frac{\log p}{n_{\mathcal{A}_0}}}(1 + h^2)
\end{equation}
with high probability 
\end{lemma}

\noindent
\textbf{Proof.} Define $\delta_i := vec(S_i - \Bar{S})$. Notice that
\begin{equation*}
\begin{split}
vec\left(\frac{1}{n_{\mathcal{A}_0}} \sum_{i \in \mathcal{A}_0}(S_i - \Bar{S})\mathcal{X}_i^{'} \mathcal{X}_i\right)
&=
vec\left(\frac{1}{n_{\mathcal{A}_0}} \sum_{i \in \mathcal{A}_0}(S_i - \Bar{S})(\mathcal{X}_i^{'} \mathcal{X}_i - \Gamma_i)\right)\\
&=
\frac{1}{n_{\mathcal{A}_0}} \sum_{i \in \mathcal{A}_0} \delta_i^{'} (I_p \otimes \mathcal{X}_i^{'} \mathcal{X}_i  - I_p \otimes \Gamma_i).
\end{split}
\end{equation*}

Therefore, we have
\begin{equation}\label{lemma5_eq1}
\begin{split}
\left\Vert \frac{1}{n_{\mathcal{A}_0}} \sum_{i \in \mathcal{A}_0}(S_i - \Bar{S})\mathcal{X}_i^{'} \mathcal{X}_i \right\Vert_{\infty}
=
\mathop{\mathrm{max}}_{s\in \{1,\cdots,p^2\}}
\left\Vert \sum_{i\in \mathcal{A}_0} \alpha_k \delta_{k}^{'}[\frac{1}{n_k}(I_p \otimes \mathcal{X}_i^{'} \mathcal{X}_i)-I_p \otimes \Gamma_k]e_s \right\Vert_{\infty}.
\end{split}
\end{equation}

Similar to lemma \ref{lemma2}, we can prove that
\begin{equation}\label{lemma5_eq2}
\begin{split}
&\mathbb{P}\left(  \left| \sum_{k\in \mathcal{A}_0} \alpha_k u_{k}^{'}[\frac{1}{n_k}(I_p \otimes \mathcal{X}_i^{'} \mathcal{X}_i)-I_p \otimes \Gamma_k]v_k \right| \geq 6\pi\mathcal{M} (\mathop{\mathrm{max}}_k\Vert v_k \Vert_2^2 + \mathop{\mathrm{max}}_k\Vert v_k \Vert_2^2) t \right)\\
&\leq 6p\mathop{\mathrm{exp}}[-c n_{\mathcal{A}_0} \mathop{\mathrm{min}}\{t,t^2\}],  
\end{split}
\end{equation}
where $\alpha_k := \frac{n_k}{n_{\mathcal{A}_0}}$ and $u_k, v_k \in \mathbb{R}^{p^2}$. Recall that $\Vert \delta_k \Vert_1 \leq h $, and thus $\Vert \delta_k \Vert_2 \leq h $. Applying \eqref{lemma5_eq2} to \eqref{lemma5_eq1}, we arrive at
\begin{equation*}
\begin{split}
&\mathbb{P}\left(  \left| \sum_{i\in \mathcal{A}_0} \alpha_k \delta_{k}^{'}[\frac{1}{n_k}(I_p \otimes \mathcal{X}_i^{'} \mathcal{X}_i)-I_p \otimes \Gamma_k]e_s \right| \geq 6\pi\mathcal{M} (1 + h^2) t \right)\\
&\leq 6p\mathop{\mathrm{exp}}[-c n_{\mathcal{A}_0} \mathop{\mathrm{min}}\{t,t^2\}].
\end{split}
\end{equation*}

Setting $t = \sqrt{\frac{6\log p}{c n_{\mathcal{A}_0}}}$ and taking the union bound over $s$ yield the final result.
% \end{proof}

\section{Propositions with Their Proofs}\label{sec:2}
\begin{proposition}\label{prop1}
Consider model \eqref{sec1_eq1}. Define
\begin{equation*}
\phi((B_0,\Sigma_0) \cdots, (B_K,\Sigma_K) ) = \mathop{\mathrm{max}}_{0\leq i\leq K} \Lambda_{max}(\Sigma^{(i)})[1 + \frac{1+\mu_{max}(B_i)}{\mu_{min}(B_i)}].
\end{equation*}

We will write $\phi$ instead of $\phi((B_0,\Sigma_0) \cdots, (B_K,\Sigma_K) )$ when the meaning of the deterministic constant is clear from the context. Then, there exist universal positive constants $c_i > 0$ such that
\begin{itemize}
\item[(1)] for $N \gtrsim p$,
\begin{equation}\label{sec2_eq3}
\mathbb{P}\left[ \Vert \frac{1}{N}\sum_i \mathcal{X}_i^{'}\mathcal{E}_i
\Vert_2 > c_0 \phi \sqrt{\frac{p}{N}} \right]\leq c_1 \mathrm{exp}[-c_2 \log p]
\end{equation}
and for $N \gtrsim \log p$
\begin{equation}\label{sec2_eq4}
\mathbb{P}\left[ \mathop{\mathrm{max}}_i \frac{1}{N}\Vert \mathcal{X}_i^{'}\mathcal{E}_i \Vert_{\infty} \geq c_0\phi \sqrt{\frac{\log p}{N}}\right]
\leq
c_1 \mathrm{exp}[-c_2 \log p]
\end{equation}
\item[(2)] for $N \gtrsim p $,
\begin{equation}\label{sec2_eq5}
\mathbb{P}\left[ \Lambda_{min}(\frac{1}{N}\sum_i \mathcal{X}_i^{'} \mathcal{X}_i)
> \frac{\mathrm{min}_i \Lambda_{min}(\Sigma_i)}{2 \mathrm{max}_i \mu_{max}(B_i)}\right]
\leq
c_1 \mathrm{exp}[-c_2 \log p]
\end{equation}
and for $n_i \gtrsim \mathrm{max}\{1, t^{-2}\}\log p$
\begin{equation}\label{sec2_eq6}
\mathbb{P}\left[\frac{1}{n_i}\Vert \mathcal{X}_{\mathcal{A}_0}\Delta \Vert_F^2
\leq
\alpha \Vert \Delta \Vert_F^2 - \tau_{n_i} \Vert \Delta\Vert_1^2\right]
\leq
\mathrm{exp}\{-\frac{c}{2}n_{i}\mathrm{min}\{1,t^2\}\}
\end{equation} 
where $ \alpha = \pi \mathfrak{m}(f_{X^{(i)}})$, $t = \frac{\mathfrak{m}(f_{X^{(i)}})}{54 \mathcal{M}(f_{X^{(i)}})}$, $\tau_{n_i} = \frac{3\mathfrak{m}(f_{X^{(i)}})\log p^2}{cn_{i}\mathrm{min}\{1,t^2\}} $ and $c > 0$
\end{itemize}
\end{proposition}

\begin{proposition}\label{prop2} Let $\mathcal{M} := \mathrm{max}_{i\in\mathcal{A}_0}\mathcal{M}(f_{X^{(i)}})$ and  $\mathfrak{m} := \mathrm{min}_{i\in\mathcal{A}_0}\mathfrak{m}(f_{X^{(i)}})$. If $\mathfrak{m} >0 $ and $\mathcal{M} < \infty$, then there exists universal constants $c_3,c_4,c_5>0$ such that
\begin{itemize}
\item[(1)]For $n_{\mathcal{A}_0} \gtrsim \log p$
\begin{equation*}
\mathbb{P}\left[ \Vert \mathcal{X}_i^{'}\mathcal{E}_i \Vert_{\infty} \geq \phi_{\mathcal{A}_0} c_3\sqrt{\frac{\log p}{n_{\mathcal{A}_0}}} \right]
\leq
c_4 \mathrm{exp}[-c_5 \log p]
\end{equation*}
\item[(2)]For $n_{\mathcal{A}_0} \gtrsim \mathrm{max}\{1, t^{-2}\}\log p$
\begin{align}\label{sec3_eq4}
\mathbb{P}\left[\frac{1}{n_{\mathcal{A}_0}}\Vert \mathcal{X}_{\mathcal{A}_0}\Delta \Vert_F^2
\leq
\alpha_2 \Vert \Delta \Vert_F^2 - \tau_{n_{\mathcal{A}_0}} \Vert \Delta\Vert_1^2\right]
\leq
\mathrm{exp}\{-\frac{c}{2}n_{\mathcal{A}_0}\mathrm{min}\{1,t^2\}\}\\
\mathbb{P}\left[\frac{1}{n_{\mathcal{A}_0}}\Vert \mathcal{X}_{\mathcal{A}_0}\Delta \Vert_F^2
\geq
\alpha_2^{'} \Vert \Delta \Vert_F^2 + \tau_{n_{\mathcal{A}_0}} \Vert \Delta\Vert_1^2\right]
\leq
\mathrm{exp}\{-\frac{c}{2}n_{\mathcal{A}_0}\mathrm{min}\{1,t^2\}\}
\end{align} 
where $ \alpha_2 = \pi \mathfrak{m}$, $\alpha_2^{'} = 3\pi\mathcal{M}$, $t = \frac{\mathfrak{m}}{54 \mathcal{M}}$, $\tau_{n_{\mathcal{A}_0}} = \frac{3\mathfrak{m}\log p^2}{cn_{\mathcal{A}_0}\mathrm{min}\{1,t^2\}} $ and $c > 0$.
\end{itemize}
\end{proposition}

\begin{remark}
The assumption $\mathfrak{m}(f_{X^(i)}) >0 $ and $\mathcal{M}(f_{X^{(i)}}) < \infty$ are fairly mild and hold for stable, invertible ARMA processes. In our case, all auxiliary models are VAR model. Therefore, it is reasonable to assume that these models are uniformly bounded by some constant $\mathfrak{m}$ and $\mathcal{M}$.
\end{remark}

\subsection{Proof of Proposition \ref{prop1}}
\textbf{Proof of Proposition \ref{prop1}}
\eqref{sec2_eq3} and \eqref{sec2_eq5} are simple modifications of Proposition 3 in \cite{basu2019low}. Also, \eqref{sec2_eq6} is a special case of \eqref{sec3_eq4}. So we omit their proofs. We next prove \eqref{sec2_eq4}. Using Proposition 3.2 in \cite{basu2015regularized}, we know that
\begin{equation*}
\mathbb{P}\left[\frac{1}{n_i}\Vert \mathcal{X}_i \mathcal{E}_i \Vert_{\infty}
> c_0 \phi \eta \right]
\leq
6 p \mathrm{exp}[-c n_i \mathrm{min}\{ \eta^2, \eta \}].
\end{equation*}

With the choice of $\eta = \sqrt{\frac{(1 + c_1)\log p}{c N}}$, we arrive at
\begin{equation*}
\mathbb{P}\left[\frac{1}{N}\Vert \mathcal{X}_i \mathcal{E}_i \Vert_{\infty}
> c_0 \phi \frac{n_i}{N}\sqrt{\frac{\log p}{N}} \right]
\leq
6 p \mathrm{exp}[-(c_1+1) \frac{n_i\log p}{N}]
=
6 \mathrm{exp}[-c_1 \frac{n_i\log p}{N}].
\end{equation*}

Taking the union set over $i$, we have
\begin{equation*}
\begin{split}
&\mathbb{P}\left[ \mathop{\mathrm{max}}_i \frac{1}{N}\Vert \mathcal{X}_i \mathcal{E}_i  \Vert_{\infty} > c_0\phi \sqrt{\frac{\log p}{N}} \right] 
\leq \sum_i 6 \mathrm{exp}[-c_1\frac{n_i\log p}{N}]\\
&\quad \leq
\sum_i 6\frac{n_i}{N} \mathrm{exp}[-c_1 \log p]
= 6\mathrm{exp}[-c_1 \log p].
\end{split}
\end{equation*}

This implies that
\begin{equation}
\mathbb{P}\left[ \mathop{\mathrm{max}}_i \frac{1}{N}\Vert \mathcal{X}_i^{'}\mathcal{E}_i \Vert_{\infty} \geq c_0\phi \sqrt{\frac{\log p}{N}}\right]
\leq
c_1 \mathrm{exp}[-c_2 \log p]
\end{equation}
for some universal constants $c_i > 0$.
%end{proof}

\subsection{Proof of Proposition \ref{prop2}}
\noindent
\textbf{Proof of Proposition \ref{prop2}.} Let $\alpha_k = \frac{n_k}{n_{\mathcal{A}_0}}$. According to the definition of $\mathcal{M}$ and $\mathfrak{m}$, we have that

\begin{equation*}
2\pi \mathcal{M}
\geq
\Lambda_{max}\left(\sum_{k \in \mathcal{A}_0} \frac{n_k}{n_{\mathcal{A}_0}} \Gamma_k\right)
\geq
\Lambda_{min}\left(\sum_{k \in \mathcal{A}_0} \frac{n_k}{n_{\mathcal{A}_0}} \Gamma_k\right) \geq 2\pi \mathfrak{m}.
\end{equation*}

For every $u\in\mathbb{R}^{p^2}$, $\Vert u\Vert_2 \leq 1$, we obtain the following inequality from Lemma \ref{lemma2},
\begin{equation*}
    \mathbb{P}\left( \left|u^{'}\left(\sum_{k \in \mathcal{A}_0}\alpha_k[\frac{1}{n_k}(I_p \otimes \mathcal{X}_i^{'} \mathcal{X}_i)-I_p \otimes \Gamma_k]\right)u\right|\geq 2\pi\mathcal{M}t \right)\leq 2p\mathrm{exp}(-cn_{\mathcal{A}_0} \mathrm{min}\{t,t^2\}).
\end{equation*}
To simplify the notation, we define $\widehat{\Gamma}^{\mathcal{A}_0}:= \sum_{k \in\mathcal{A}_0}\alpha_k[\frac{1}{n_k}(I_p \otimes \mathcal{X}_i^{'} \mathcal{X}_i)]$. Then the inequality can be simplified as
\begin{equation*}
    \mathbb{P}\left( \left|u^{'}\left(\widehat{\Gamma}^{\mathcal{A}_0} - I_p\otimes\sum_{k \in \mathcal{A}_0 }\alpha_k^{'}\Gamma_k\right)u\right|\geq 2\pi\mathcal{M}t \right)\leq 2p\mathrm{exp}(-c n_{\mathcal{A}_0} \mathrm{min}\{t,t^2\}).
\end{equation*}
Applying Supplementary Lemma F.2 in \cite{basu2015regularized}, we have
\begin{equation}\label{prop2_eq0}
\begin{split}
    &\mathbb{P}\left( \frac{1}{n_{\mathcal{A}_0}}\mathop{\mathrm{sup}}_{u\in\mathcal{K}(2s)}\left| u^{'}\left(\widehat{\Gamma}^{\mathcal{A}_0} - I_p\otimes\sum_{k \in \mathcal{A}_0}\alpha_k\Gamma_k\right) u \right|\geq 2\pi\mathcal{M} t \right)\\
    &\quad \quad \leq
    2p\mathrm{exp}[-c n_{\mathcal{A}_0} \mathrm\{t,t^2\}+2s\mathrm{min}\{\log p^2,\log (\frac{21ep^2}{s})\}]\\
    &\quad \quad \leq
    2\mathrm{exp}(-cn_{\mathcal{A}_0}\mathrm{min}\{1,t^2\}+3s\log p^2).
\end{split}
\end{equation}
Setting $t=\frac{\mathfrak{m}}{54\mathcal{M}}$, we have
\begin{equation}\label{prop2_eq1}
    \mathop{\mathrm{sup}}_{u\in \mathcal{K}(2s)}\left| u^{'}\left(\widehat{\Gamma}^{\mathcal{A}_0} - I_p\otimes\sum_{k \in \mathcal{A}_0}\alpha_k\Gamma_k\right) u \right|
    \leq
    \frac{2\pi \mathfrak{m}}{54},
\end{equation}
with probability greater than $1-2\mathrm{exp}(-c n_{\mathcal{A}_0} \mathrm{min}\{1,t^2\}+3s\log p^2)$. Given \eqref{prop2_eq1}, applying supplementary Lemma 12 in \cite{loh2011high}, we have that for all $u\in\mathbb{R}^{p^2}$,
\begin{equation}\label{prop2_eq2}
    \left\vert u^{'}\left(\widehat{\Gamma}^{\mathcal{A}_0} - I_p\otimes\sum_{k \in \mathcal{A}_0}\alpha_k\Gamma_k\right)u\right\vert
    \leq
    \frac{2\pi \mathfrak{m}}{2}\Vert u\Vert_2^2 + \frac{2\pi \mathfrak{m}}{2s}\Vert u\Vert_1^2.
\end{equation}
Note that
\begin{equation}\label{prop2_eq3}
     2\pi \mathcal{M} \Vert u\Vert_2^2
     \geq
     u^{'}\left( I_p\otimes\sum_{k \in \mathcal{A}_0}\alpha_k\Gamma_k\right)u
     \geq 
     2\pi \mathfrak{m} \Vert u\Vert_2^2.
\end{equation}
By \eqref{prop2_eq2} and \eqref{prop2_eq3}, we have that
\begin{equation*}
    3\pi \mathcal{M}\Vert u\Vert_2^2 - \frac{\pi \mathfrak{m}}{s}\Vert u \Vert_1^2
    \geq
    u^{'} \widehat{\Gamma}^{\mathcal{A}_0} u \geq 
    \pi \mathfrak{m}\Vert u\Vert_2^2 - \frac{\pi \mathfrak{m}}{s}\Vert u \Vert_1^2,
    \quad \text{for all} \ u\in\mathbb{R}^{p^2},
\end{equation*}
with probability greater than $1-2\mathrm{exp}(-c n_{\mathcal{A}_0}\mathrm{min}\{1,t^2\}+3s\log p^2)$. Setting $s=\frac{c n_{\mathcal{A}_0}\mathrm{min}\{1,t^2\}}{6\log (p^2)}$ yields the final result.
% \end{proof}

\begin{lemma}\label{th_lemma1}
Recall that $\Bar{S} := (\sum_{i \in \mathcal{A}_0} \Gamma_i)^{-1}(\sum_{i \in \mathcal{A}_0} \Gamma_i S_i)$ and $\tilde{S}$ defined in Algorithm~1. Under the assumptions of Theorem \ref{th2}, we have
\begin{equation*}
\begin{split}
&\Vert \tilde{S} - \Bar{S} \Vert_F^2 \lesssim \frac{s \log p}{n_{\mathcal{A}_0}}(1 \vee h^4) + \sqrt{\frac{\log p}{n_{\mathcal{A}_0}}} h + \frac{s\log p + rp}{N} \\
&\Vert \tilde{S} - \Bar{S} \Vert_1 \lesssim
s\sqrt{\frac{\log p}{n_{\mathcal{A}_0}}}(1 \vee h^2) + 2h + \frac{pr + s\log p}{N} \sqrt{\frac{n_{\mathcal{A}_0}}{\log p}}
\end{split}
\end{equation*}
\end{lemma}

\textbf{proof}
According to \eqref{alg_eq2}, we have
\begin{equation}\label{th_lemma1_eq1}
\frac{1}{2 n_{\mathcal{A}_0}}\sum_{i \in \mathcal{A}} \Vert \mathcal{Y}_i - \mathcal{X}_i(\widehat{L} + \tilde{S}) \Vert_F^2 + \mu \Vert \tilde{S} \Vert_1
\leq
\frac{1}{2 n_{\mathcal{A}_0}}\sum_{i \in \mathcal{A}} \Vert \mathcal{Y}_i - \mathcal{X}_i(\widehat{L} + \Bar{S}) \Vert_F^2 + \mu \Vert \Bar{S} \Vert_1
\end{equation}

After some algebra
\begin{equation}\label{th_lemma1_eq2}
\begin{split}
\frac{1}{2n_{\mathcal{A}_0}} \sum_{i \in \mathcal{A}} \Vert \mathcal{X}_i(\tilde{S} - \Bar{S}) \Vert_F^2
&\leq
\frac{1}{n_{\mathcal{A}_0}}\sum_{i \in \mathcal{A}} < \tilde{S} - \Bar{S} , \mathcal{X}_i^{'}\mathcal{E}_1 > + \frac{1}{n_{\mathcal{A}_0}}\sum_{i \in \mathcal{A}}
<\mathcal{X}_i(\tilde{S} - \Bar{S} ) , \mathcal{X}_i(L - \widehat{L}) >\\
&\quad +
\frac{1}{n_{\mathcal{A}_0}} \sum_{i \in \mathcal{A}}<\mathcal{X}_i(\tilde{S} - \Bar{S}), \mathcal{X}_i(S_i - \Bar{S}) > +  \mu\Vert \Bar{S} \Vert_1 - \mu \Vert 
\tilde{S} \Vert_1\\
&\leq
\Vert \tilde{S} - \Bar{S} \Vert_1 \Vert \frac{1}{n_{\mathcal{A}_0}}\sum_{i \mathcal{A}}\mathcal{X}_i^{'}\mathcal{E}_i\Vert_{\infty} + \frac{1}{4 n_{\mathcal{A}_0}}\sum_{i \in \mathcal{A}} \Vert \mathcal{X}_i (\tilde{S} - \widehat{S})\Vert_F^2 + \frac{1}{n_{\mathcal{A}_0}} \sum_{i\in\mathcal{A}}\Vert \mathcal{X}_i(L-\widehat{L}) \Vert_F^2\\
&\quad
+ \Vert  \tilde{S} - \Bar{S} \Vert_1 \Vert \frac{1}{n_{\mathcal{A}_0}} \sum_{i \in \mathcal{A}}(S_i - \Bar{S})\mathcal{X}_i^{'} \mathcal{X}_i \Vert_{\infty} 
+ \mu \Vert \Bar{S} \Vert_1 - \mu \Vert \tilde{S} \Vert_1
\end{split}
\end{equation}
According to lemma \ref{lemma4}, we know that
\begin{equation}\label{th_lemma1_eq3}
\Vert \frac{1}{n_{\mathcal{A}_0}} \sum_{i \in \mathcal{A}_0}(S_i - \Bar{S})\mathcal{X}_i^{'} \mathcal{X}_i \Vert_{\infty} \leq c_{\mathcal{M}} \sqrt{\frac{\log p}{n_{\mathcal{A}_0}}}(1 + h^2)
\end{equation}
From proposition \ref{prop1}, we know that
\begin{equation}\label{th_lemma1_eq4}
\Vert \frac{1}{n_{\mathcal{A}_0}} \sum_{i \in \mathcal{A}_0} \mathcal{X}_i \mathcal{E}_i \Vert_{\infty} 
\leq
c_0 \phi \sqrt{\frac{\log p}{n_{\mathcal{A}_0}}}.
\end{equation}
Inserting \eqref{th_lemma1_eq3} and \eqref{th_lemma1_eq4} into \eqref{th_lemma1_eq2} and setting $\mu = 2(c_0 + c_\Sigma)(1 \vee h^2)\sqrt{\frac{\log p}{n_{\mathcal{A}_0}}}$, we have
\begin{equation}\label{th_lemma1_eq5}
\begin{split}
\frac{1}{4n_{\mathcal{A}_0}}\sum_{i \in \mathcal{A}} \Vert \mathcal{X}_i (\tilde{S} - 
\Bar{S} ) \Vert_F^2
&\leq
\frac{1}{2}\mu \Vert \tilde{S} - \Bar{S} \Vert_1 + \mu \Vert \Bar{S} \Vert_1 -\mu \Vert \tilde{S} \Vert_1 + \frac{pr + s\log p}{N}\\
&\leq
\frac{3 \mu}{2}\Vert \tilde{S} - \Bar{S} \Vert_{1,M_1} - \frac{\mu}{2}\Vert \tilde{S} - \Bar{S} \Vert_{1,M_1^c} + 2\mu \Vert \Bar{S} \Vert_{1,M_1^c} + \frac{pr s\log p}{N}\\
&\leq
\frac{3 \mu}{2}\Vert \tilde{S} - \Bar{S} \Vert_{1,M_1} - \frac{\mu}{2}\Vert \tilde{S} - \Bar{S} \Vert_{1,M_1^c} + 2\mu h + \frac{pr + s\log p}{N}
\end{split}
\end{equation}
(\RN{1}) If $\Vert \tilde{S} - \Bar{S} \Vert_{1,M_1} \geq 2h + \frac{pr + s\log p}{N \mu}$, we arrive at,
\begin{equation}\label{th_lemma1_eq6}
\frac{1}{4n_{\mathcal{A}_0}} \sum_{i \in \mathcal{A}}\Vert \mathcal{X}_i(\tilde{S} - \Bar{S}) \Vert_F^2
\leq
\frac{5\mu}{2}\Vert \tilde{S} - \Bar{S} \Vert_{1,M_1} - \frac{\mu}{2} \Vert \tilde{S} - \Bar{S} \Vert_{1,M_1^c}
\end{equation}
This implies $\Vert \tilde{S} - \Bar{S} \Vert_1\leq 6 \Vert \tilde{S} - \Bar{S} \Vert_{1,M_1}\leq 6\sqrt{s} \Vert \tilde{S} - \Bar{S} \Vert_F $. Using RE condition again, we have $\frac{1}{4n_{\mathcal{A}_0}}\sum_{i\in \mathcal{A}}\Vert \mathcal{X}_i(\tilde{S} - \Bar{S}) \Vert_F^2 \geq \frac{\alpha}{8}\Vert \tilde{S} - \Bar{S} \Vert_F^2$. Using \eqref{th_lemma1_eq6}, we arrive at $\frac{\alpha}{8}\Vert 
 \tilde{S} - \Bar{S} \Vert_F^2 \leq \frac{5\mu}{2}\sqrt{s}\Vert 
 \tilde{S} - \Bar{S} \Vert_F $, which implies
\begin{equation}
\Vert \tilde{S} - \Bar{S} \Vert_F^2 \lesssim \frac{s\log p}{n_{\mathcal{A}_0}}(1\vee h^4), \quad 
\Vert \tilde{S} - \Bar{S} \Vert_1 \lesssim s\sqrt{\frac{\log p}{n_{\mathcal{A}_0}}}(1 \vee h^2)
\end{equation}

(\RN{2}) If $\Vert \tilde{S} - \Bar{S} \Vert_{1,M_1} \leq 2h + \frac{pr + s\log p}{N \mu}$, from \eqref{th_lemma1_eq5} we know that $\Vert \tilde{S} - \Bar{S} \Vert_{1,M_1^c} \lesssim 2h + \frac{pr + s\log p}{N \mu}$. This implies $\Vert \tilde{S} - \Bar{S} \Vert_1 \lesssim h + \frac{pr + s\log p}{N \mu}$. Applying RSC condition again for \eqref{th_lemma1_eq5}, we arrive at,
\begin{equation}
\begin{split}
\frac{\alpha}{4}\Vert \tilde{S} - \Bar{S} \Vert_F^2
&\lesssim
\frac{\log p}{n_{\mathcal{A}_0}}(h + \frac{pr + s\log p}{N \mu})^2 + h\mu + \frac{pr + s\log p}{N}\\
&\lesssim
\frac{\log p}{n_{\mathcal{A}_0}}h^2 + h\mu + \frac{pr + s\log p}{N}
\end{split}
\end{equation}
Combining (\RN{1}) and (\RN{2}), we have
\begin{equation}
\begin{split}
&\Vert \tilde{S} - \Bar{S} \Vert_F^2 \lesssim \frac{s \log p}{n_{\mathcal{A}_0}}(1 \vee h^4) + h\mu + \frac{pr + s\log p}{N}  \\
&\Vert \tilde{S} - \Bar{S} \Vert_1 \lesssim s\sqrt{\frac{\log p}{n_{\mathcal{A}_0}}}(1 \vee h^2) + h + \frac{pr + s\log p}{N} \sqrt{\frac{n_{\mathcal{A}_0}}{\log p}}.
\end{split}
\end{equation}
%end{proof}

\begin{lemma}\label{th_th_lemma2}
Define $\delta = S_0 - \Bar{S}$. Under the assumptions of Theorem \ref{th2}, we have

\begin{equation*}
\begin{split}
&\Vert \tilde{\delta} - \delta \Vert_F^2
\lesssim
h\sqrt{\frac{\log p}{n_0} } \wedge h^2 + (1 \vee h^4) \frac{s \log p}{n_{\mathcal{A}_0}} + \frac{n_{\mathcal{A}_0}(pr + s\log p)^2 }{n_0 N^2}\\
&\Vert \tilde{\delta} - \delta \Vert_1
\lesssim
h + (1 \vee h^4) \frac{s \sqrt{n_0\log p}}{n_{\mathcal{A}_0}} + \frac{n_{\mathcal{A}_0}(pr + s\log p)^2 }{\sqrt{n_0\log p} N^2}.\\
\end{split}
\end{equation*}
\end{lemma}

\begin{remark}
As shown in  Theorem \ref{th1}, the upper bound of estimation error for $S_0$ is $\frac{s \log p + rp}{n_0} $ without transfer learning algorithm. Under the above assumption, considering other informative set improves estimation result when $h = o(\frac{ s \log p + rp  }{\sqrt{n_0 \log p}} \vee \left(\frac{n_{\mathcal{A}_0}}{n_0}\right)^{\frac{1}{4}})$. For traditional lasso method, the estimation rate is $O_p(\frac{s\log p}{n_0})$. Transfer learning has a faster convergence rate when $h\lesssim s\sqrt{\frac{\log (p)}{n_0}}$ and $N \gtrsim \sqrt{\frac{n_{\mathcal{A}_0}}{s \log p}}(pr + s \log p)$.

\end{remark}

\textbf{Proof}
According to \eqref{alg_eq3}, we have
\begin{equation}\label{th_th_lemma2_eq1}
\begin{split}
\frac{1}{2n_0} \Vert \mathcal{Y}_i - \mathcal{X}_i(\widehat{L} + \tilde{S} + \tilde{\delta}) \Vert_F^2 + \lambda_{\delta} \Vert \tilde{\delta} \Vert_1
\leq
\frac{1}{2n_0} \Vert \mathcal{Y}_i - \mathcal{X}_i(\widehat{L} + \tilde{S} + \delta) \Vert_F^2 + \lambda_{\delta} \Vert \delta \Vert_1
\end{split}
\end{equation}   
Similar to \eqref{th_lemma1_eq2} and \eqref{th_lemma1_eq4}, we have
\begin{equation}\label{th_th_lemma2_eq2}
\begin{split}
\frac{1}{4 n_0} \Vert \mathcal{X}_1(\tilde{\delta} - \delta) \Vert_F^2
\leq
2 \lambda_{\delta} \Vert \delta \Vert_1
- \frac{\lambda_\delta}{2}\Vert \tilde{\delta} -\delta \Vert_1
+ \frac{pr + s\log p}{N} + \frac{2}{n_0}\Vert \mathcal{X}(\tilde{S} - \Bar{S}) \Vert_F^2
\end{split}
\end{equation}
If (\RN{1}) $\Vert \tilde{S} - \Bar{S} \Vert_{1,M_1} \geq 2h + \frac{pr + s\log p}{N \mu}$, we know that $\Vert \tilde{S} -\Bar{S} \Vert_1 \leq  6\sqrt{s} \Vert \tilde{S} - \Bar{S} \Vert_F$ from lemma \ref{th_lemma1}. Using RSC condition, we have
\begin{equation}
\begin{split}
\frac{2}{n_0}\Vert \mathcal{X}_1(\tilde{S}- \Bar{S}) \Vert_F^2
&\lesssim \Vert \widehat{S} - \Bar{S} \Vert_F^2 + \frac{\log p}{n_0} \Vert \widehat{S} - \Bar{S} \Vert_1^2\\
&\lesssim
(1 + 36 s \frac{\log p}{n_{\mathcal{A}_0}})\Vert \tilde{S}- \Bar{S} \Vert_F^2
\end{split}
\end{equation}
If (\RN{2}) $\Vert \tilde{S} - \Bar{S} \Vert_{1,M_1} \leq 2h + \frac{pr + s\log p}{N \mu}$, using RE condition again we have
\begin{equation}
\begin{split}
\frac{2}{n_0}\Vert \mathcal{X}_1(\tilde{S} - \Bar{S}) \Vert_F^2
&\lesssim
\Vert \tilde{S} - \Bar{S} \Vert_F^2 + \frac{\log p}{n_0}(2h + \frac{pr + s\log p}{N\mu})^2\\
&\lesssim
\Vert \tilde{S} - \Bar{S} \Vert_F^2 + \frac{\log p}{n_0}h^2 +  \frac{n_{\mathcal{A}_0}(pr + s\log p)^2}{N^2 n_0}
\end{split}
\end{equation}
Inserting $ (\RN{1})$ and $(\RN{2})$ into \eqref{th_th_lemma2_eq2}, we arrive at
\begin{equation}\label{th_th_lemma2_eq3}
\begin{split}
\frac{1}{4n_0}\Vert \mathcal{X}_1(\tilde{\delta} - \delta) \Vert_F^2 \vee \frac{\lambda_\delta}{2}\Vert \tilde{\delta} - \delta \Vert_1
&\lesssim
2h\sqrt{\frac{\log p}{n_0}} + \Vert \tilde{S} - \Bar{S} \Vert_F^2 
+
\frac{\log p}{n_0}h^2 +  \frac{n_{\mathcal{A}_0}(pr + s\log p)^2}{N^2 n_0}\\
&\lesssim
2h\sqrt{\frac{\log p}{n_0}} + \frac{s\log p}{n_{\mathcal{A}_0}}(1 \vee h^4)
+  \frac{n_{\mathcal{A}_0}(pr + s\log p)^2}{N^2 n_0}
\end{split}
\end{equation}
Using RE condition again for $\mathcal{X}_1(\tilde{\delta} - \delta)$, we have
\begin{equation}\label{th_th_lemma2_eq4}
\Vert \tilde{\delta} -\delta \Vert_F^2
\lesssim
h \sqrt{\frac{\log p}{n_0}} + \frac{s \log p}{n_{\mathcal{A}_0}}(1 \vee h^4)
+
\frac{n_{\mathcal{A}_0}(pr + s\log p)^2}{N^2 n_0}
\end{equation}
Inserting $\Vert \tilde{\delta} -\delta\Vert_F \leq \Vert \tilde{\delta} -\delta\Vert_1$ into \eqref{th_th_lemma2_eq3}, we have
\begin{equation}\label{th_th_lemma2_eq5}
\begin{split}
\Vert \tilde{\delta} -\delta \Vert_F^2
&\lesssim
h^2 + (\frac{s \log p}{n_{\mathcal{A}_0}}(1 \vee h^4))^2/\lambda_{\delta}
+
(\frac{n_{\mathcal{A}_0}(pr + s\log p)^2}{N^2 n_0})^2/\lambda_\delta\\
&<
h^2 + \frac{s\log p}{n_{\mathcal{A}_0}}(1 \vee h^4)
+  \frac{n_{\mathcal{A}_0}(pr + s\log p)^2}{N^2 n_0}
\end{split}
\end{equation}
Using \eqref{th_th_lemma2_eq4} and \eqref{th_th_lemma2_eq5} yields the final result.
%end{proof}

\section{Proof of Theorems}\label{sec:3}
\subsection{Proof of Theorem \ref{th1}}
\textbf{Proof}
According to \eqref{alg_eq1}, we have 
\begin{equation}\label{th1_eq1}
\sum_i \frac{1}{N}\Vert \mathcal{Y}_i - \mathcal{X}_i B_i \Vert_F^2
+
\lambda \Vert \widehat{L} \Vert_{*} + \sum_i \mu_i\sqrt{\frac{n_i}{N}} \Vert \widehat{S}_i \Vert_1
\leq
\sum_i \frac{1}{N}\Vert \mathcal{Y}_i - \mathcal{X}_i B_i \Vert_F^2
+
\lambda \Vert L \Vert_{*} + \sum_i \mu_i\sqrt{\frac{n_i}{N}} \Vert S_i \Vert_1
\end{equation}
Let $\Delta^L = \widehat{L} - L$ and $\Delta_i^S = \widehat{S}_i - S_i$. Using $\mathcal{Y}_i = \mathcal{X}_i(L+S_i) + \mathcal{E}_i$ and simple algebra, we have

\begin{equation}\label{th1_eq2}
\begin{split}
\frac{1}{N}\sum_i \Vert \mathcal{X}_i(\Delta^L + \Delta_i^S) \Vert_F^2
&\leq
\frac{2}{N}\sum_i <\widehat{\Delta}^L + \widehat{\Delta}_i^S, \mathcal{X}_i\mathcal{E}_i>
+ \lambda \Vert L \Vert_{*}
+\sum_i\sqrt{\frac{n_i}{N}} \mu_i \Vert S_i \Vert_1
-\lambda \Vert \widehat{L} \Vert_{*}
-\sum_i \sqrt{\frac{n_i}{N}}\mu_i \Vert \widehat{S}_i \Vert_1\\
&\leq
\sum_i \frac{2}{N}\Vert \Delta_i^S \Vert_1 \cdot \Vert \mathcal{X}_i^{'}\mathcal{E}_i \Vert_{\infty}
+
\Vert \widehat{\Delta}^L \Vert_{*} \left\Vert \frac{2}{N}\sum_i \mathcal{X}_i^{'}\mathcal{E}_i \right\Vert_2
+\lambda(\Vert \widehat{\Delta}_A^L\Vert_{*}
 - \Vert \widehat{\Delta}_B^L\Vert_{*} )\\
&\quad+ \sum_i\mu_i\sqrt{\frac{n_i}{N}}(\Vert \Delta_i^S \Vert_{1,M_i} - \Vert \Delta_i^S \Vert_{1,M_i^c})\\
&\leq
\sum_i \frac{\mu_i}{2}\sqrt{\frac{n_i}{N}}\Vert \Delta_i^S \Vert_1
+
\frac{\lambda}{2}\Vert \Delta^L \Vert_{*} + \lambda(\Vert \Delta_A^L \Vert_{*} - \Vert \Delta_B^L \Vert_{*})\\
&+
\sum_i\mu_i\sqrt{\frac{n_i}{N}}(\Vert \Delta_i^S \Vert_{1,M_i} - \Vert \Delta_i^S \Vert_{1,M_i^c})\\
&=
\frac{3\lambda}{2}\Vert \Delta_A^L \Vert_{*} - \frac{\lambda}{2}\Vert \Delta_B^L \Vert_{*}
+\sum_i \frac{3\mu_i}{2}\sqrt{\frac{n_i}{N}}\Vert \Delta_i^S \Vert_{1,M_i} - \sum_i \frac{\mu_i}{2}\sqrt{\frac{n_i}{N}}\Vert \Delta_i^S \Vert_{1,M_i^c}
\end{split}
\end{equation}
where the matrices $(A, B) \in \{(A,B): AB^{'} = 0 \ \text{and} \ A^{'}B = 0 \}$, $M_i$ and $M_i^c$ corresponds to non-zero entries and zero entries of matrices $L_i$ separately. The second inequality derives from Lemma 1 in \cite{agarwal2012noisy} and lemma 2.3 in \cite{recht2010guaranteed}. The third inequality derives from Proposition \ref{prop1}, $\mu_i = 2c_0\phi\sqrt{\frac{\log p}{N}} + \theta$ and $\lambda = 2c_0\phi \sqrt{\frac{p}{N}} $. Now, \eqref{th1_eq2} implies
\begin{equation}\label{th1_eq3}
\lambda\Vert \Delta_B^L \Vert_{*} +  \sum_i \mu_i\sqrt{\frac{n_i}{N}} \Vert \Delta_i^S\Vert_{1,M_i}
\leq
3\lambda\Vert \Delta_A^L \Vert_{*} + \sum_i \mu_i\sqrt{\frac{n_i}{N}} \Vert \Delta_i^S\Vert_{1,M_i^c}
\end{equation}

Using RSC conditions and $\tau^{'} \leq \tau $, we have
\begin{equation}\label{th1_eq4}
\begin{split}
\sum_i \frac{1}{N} \Vert \mathcal{X}_i(\Delta^L + \Delta_i^S) \Vert_F^2  
&=
\sum_i \frac{n_i}{N} \frac{1}{n_i} \Vert \mathcal{X}_i(\Delta^L + \Delta_i^S) \Vert_F^2  
\\
&=
\sum_{i}\frac{1}{N}\Vert \mathcal{X}_i \Delta^L \Vert_F^2 + \sum_{i}\frac{n_i}{N}\frac{1}{n_i}\Vert \mathcal{X}_i \Delta_i^S\Vert_F^2 + 
\sum_i\frac{n_i}{N}< \Delta^L, \frac{1}{n_i}\mathcal{X}_i\mathcal{X}_i^{'}\Delta_i^S >\\
&\geq
\alpha \Vert \Delta^L \Vert_F^2 - \tau^{'}\Vert \Delta^L \Vert_{*}^2
+\sum_i\frac{n_i}{N}(\Vert \Delta_i^S\Vert_F^2 - \tau\frac{\mu_i^2}{\lambda^2}\Vert \Delta_i^S \Vert_1^2) \\
&\quad - \Vert \Delta^L \Vert_{\infty}(\mathrm{max}_i\Vert \Gamma_i \Vert_{\infty} +\phi\sqrt{\frac{\log p}{n_i}})(\sum_i\frac{n_i}{N}\Vert \Delta_i^S \Vert_1)
\\
&\geq
\alpha \Vert \Delta^L \Vert_F^2 + \alpha \sum_i \frac{n_i}{N}\Vert \Delta_i^S \Vert_F^2 - 2\Vert \Delta^L \Vert_{\infty}(\sum_i\frac{n_i}{N}\Vert \Delta_i^S \Vert_1)\\
&\quad - 2\tau\Vert \Delta_A^L \Vert_{*}^2 - 2\tau\Vert \Delta_B^L \Vert_{*}^2
-2\tau\sum_i \frac{n_i}{N} \frac{\mu_i^2}{\lambda^2} \Vert \Delta_i^S \Vert_{1,M_i}^2
-2\tau\sum_i \frac{n_i}{N} \frac{\mu_i^2}{\lambda^2} \Vert \Delta_i^S \Vert_{1,M_i^c}^2\\
&\geq
\alpha \Vert \Delta^L \Vert_F^2 + \alpha \sum_i \frac{n_i}{N}\Vert \Delta_i^S \Vert_F^2 -2 \Vert \Delta^L \Vert_{\infty}(\sum_i \frac{n_i}{N}\Vert \Delta_i^S \Vert_1)\\
&\quad 
-2\tau ( \Vert \Delta_A^L \Vert_{*} + \sum_i\frac{\mu_i}{\lambda}\sqrt{\frac{n_i}{N}}\Vert \Delta_i^S \Vert_{1,M_i} )^2 -2\tau ( \Vert \Delta_B^L \Vert_{*} + \sum_i\frac{\mu_i}{\lambda}\sqrt{\frac{n_i}{N}}\Vert \Delta_i^S \Vert_{1,M_i^c} )^2,
\end{split}
\end{equation}
where  the first inequality comes from lemma \ref{lemma3} with the choice of $t = \sqrt{\frac{\log p}{n_i}}$, $\Vert \frac{1}{n_i}\mathcal{X}_i\mathcal{X}_i^{'} \Vert_{\infty}\leq \Vert \Gamma_i \Vert_{\infty} + \phi \sqrt{\frac{\log p}{n_i}} $.
Combining \eqref{th1_eq3} and \eqref{th1_eq4}, we arrive at
\begin{equation}\label{th1_eq5}
\begin{split}
\sum_i \frac{1}{N} \Vert \mathcal{X}_i(\Delta^L + \Delta_i^S) \Vert_F^2 
&\geq
\alpha \Vert \Delta^L \Vert_F^2 + \alpha \sum_i \frac{n_i}{N}\Vert \Delta_i^S \Vert_F^2 -2  \Vert \Delta^L \Vert_{\infty}(\sum_i \frac{n_i}{N}\Vert \Delta_i^S \Vert_1)\\
&\quad 
-20\tau ( \Vert \Delta_A^L \Vert_{*} + \sum_i\frac{\mu_i}{\lambda}\sqrt{\frac{n_i}{N}}\Vert \Delta_i^S \Vert_{1,M_i} )^2
\end{split}
\end{equation}
Since $\Delta_A^L$ has rank at most $2r$ and $M_i \leq s$, we have
\begin{equation}\label{th1_eq6}
\begin{split}
\Vert \Delta_A^L \Vert_{*} + \sum_i \frac{\mu_i}{\lambda}\sqrt{\frac{n_i}{N}}\Vert \Delta_i^S \Vert_{1,M_i}
&\leq
\sqrt{2r} \Vert \Delta^L \Vert_F + \sum_i \frac{\mu_i}{\lambda}\sqrt{s}\sqrt{\frac{n_i}{N}} \Vert \Delta_i^S \Vert_F\\
&\leq
(\Vert \Delta \Vert_F^2 + \sum_i\frac{n_i}{N} \Vert \Delta_i^S \Vert_F^2)^{\frac12}
(2r + s\sum_i \frac{\mu_i}{\lambda})^{\frac12}
\end{split}
\end{equation}
Recall that $\mu_i = 2c_0\phi\sqrt{\frac{\log p}{N}} +\theta$, $\lambda = 2c_0\phi \sqrt{\frac{p}{N}} $ and $\theta = o(\sqrt{\frac{p}{N}})$. Thus $\sum_i \frac{\mu_i}{\lambda} = o(1)$.  With \eqref{th1_eq5} and \eqref{th1_eq6}, we arrive at
\begin{equation}\label{th1_eq7}
\sum_i \frac{1}{N} \Vert \mathcal{X}_i(\Delta^L + \Delta_i^S) \Vert_F^2  
\geq
\frac{\alpha}{2} \Vert \Delta^L \Vert_F^2 + \frac{\alpha}{2} \sum_i \frac{n_i}{N}\Vert \Delta_i^S \Vert_F^2 -2 \Vert \Delta^L \Vert_{\infty}(\sum_i \frac{n_i}{N}\Vert \Delta_i^S \Vert_1)
\end{equation}
Inserting \eqref{th1_eq7} into  \eqref{th1_eq2}, we have
\begin{equation}\label{th1_eq8}
\begin{split}
\frac{\alpha}{2}\Vert \Delta^L \Vert_F^2 + \frac{\alpha}{2} \sum_i \frac{n_i}{N}\Vert \Delta_i^S \Vert_F^2
&\leq
\frac{3\lambda}{2}\Vert \Delta_A^L \Vert_{*} + \sum_i \frac{5\mu_i}{2} \sqrt{\frac{n_i}{N}}\Vert \Delta_i^S \Vert_{1,M_i}\\
&\leq
\frac{3\lambda}{2} \sqrt{2r} \Vert \Delta^L \Vert_F + \sum_i \frac{5\mu_i}{2}\sqrt{s} \sqrt{\frac{n_i}{N}} \Vert \Delta_i^S \Vert_F\\
&\leq
\sqrt{ (3\lambda\sqrt{2r})^2 + ( 5\mu \sqrt{s} )^2 } \cdot
\sqrt{ \frac{1}{2}\Vert \Delta^L \Vert_F^2 + \frac{1}{2}\sum_i \frac{n_i}{N} \Vert \Delta_i^S \Vert_F^2 }
\end{split}
\end{equation}
Using  $\theta = o(\sqrt{\frac{p}{N}})$ yields the final result.
%end{proof}

\subsection{Proof of Theorem \ref{th2}}
\textbf{Proof}
Using lemma \ref{th_lemma1} and lemma \ref{th_th_lemma2} yields Theorem \ref{th2} directly.
%end{proof}

\subsection{Proof of Theorem \ref{th3}}
\textbf{Proof}
Assumption \ref{assum10} implies $h = O(1)$. Using lemma \ref{th_lemma1}, we have
\begin{equation}\label{th3_eq1}
    \Vert \widehat{S}_k - (S_0 + \tilde{\delta}_k) \Vert_2^2
    \lesssim
    \frac{s\log(p)}{n_0/2+n_k} + \sqrt{\frac{\log(p)}{n_0/2+n_k}} + \frac{s\log p + rp}{N} 
\end{equation}
\begin{equation}\label{th3_eq2}
    \Vert \widehat{S}_k - (S_0 + \tilde{\delta}_k) \Vert_1
    \lesssim
    s\sqrt{\frac{\log(p)}{n_0/2+n_k}} + h +  \frac{pr + s\log p}{N} \sqrt{\frac{n_0/2+n_k}{\log p}}
\end{equation}

where 
\begin{equation*}
\begin{split}
\tilde{\delta}_k = [(\alpha_0\Gamma_{0} +\alpha_k\Gamma_{k})]^{-1}[\alpha_k\Gamma_{k}\delta^{(k)}], \\
\alpha_0 = \frac{n_0/2}{n_0/2 + n_k}, \alpha_k = \frac{n_k}{n_0/2 + n_k}
\end{split}
\end{equation*}

We can see that $\Vert \tilde{\delta}^{(k)}\Vert_2^2\asymp \Vert \delta^{(k)}\Vert_2^2$ and $\Vert\tilde{\delta}^{(k)}\Vert_1 \leq C_1 h$, where $C$ depends on $\Sigma_k$ and $\Sigma^{(0)}$.
From assumption \ref{assum10} and equation \eqref{th3_eq2}, we know that $\Vert \widehat{S}_k - S_0\Vert_1$ is bounded by $C_2$, $C_2$ depends on $\Sigma_k$ and $\Sigma_1$. We can prove $\Vert \widehat{S}_{0,\mathcal{I}} - S_0\Vert_1$ is bounded by $s\sqrt{\frac{\log(p^2)}{n_0/2}} + \frac{pr + s\log p}{N} \sqrt{\frac{n_0/2+n_k}{\log p}} $ for the same reason as \eqref{th3_eq2}. With the boundness of $\Vert\widehat{S}_{0,\mathcal{I}} - S_0\Vert_1$ and $\Vert \widehat{S}_k - S_0\Vert_1 $, we have that 
$\Vert\widehat{S}_{0,\mathcal{I}} - \widehat{S}_k\Vert_1\leq CM$.

According to the definition of $R^{(k)}$ and $R_1^{(0)}$,
\begin{equation*}
\begin{split}
R^{(k)}-R_1^{(0)} 
&= 
\Vert \mathcal{Y}_{0,\mathcal{I}^c} - \mathcal{X}_{0,\mathcal{I}^c}(\widehat{L}+\widehat{S}_k)\Vert_F^2
-
 \Vert \mathcal{Y}_{0,\mathcal{I}^c} - \mathcal{X}_{0,\mathcal{I}^c}(\widehat{L}+\widehat{S}_{0,\mathcal{I}})\Vert_F^2\\
&\leq
2<\mathcal{E}_{0,\mathcal{I}^c}^{'}\mathcal{X}_{0,\mathcal{I}^c},\widehat{S}_{0,\mathcal{I}} - \widehat{S}_k> +2
<(\mathcal{X}_{0,\mathcal{I}^c})(S_0 - \widehat{S}_{0,\mathcal{I}}),(\mathcal{X}_{0,\mathcal{I}^c})(\widehat{S}_{0,\mathcal{I}} - \widehat{S}_k)>\\
 &+
2 <(\mathcal{X}_{0,\mathcal{I}^c})(\widehat{S}_{0,\mathcal{I}} - \widehat{S}_k), (\mathcal{X}_{0,\mathcal{I}^c})  (\widehat{S}_{0,\mathcal{I}} - \widehat{S}_k)>
+
2\Vert \mathcal{X}_{0,\mathcal{I}^c}(\widehat{L} - L) \Vert_F^2
\end{split}
\end{equation*}

Using Proposition \ref{prop1} for the first term, we know that
\begin{equation*}
\begin{split}
<\mathcal{E}_{0,\mathcal{I}^c}^{'}\mathcal{X}_{0,\mathcal{I}^c},\widehat{S}_{0,\mathcal{I}} - \widehat{S}_k>
&\leq
\Vert (\mathcal{E}_{0,\mathcal{I}^c})^{'}\mathcal{X}_{0,\mathcal{I}^c} \Vert_\infty \Vert \widehat{S}_{0,\mathcal{I}} - \widehat{S}_k\Vert_1\\
&\lesssim
CM\sqrt{\frac{\log(p^2)}{n_0/2}}
\end{split}
\end{equation*}
with probability greater than $1-O(p^{-2})$.

For the second term, we have
\begin{equation*}
\begin{split}
\left| <(\mathcal{X}_{1,\mathcal{I}^c})(S_0 - \widehat{S}_{0,\mathcal{I}}),(\mathcal{X}_{0,\mathcal{I}^c})(\widehat{S}_{0,\mathcal{I}} - \widehat{S}_k)> \right|
&\leq
2<(\mathcal{X}_{0,\mathcal{I}^c})(S_0 - \widehat{S}_{0,\mathcal{I}}),(\mathcal{X}_{0,\mathcal{I}^c})(S_0 - \widehat{S}_{0,\mathcal{I}})>\\
&+
\frac{1}{2}<(\mathcal{X}_{0,\mathcal{I}^c})(\widehat{S}_{0,\mathcal{I}} - \widehat{S}_k),(\mathcal{X}_{0,\mathcal{I}^c})(\widehat{S}_{0,\mathcal{I}} - \widehat{S}_k)>
\end{split}
\end{equation*}

For the last term, $\Vert \mathcal{X}_{0,\mathcal{I}^c} (\widehat{L} - L)\Vert_F^2 \leq \frac{s\log p + rp}{N} $
Therefore,
\begin{equation}\label{th3_eq3}
\begin{split}
R^{(k)} - R_1^{(0)} \lesssim
\Vert (\mathcal{X}_{0,\mathcal{I}^c})(\widehat{S}_{0,\mathcal{I}} - \widehat{S}_k)\Vert_F^2
+
\Vert (\mathcal{X}_{0,\mathcal{I}^c})(S_0 - \widehat{S}_{0,\mathcal{I}})\Vert_F^2
+
\sqrt{\frac{\log p}{n_0/2}} + \frac{s \log p + rp}{N}
\end{split} 
\end{equation}
Similarly, we have
\begin{equation}\label{th3_eq3_1}
\begin{split}
R^{(k)} - R_1^{(0)} \gtrsim
\Vert (\mathcal{X}_{0,\mathcal{I}^c})(\widehat{S}_{0,\mathcal{I}} - \widehat{S}_k)\Vert_F^2
+
\Vert (\mathcal{X}_{0,\mathcal{I}^c})(S_0 - \widehat{S}_{0,\mathcal{I}})\Vert_F^2
+
\sqrt{\frac{\log p}{n_0/2}} + \frac{s \log p + rp}{N}
\end{split} 
\end{equation}

Using Proposition \ref{prop1} and the boundedness of $\Vert S_0 - \widehat{S}_{0,\mathcal{I}} \Vert_1$, $\Vert \widehat{S}_k - \widehat{S}_{0,\mathcal{I}} \Vert_1$, we have with propbability greater than $1-O(\mathrm{exp}(-n_0))$
\begin{equation}\label{th3_eq4}
\begin{split}
&\Vert \widehat{S}_{0,\mathcal{I}} - \widehat{S}_k\Vert_F^2 -M^2\frac{\log p}{n_0/2} - \frac{(s\log p + rp)^2}{N^2}
\lesssim
\Vert \mathcal{X}_{0,\mathcal{I}^c}(\widehat{S}_{0,\mathcal{I}} - \widehat{S}_k)\Vert_F^2\\
&\quad \quad \quad \quad \quad\lesssim
\Vert \widehat{S}_{0,\mathcal{I}} - \widehat{S}_k\Vert_F^2 + M^2\frac{\log p}{n_0/2} + \frac{(s\log p + rp)^2}{N^2}\\
&\Vert S_0 - \widehat{S}_{0,\mathcal{I}} \Vert_F^2 -M^2\frac{\log p}{n_0/2} - \frac{(s\log p + rp)^2}{N^2}
\lesssim
\Vert \mathcal{X}_{0,\mathcal{I}^c}(S_0 - \widehat{S}_{0,\mathcal{I}})\Vert_F^2\\
&\quad \quad \quad \quad \quad \lesssim
\Vert S_0 - \widehat{S}_{0,\mathcal{I}}\Vert_F^2 + M^2\frac{\log p}{n_0/2} +  \frac{(s\log p + rp)^2}{N^2}.
\end{split}
\end{equation}

Plugging \eqref{th3_eq4} in \eqref{th3_eq3} and \eqref{th3_eq3_1}, we arrive at
\begin{equation*}
\begin{split}
&\frac{1}{2}C_0\Vert\widehat{S}_{0,\mathcal{I}} - \widehat{S}_k\Vert_2^2
-
2C_0\Vert S_0 - \widehat{S}_{0,\mathcal{I}}\Vert_2^2
-
\sqrt{\frac{\log(p)}{n_0/2}}
-\frac{s \log p + rp}{N}
\\
&\leq R^{(k)} - R_1^{(0)} \leq
\\
&\frac{3}{2}C_1\Vert\widehat{S}_{0,\mathcal{I}} - \widehat{S}_k\Vert_2^2
+
2C_1\Vert S_0 - \widehat{S}_{0,\mathcal{I}}\Vert_2^2
+
\sqrt{\frac{\log(p)}{n_0/2}}
+\frac{s \log p + rp}{N}
\end{split}
\end{equation*}
with probability greater than $1-O(p^{-2})$. 
From $\frac{n_0(pr + s\log p)}{N \log p} = o(1)$, we know that $\sqrt{\frac{\log p}{n_0/2}} \gtrsim \frac{s\log p + rp}{N}$.
Note that $\Vert S_0 - \widehat{S}_{0,\mathcal{I}}\Vert_F\leq s\sqrt{\frac{\log p}{n_0}}$ and $ \Vert \widehat{S}_{0,\mathcal{I}}- \widehat{S}_k\Vert_F \leq \Vert\tilde{\delta}^{(k)}\Vert_F + \Vert S_0 - \widehat{S}_{0,\mathcal{I}}\Vert_F$. The upper bound for $R^{(k)} - R_1^{(0)}$ is
\begin{equation*}
R^{(k)} - R_1^{(0)}\lesssim \Vert \tilde{\delta}^{(k)} \Vert_F^2 + \sqrt{\frac{\log(p^2)}{n_0/2}}
\asymp 
\Vert \delta^{(k)} \Vert_F^2 + \sqrt{\frac{\log p}{n_0/2}}
\end{equation*}
The lower bound for $R^{(k)} - R_1^{(0)}$ is
\begin{equation*}
R^{(k)} - R_1^{(0)}\gtrsim \Vert \tilde{\delta}^{(k)} \Vert_F^2 + \sqrt{\frac{\log(p^2)}{n_0/2}}
\asymp 
\Vert \delta^{(k)} \Vert_F^2 + \sqrt{\frac{\log p}{n_0/2}}
\end{equation*}
Taking the union bound over $1\leq k\leq K$ yields the final result.
%end{proof}

\section{Additional Details For The Algorithms}\label{sec:4}
\subsection{Informative Set Selection Algorithm}
Shown in Algorithm \ref{algorithm2}
\begin{algorithm}
\caption{: Selecting Informative Set}\label{algorithm2}
\hspace*{\algorithmicindent} \textbf{Input} : observations from target model and auxiliary model $\{X_t^{(k)}\}, k=0,\cdots K$. \\
\hspace*{\algorithmicindent} \quad \quad \quad penalty parameters $\lambda_k, k=1,\cdots K$. Some constant $c > 0$.\\
\hspace*{\algorithmicindent} \quad \quad \quad low-rank matrix estimator $\widehat{L}$ from algorithm \ref{algorithm1}\\
\hspace*{\algorithmicindent} \textbf{Output} : Informative set $\widehat{A}$.

\begin{algorithmic}[0]
\State\textbf{Step 1} Split the target data into two parts $X_{0,\mathcal{I}}$ and $X_{0,\mathcal{I}^c}$, where $\mathcal{I}=\{1,2,\cdots,n_0/2\}$, $\mathcal{I}^c=\{n_0/2+1,\cdots,n_0\}$.

\State\textbf{Step 2} For each $k\in\{1,\cdots,K\}$, compute
\begin{equation*}
    \widehat{S}_k = \mathop{\mathrm{argmin}}_{S}\frac{1}{n_k+n_0/2}\left(\Vert \mathcal{Y}_{0,\mathcal{I}} - \mathcal{X}_{0,\mathcal{I}}(\widehat{L} + S) \Vert_2^2 + \Vert \mathcal{Y}_k - \mathcal{X}_k(\widehat{L} + S) \Vert_2^2\right) + \lambda_k\Vert S \Vert_1
\end{equation*}
\State\textbf{Step 3} For k = 0, compute
\begin{equation*}
\begin{split}
    &\widehat{S}_{0,\mathcal{I}} = \mathop{\mathrm{argmin}}_S\frac{1}{n_0/2}\Vert \mathcal{Y}_{0,{\mathcal{I}}} - \mathcal{X}_{0,\mathcal{I}}(\widehat{L} + S) \Vert_2^2 + \lambda_0 \Vert S \Vert_1\\
    &\widehat{S}_{0,\mathcal{I}^c} = \mathop{\mathrm{argmin}}_{S}\frac{1}{n_0/2}\Vert \mathcal{Y}_{0,\mathcal{I}^c} - \mathcal{X}_{0,\mathcal{I}^c}(\widehat{L} + S) \Vert_2^2 + \lambda_0 \Vert S \Vert_1
\end{split}
\end{equation*}

\State\textbf{Step 4}
For $k\in\{1,\cdots,K\}$, $ R^{(k)} = \Vert \mathcal{Y}_{0,\mathcal{I}^c}- \mathcal{X}_{0,\mathcal{I}^c} (\widehat{L} + \widehat{S}_k) \Vert_2^2$.\\
\quad \quad \quad For $k=0$, $R_1^{(0)}= \Vert \mathcal{Y}_{0,\mathcal{I}^c}- \mathcal{X}_{0,\mathcal{I}^c} (\widehat{L} +  \widehat{S}_{0,\mathcal{I}}) \Vert_2^2$ and $R_2^{(0)}= \Vert \mathcal{Y}_{0,\mathcal{I}^c}- \mathcal{X}_{0,\mathcal{I}^c} (\widehat{L} + \widehat{S}_{0,\mathcal{I}^c}) \Vert_2^2$

\State\textbf{Step 5}
$\widehat{\mathcal{A}} = \{k\in\{1,2,\cdots,K\}: R^{(k)} - R_1^{(0)}\leq c |R_1^{(0)}- R_2^{(0)}| \}$

\end{algorithmic}
\end{algorithm}

\subsection{Inference for Model Parameter Algorithm}
The first step is debiasing $\widehat{S}_{tran}$. The explicit form of debiased estimator is as below,
\begin{equation*}
   \widehat{S}^{on} = \widehat{S}_{tran} + \frac{1}{n_0}\sum_{i=1}^{n_0}M_i X_i^{(0)}(X_{i+1}^{(0)} - X_i^{(0)}(\widehat{L} + \widehat{S}_{tran}))^{'}.
\end{equation*}

$M_i$ is called debiasing matrix and needs to be estimated by target model. If observations are i.i.d, setting $M_1=M_2\cdots=M_{n_0}$ is an effective way to debias $\widehat{L}$ \citep{javanmard2014confidence}. However, for VAR model, the existence of dependency destroy the asymptotic normality. To fix this problem, \cite{deshpande2021online} estimate $M_i$ by the past observations, $\{X_t\}_{t< i}$, which makes $M_i$ predictable. The term ``online" comes from the imposed predictability.

We next introduce the first step specifically. We split processed target data $\{X_t\}_{t=1}^{n_0}$ into $\ell$ segments
\begin{equation*}
    \{X_t^{(0)}\}_{t\in E_0},\ \{X_t^{(0)}\}_{t\in E_1},\ \cdots,\ \{X_t^{(0)}\}_{t\in E_{\ell-1}}
\end{equation*}
where $E_i:=\{m_i+1,m_i+2,\cdots,m_{i+1}\}$ and $0=m_0< m_1<\cdots< m_\ell = n_0$. Define the length of $E_i$ as $r_i$, $r_i := m_{i+1}-m_i$. Define
the sample covariance of the observations in the first $j$ segments,
\begin{equation*}
\widehat{\Sigma}^{(j)}:= \frac{1}{m_j} \sum_{t \in E_0\cup\cdots\cup E_{j-1}}X_t^{(0)}(X_t^{(0)})^{'}, \quad j= 1,\cdots\ell
\end{equation*}
The matrix $M^{(j)}=M_j$ is a $p\times p$ matrix. The $a$-th row of $M^{(j)}$ denoted as $\mathbf{m}_a^{(j)}$, $a=1,\cdots p$, is the solution of the following optimization:
\begin{equation*}
\begin{split}
&\text{minimize}\quad \mathbf{m}_a^{(j)}\widehat{\Sigma}^{(j)}(\mathbf{m}_a^{(j)})^{'}\\
&\text{subject to} \quad \Vert \widehat{\Sigma}^{(j)}(m_a^{(j)})^{'} - e_a\Vert_\infty \leq \mu_j, \quad \Vert m_a^{(j)}\Vert_1\leq C
\end{split}
\end{equation*}
for appropriate values of $\mu_j, C > 0$. Then constructing $\widehat{S}^{on}$ as below,
\begin{equation}\label{sec3_eq6}
   \widehat{S}^{on} = \widehat{S}_{tran} + \frac{1}{n_0}\sum_{j=1}^{\ell}\sum_{t\in E_j}M^{(j)}X_t^{(0)}(X_{t+1}^{(0)} - X_t^{(0)}(\widehat{L} + \widehat{S}_{tran}))^{'}
\end{equation}
$r_0$ is selected to be $\sqrt{n_0}$ and $r_j \asymp \alpha^{j}$ for some constant $\alpha>1$.

The second step is constructing the variance of $\widehat{S}^{on}$. Corollary 3.7 in \cite{deshpande2021online} shows that the $(a,i)$-th entry of $\widehat{S}^{on}$ has the following variance
\begin{equation*}
V_{a,i} = \frac{\left(\Sigma_\epsilon^{(0)}\right)_{i,i}}{n_0}\sum_{j=1}^{\ell}\sum_{t\in E_j}(\mathbf{m}_a^{(j)}X_t^{(0)})^2
\end{equation*}
where $a \in \{1,2\cdots,p\}$ and $i \in \{1,2,\cdots,p\}$. Therefore, we can use the scaled residual $\sqrt{n_0}(\widehat{S}_{a,i}^{on} - S_{a,i}^{(0)})/\sqrt{V_{a,i}}$ as the test statistics and construct entry-wise confidence intervals accordingly. Theorem is given as below.

\begin{Corollary}\label{corollary1}
Let $\widehat{S}^{on}$ be the debiased estimator \eqref{sec3_eq6} with $\mu_j = C_1 \omega \sqrt{\log p/m_j}$ and $\widehat{S}$ derived from Algorithm \ref{algorithm1}. Let $C = C_0\Vert \Omega \Vert_1$ for an arbitrary constant $C_0 \geq 1$, where $\Omega = (\mathbb{E}[X^{(0)}_t(X^{(0)}_t)^{'}])^{-1}$. For any fixed sequence of integers $a(n) \in \{1,\cdots,p\}$, define the conditional variance $V_n$ as
\begin{equation}\label{sec3_eq7}
V_{a,i} = \frac{\left(\Sigma_\epsilon^{(0)}\right)_{i,i}}{n_0}\sum_{j=1}^{\ell}\sum_{t\in E_j}(\mathbf{m}_a^{(j)}X_t^{(1)})^2 
\end{equation}
Assume that $\Vert \Omega\Vert_1 = o(\sqrt{n_0}/\log p)$, $h = o(s\sqrt{\log(p)/n_0})$, $\frac{n_0(s\log p + rp)}{N} = o(1)$. For any fixed coordinate $a \in\{1,\cdots,p\}, \ i\in\{1,2\cdots,p\}$ and for all $x\in \mathbb{R}$, we have
\begin{equation*}
    \mathop{\mathrm{lim}}_{n\rightarrow \infty} \left|\mathbb{P}\{ \frac{\sqrt{n_0}(\widehat{S}_{a,i}^{on} - S_{a,i}^{(0)})}{\sqrt{V_{a,i}}}\leq x \} - \Phi(x)\right| = 0,
\end{equation*}
where $\Phi$ is the standard Gaussian cdf.
\end{Corollary}
\begin{remark}
For getting the conditional variance $V_n$, Corollary requires the knowledge of $\Sigma_{\epsilon}^{(0)}$. Typically, $\Sigma_{\epsilon}^{(0)}$ is estimated by the training error. Since transfer learning improves the estimation accuracy and lowers the training error, confidence intervals constructed with transfer learning are always narrower than confidence intervals constructed with typical lasso method without losing any confidence. More detailed comparison is provided in simulation experiments section. 
\end{remark}

\begin{proof}[Proof of Corollary \ref{corollary1}]
The proof is similar to Corollary 3.7 in \cite{deshpande2021online}, so we omit some details. We rewrite \eqref{sec3_eq6} as
\begin{equation*}
\begin{split}
\widehat{S}^{on} - S_0 
&= \underbrace{\left(I - \frac{1}{n_0}\sum_{j=1}^{\ell}\sum_{t\in E_j}M^{(j)}X^{(0)}_t (X^{(0)}_t)^{'}\right)(\widehat{S} - S_0)}_{\Delta}
+ 
\underbrace{\frac{1}{n_0}\sum_{j=1}^{\ell}\sum_{t\in E_j}M^{(j)}X_t^{(0)} (\epsilon_t^{(0)})^{'}}_{W}\\
&+
\underbrace{\frac{1}{n_0}\sum_{j=1}^{\ell}\sum_{t\in E_j}M^{(j)}X_t^{(0)}(X_t^{(0)})^{'} (\widehat{L} - L)}_{U}
\end{split}
\end{equation*}
From Theorem 3.4 in \cite{deshpande2021online}, we have that $\mathbb{P}(\Vert \Delta\Vert_{\infty}\geq C s\frac{\log(p)}{n_0})\leq (p)^{-4}$ which implies that $\sqrt{n_0}\Delta$ is negligible with high probability. Lemma 3.6 in \cite{deshpande2021online} shows that $\sqrt{n_0}W$ is a martingale with variance equal to $V_{a,i}$. $U$ can be decomposed as
\begin{equation*}
\underbrace{\frac{1}{n_0}\sum_{j=1}^{\ell}\sum_{t\in E_j}M^{(j)}X_t^{(0)}(X_t^{(0)})^{'} (\widehat{L} - L)}_{U}
=
\underbrace{ \left(\frac{1}{n_0}\sum_{j=1}^{\ell}\sum_{t\in E_j}M^{(j)}X_t^{(0)}(X_t^{(0)})^{'} - I \right) (\widehat{L} - L) }_{U_1} + \underbrace{\Bigg(\widehat{L} - L\Bigg)}_{U_2}
\end{equation*}
Similar to $\Delta$, we have that $U_1 \leq \sqrt{\frac{\log p}{n_0}} \sqrt{\frac{s\log p + rp}{N}}$ with high probability. We also have $\Vert U_2\Vert_{\infty} \leq \sqrt{\frac{s\log p + rp}{N}}$, Thus $\Vert \sqrt{n_0}U\Vert_{\infty} \leq \sqrt{\frac{n_0(s\log p + rp)}{N}} = o(1)$.
The final result derives from Martingale central limit Theorem in \cite{hall2014martingale}, i.e. Corollary 3.2.
\end{proof}

%\newpage
\section{Additional Numerical Results and Considerations}\label{sec:5}
\subsection{Figure for Section 4}
\begin{figure}[h]
    \centering
    \includegraphics[width=4in]{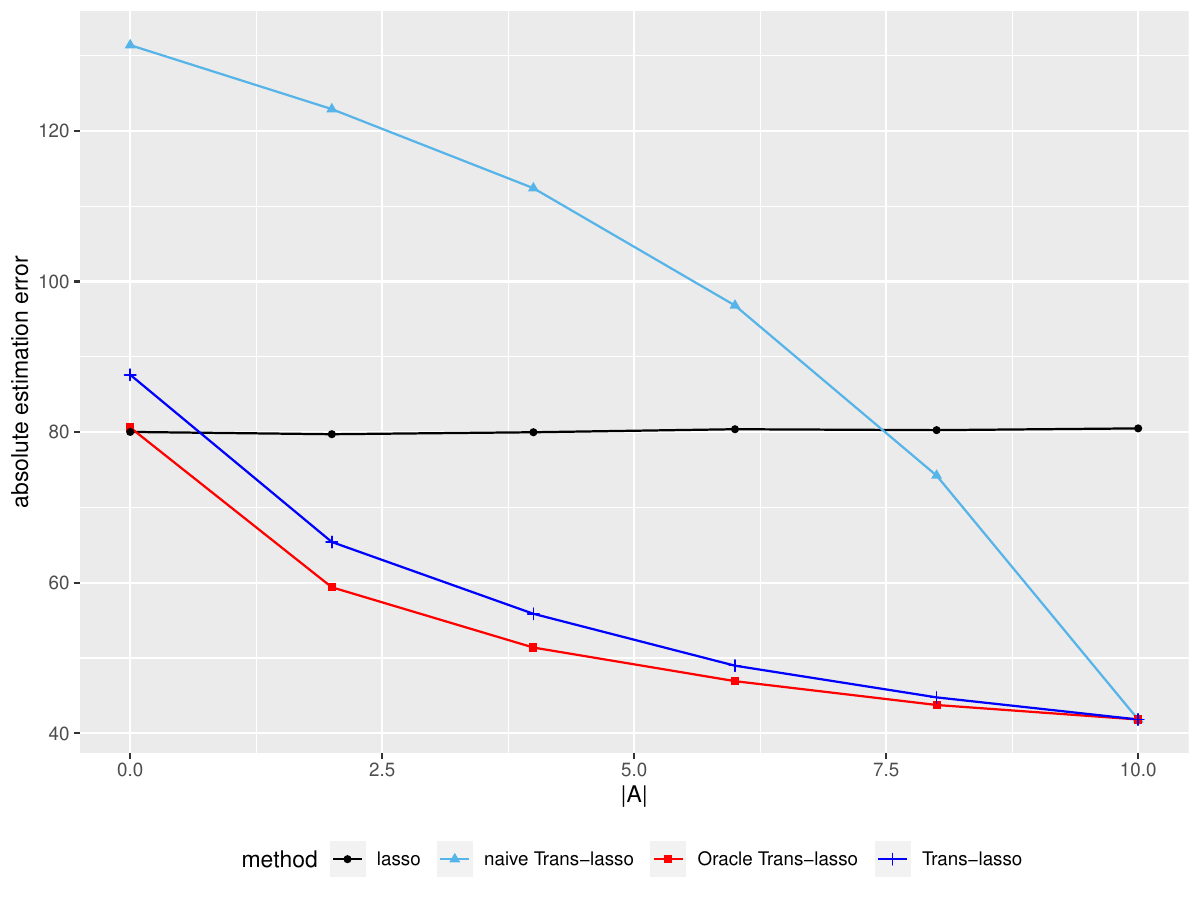}
    \caption{Absolute Estimation Error for Sparse Matrix}
    \label{simu2_fig1}
\end{figure}

\newpage
\subsection{Figure for Section 5}
\begin{figure}[h]
    \begin{tabular}{cccc}
        \centering
        \subfigure[T = 0(Video start)]{
        \begin{minipage}[t]{0.21\linewidth}
        \centering
        \includegraphics[width=1.3in]{pic/realdata/Meet_WalkTogether2000_c.jpg}
        %\caption{fig1}
        \end{minipage}%
        }&\subfigure[T=115(First man walk out of lobby)]{
        \begin{minipage}[t]{0.21\linewidth}
        \centering
        \includegraphics[width=1.3in]{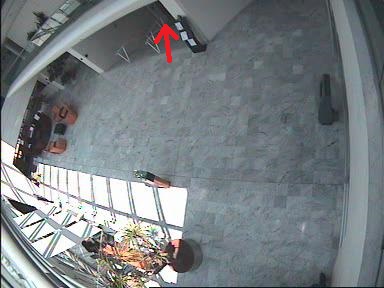}
        %\caption{fig1}
        \end{minipage}%
        }&\subfigure[T = 173(Two man walk in lobby)]{
        \begin{minipage}[t]{0.21\linewidth}
        \centering
        \includegraphics[width=1.3in]{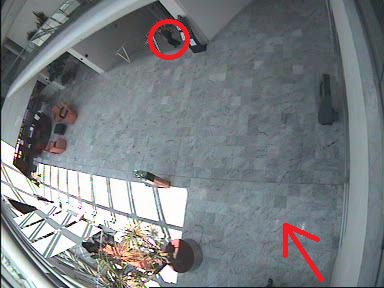}
        %\caption{fig1}
        \end{minipage}%
        }&\subfigure[T=231(Two man walk in lobby)]{
        \begin{minipage}[t]{0.21\linewidth}
        \centering
        \includegraphics[width=1.3in]{pic/realdata/Meet_WalkTogether2231_c.jpg}
        %\caption{fig1}
        \end{minipage}%
        }
        \\
        \subfigure[T=289(Two man stand together)]{
        \begin{minipage}[t]{0.21\linewidth}
        \centering
        \includegraphics[width=1.3in]{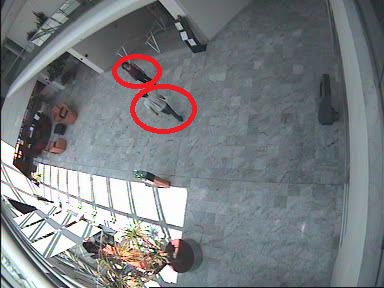}
        %\caption{fig1}
        \end{minipage}%
        }&\subfigure[T = 347(Two man stand together)]{
        \begin{minipage}[t]{0.21\linewidth}
        \centering
        \includegraphics[width=1.3in]{pic/realdata/Meet_WalkTogether2347_c.jpg}
        %\caption{fig1}
        \end{minipage}%
        }&\subfigure[T=405(Two man walk together)]{
        \begin{minipage}[t]{0.21\linewidth}
        \centering
        \includegraphics[width=1.3in]{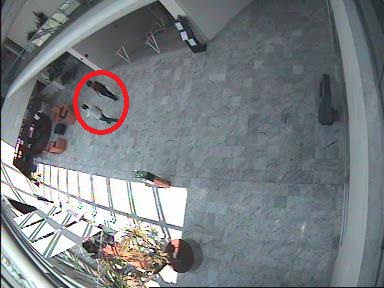}
        %\caption{fig1}
        \end{minipage}%
        }&\subfigure[T = 463(Two man walk out of lobby)]{
        \begin{minipage}[t]{0.21\linewidth}
        \centering
        \includegraphics[width=1.3in]{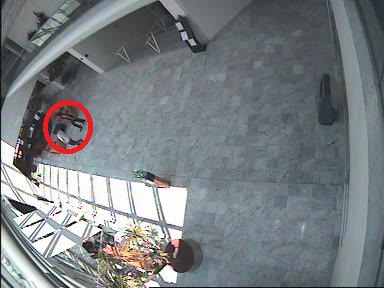}
        %\caption{fig1}
        \end{minipage}%
        }
        \\
        \subfigure[T = 521(Two man walk to the door)]{
        \begin{minipage}[t]{0.21\linewidth}
        \centering
        \includegraphics[width=1.3in]{pic/realdata/Meet_WalkTogether2521_c.jpg}
        %\caption{fig1}
        \end{minipage}%
        }&\subfigure[T = 579(Two man walk through the door)]{
        \begin{minipage}[t]{0.21\linewidth}
        \centering
        \includegraphics[width=1.3in]{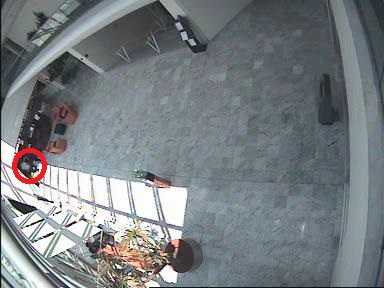}
        %\caption{fig1}
        \end{minipage}%
        }&\subfigure[T = 637(Two man exit)]{
        \begin{minipage}[t]{0.21\linewidth}
        \centering
        \includegraphics[width=1.3in]{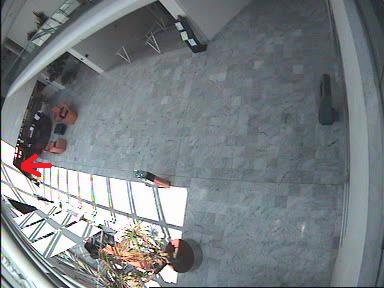}
        %\caption{fig1}
        \end{minipage}%
        }&\subfigure[T = 695(Empty lobby)]{
        \begin{minipage}[t]{0.21\linewidth}
        \centering
        \includegraphics[width=1.3in]{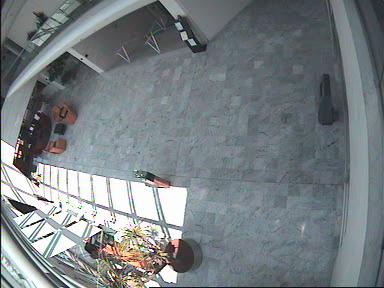}
        %\caption{fig1}
        \end{minipage}%
        }
    \end{tabular}
\caption{View of Footage}
\label{fig4}
\end{figure}

\begin{figure}[h]
    \begin{tabular}{cccc}
        \centering
        \subfigure[T = 0(Video start)]{
        \begin{minipage}[t]{0.21\linewidth}
        \centering
        \includegraphics[width=1.3in]{pic/realdata/segment1.pdf}
        %\caption{fig1}
        \end{minipage}%
        }&\subfigure[T=115(First man walk out of lobby)]{
        \begin{minipage}[t]{0.21\linewidth}
        \centering
        \includegraphics[width=1.3in]{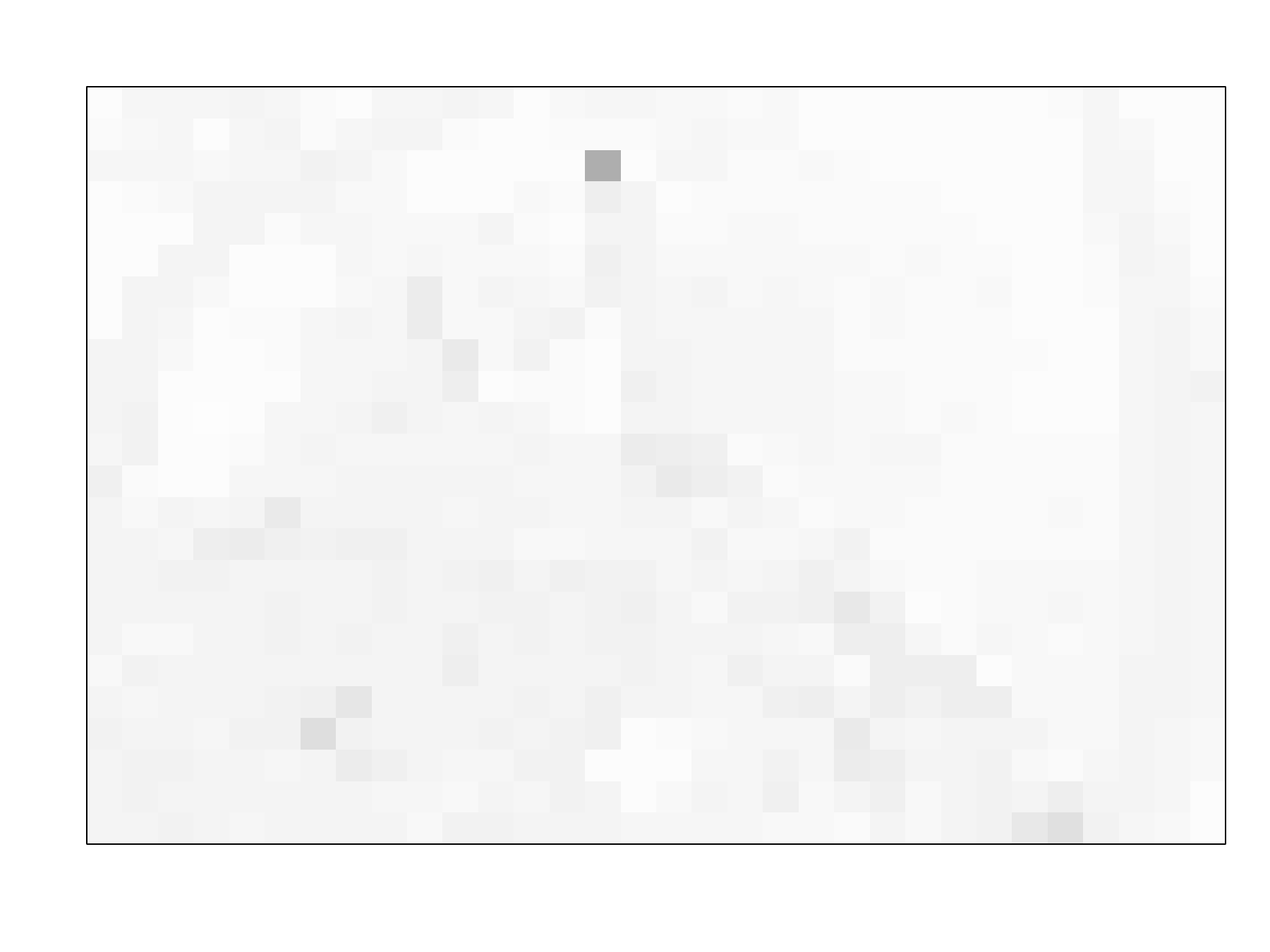}
        %\caption{fig1}
        \end{minipage}%
        }&\subfigure[T = 173(Two man walk in lobby)]{
        \begin{minipage}[t]{0.21\linewidth}
        \centering
        \includegraphics[width=1.3in]{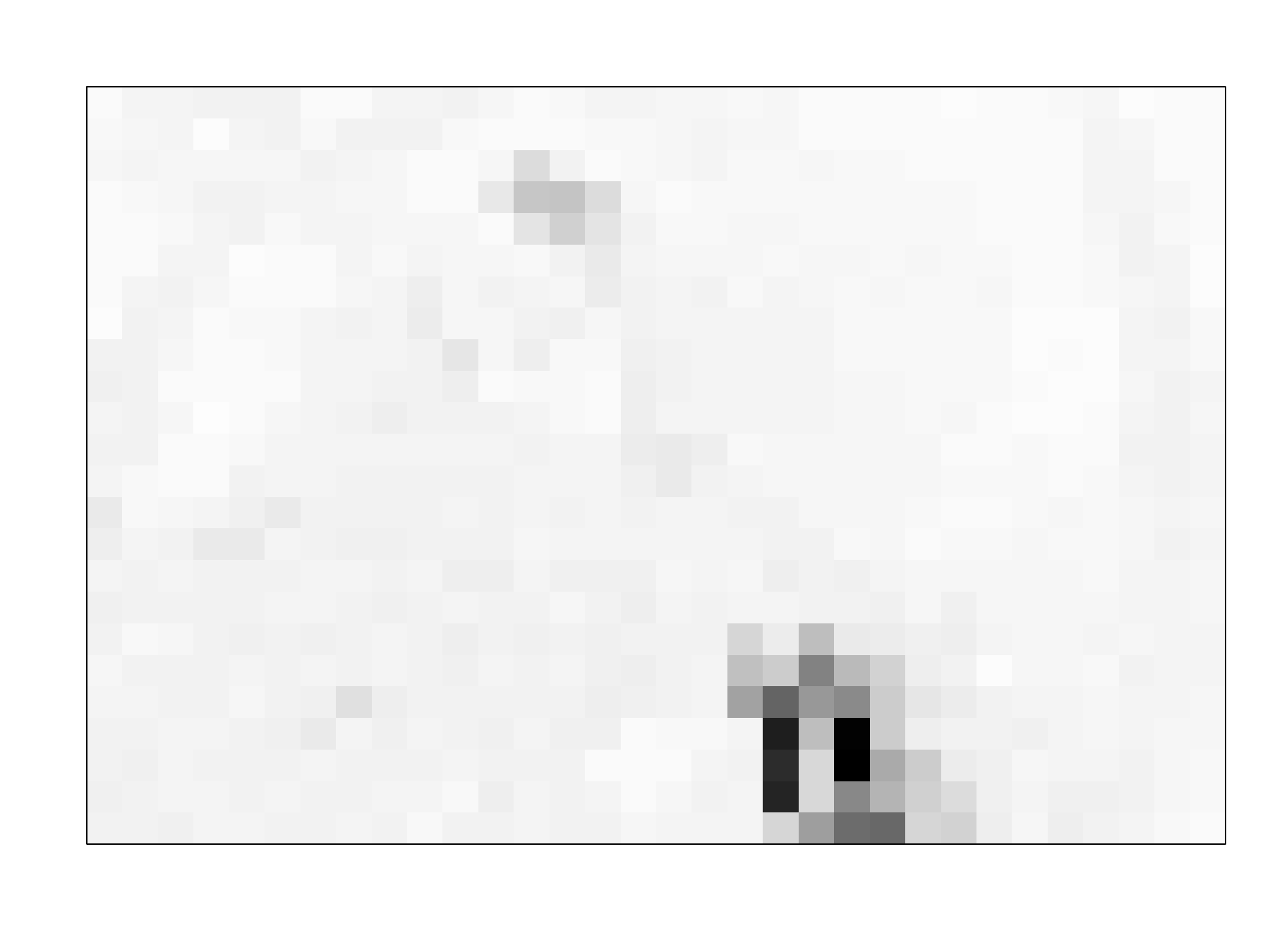}
        %\caption{fig1}
        \end{minipage}%
        }&\subfigure[T=231(Two man walk in lobby)]{
        \begin{minipage}[t]{0.21\linewidth}
        \centering
        \includegraphics[width=1.3in]{pic/realdata/segment4.pdf}
        %\caption{fig1}
        \end{minipage}%
        }
        \\
        \subfigure[T=289(Two man stand together)]{
        \begin{minipage}[t]{0.21\linewidth}
        \centering
        \includegraphics[width=1.3in]{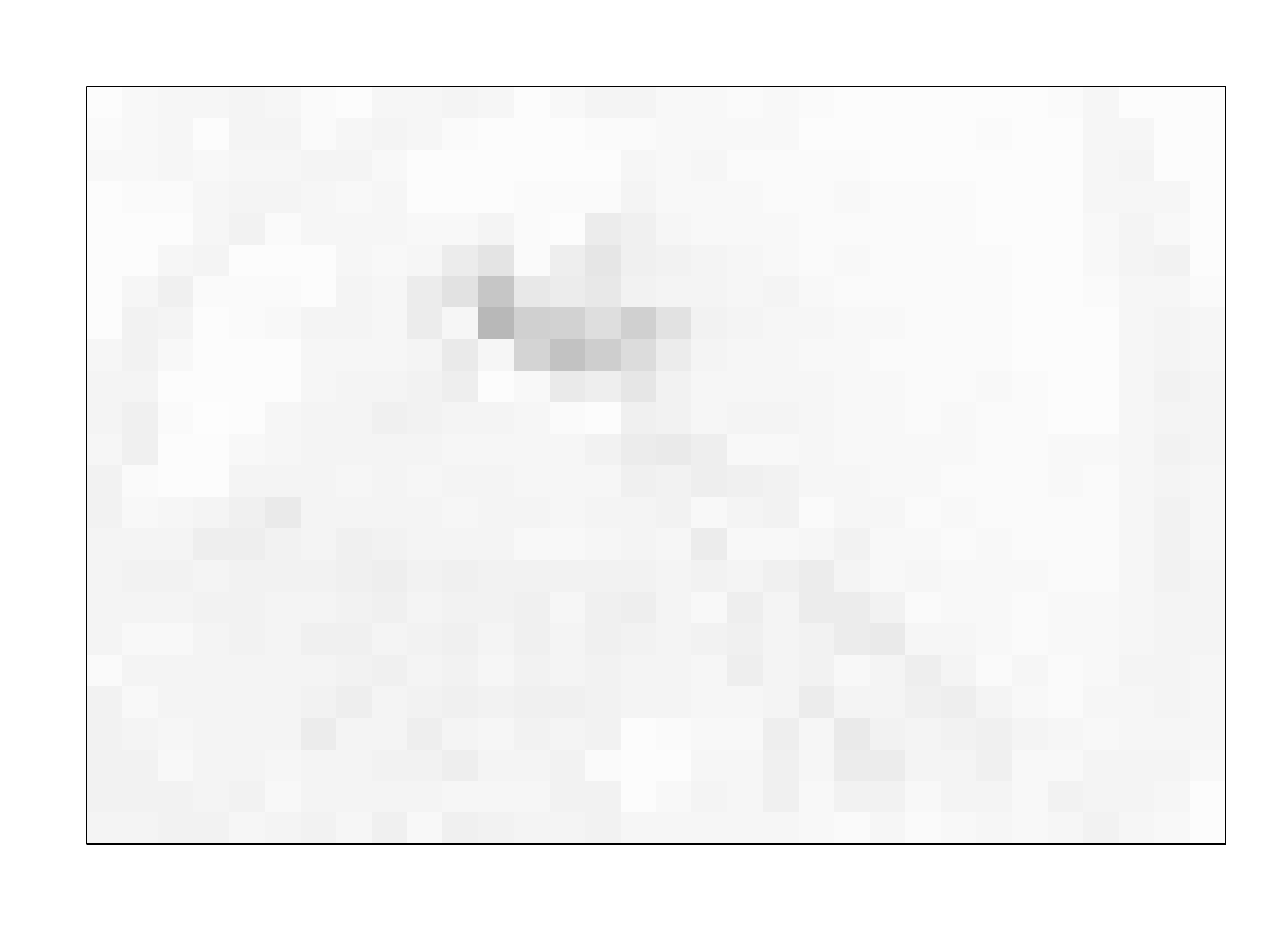}
        %\caption{fig1}
        \end{minipage}%
        }&\subfigure[T = 347(Two man stand together)]{
        \begin{minipage}[t]{0.21\linewidth}
        \centering
        \includegraphics[width=1.3in]{pic/realdata/segment6.pdf}
        %\caption{fig1}
        \end{minipage}%
        }&\subfigure[T=405(Two man walk together)]{
        \begin{minipage}[t]{0.21\linewidth}
        \centering
        \includegraphics[width=1.3in]{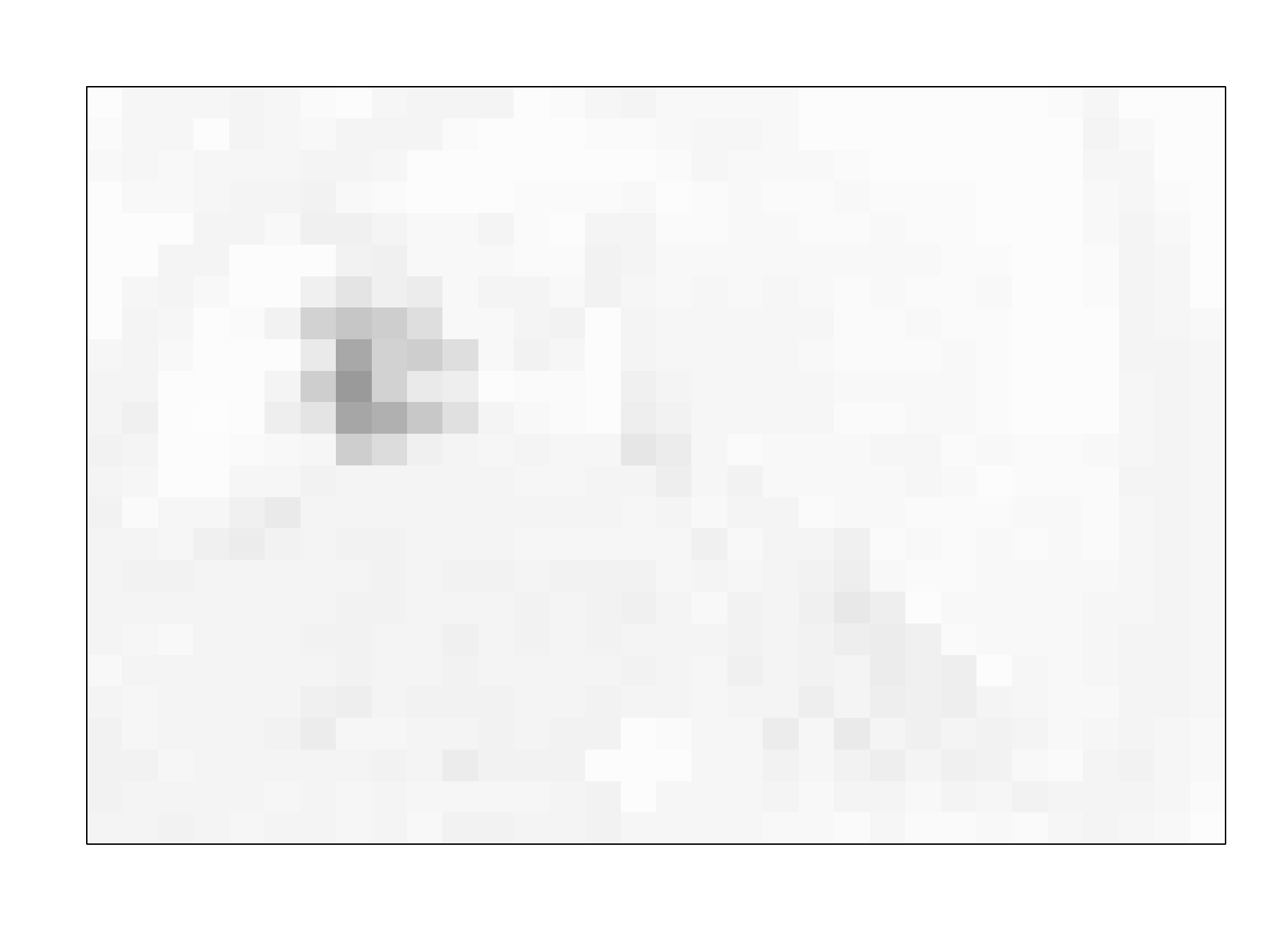}
        %\caption{fig1}
        \end{minipage}%
        }&\subfigure[T = 463(Two man walk out of lobby)]{
        \begin{minipage}[t]{0.21\linewidth}
        \centering
        \includegraphics[width=1.3in]{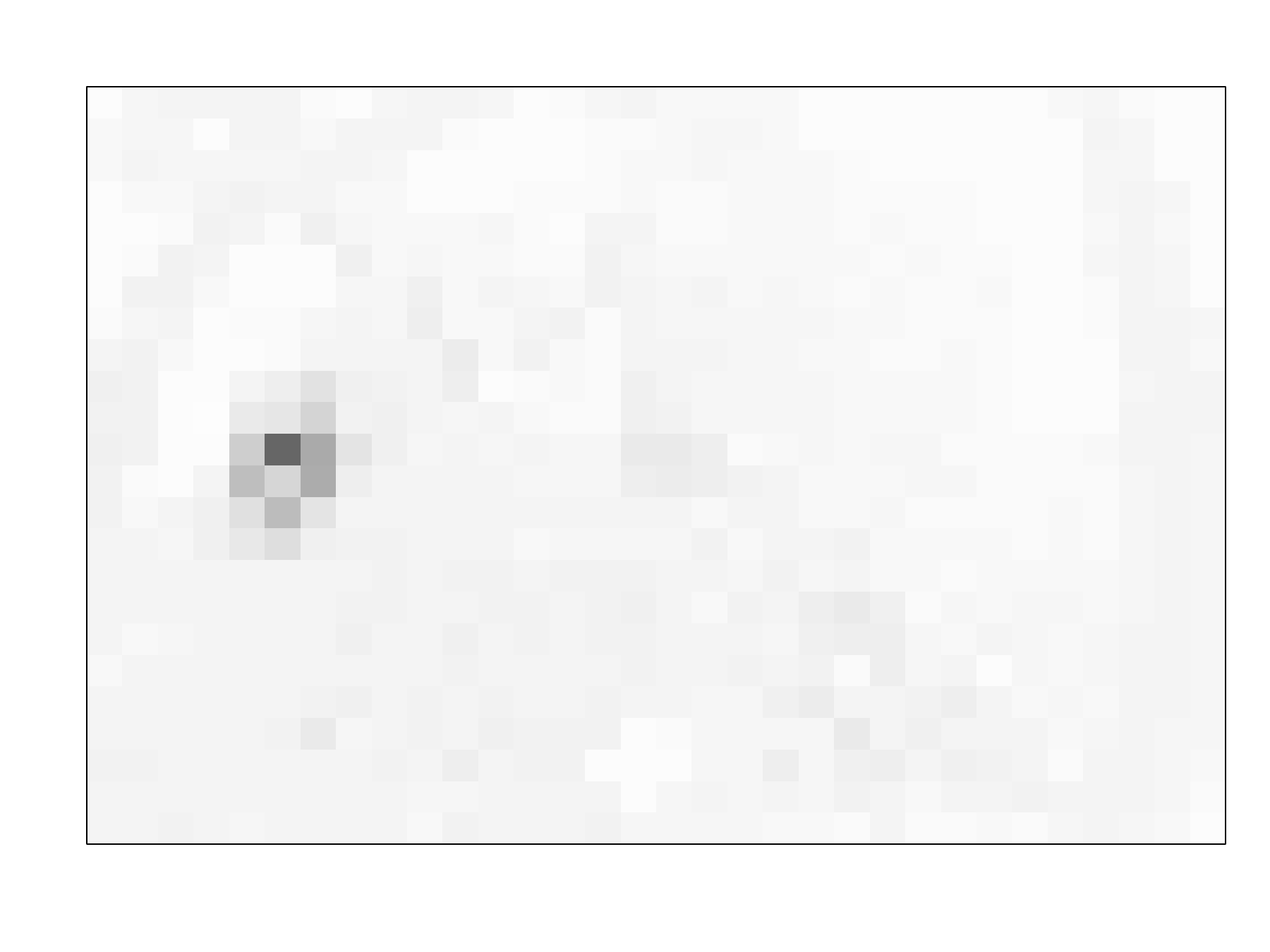}
        %\caption{fig1}
        \end{minipage}%
        }
        \\
        \subfigure[T = 521(Two man walk to the door)]{
        \begin{minipage}[t]{0.21\linewidth}
        \centering
        \includegraphics[width=1.3in]{pic/realdata/segment9.pdf}
        %\caption{fig1}
        \end{minipage}%
        }&\subfigure[T = 579(Two man walk through the door)]{
        \begin{minipage}[t]{0.21\linewidth}
        \centering
        \includegraphics[width=1.3in]{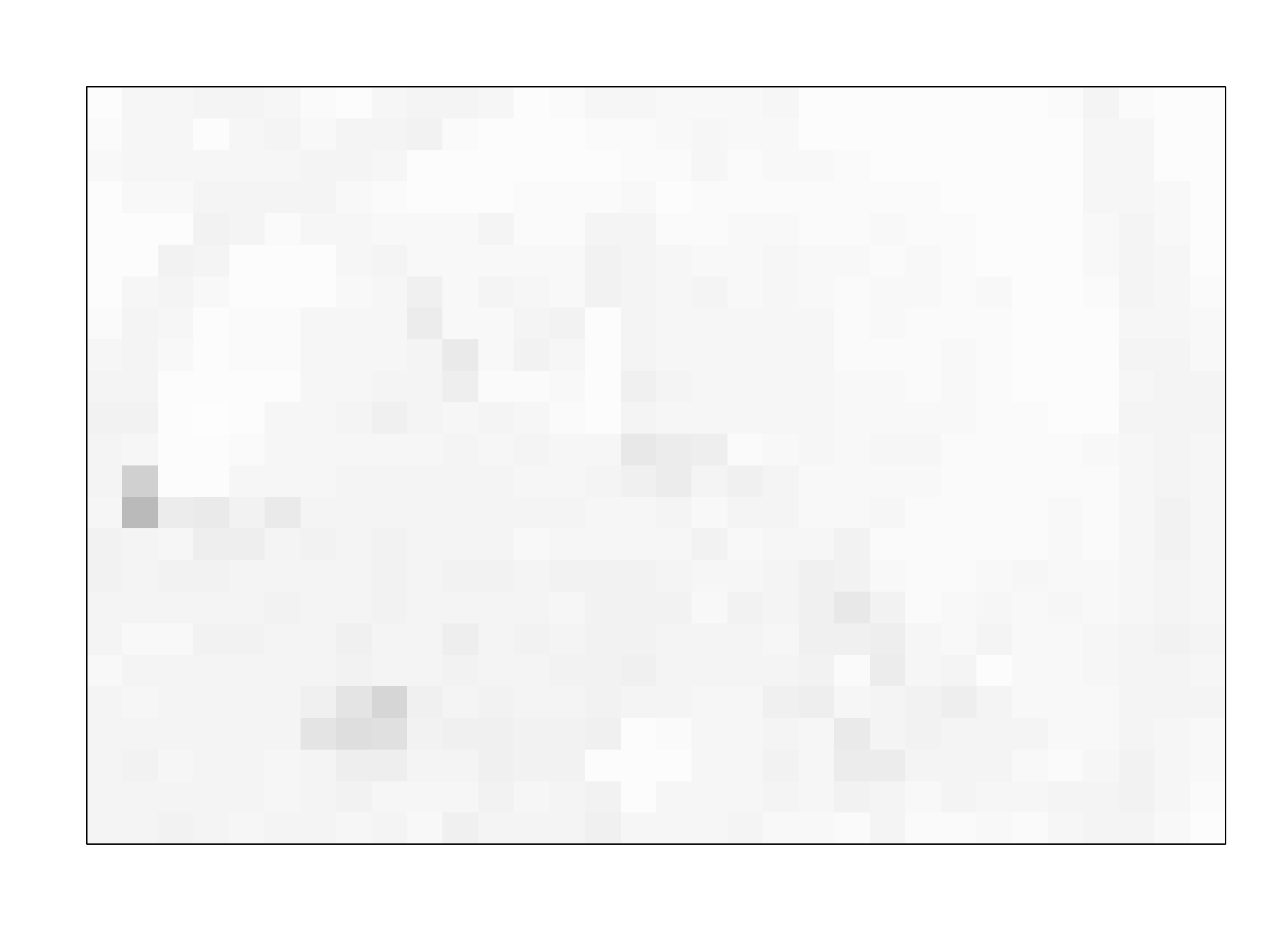}
        %\caption{fig1}
        \end{minipage}%
        }&\subfigure[T = 637(Two man exit)]{
        \begin{minipage}[t]{0.21\linewidth}
        \centering
        \includegraphics[width=1.3in]{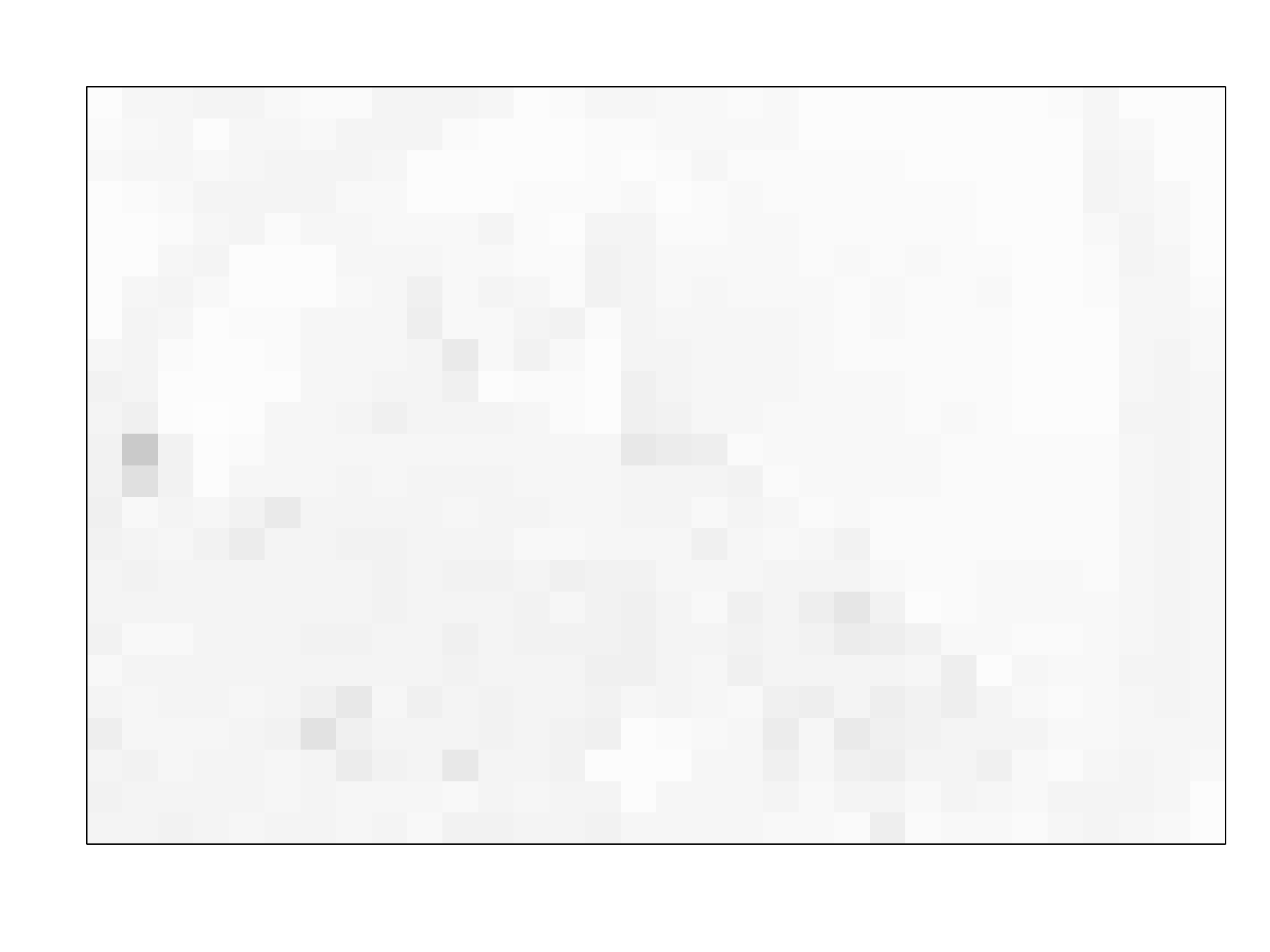}
        %\caption{fig1}
        \end{minipage}%
        }&\subfigure[T = 695(Empty lobby)]{
        \begin{minipage}[t]{0.21\linewidth}
        \centering
        \includegraphics[width=1.3in]{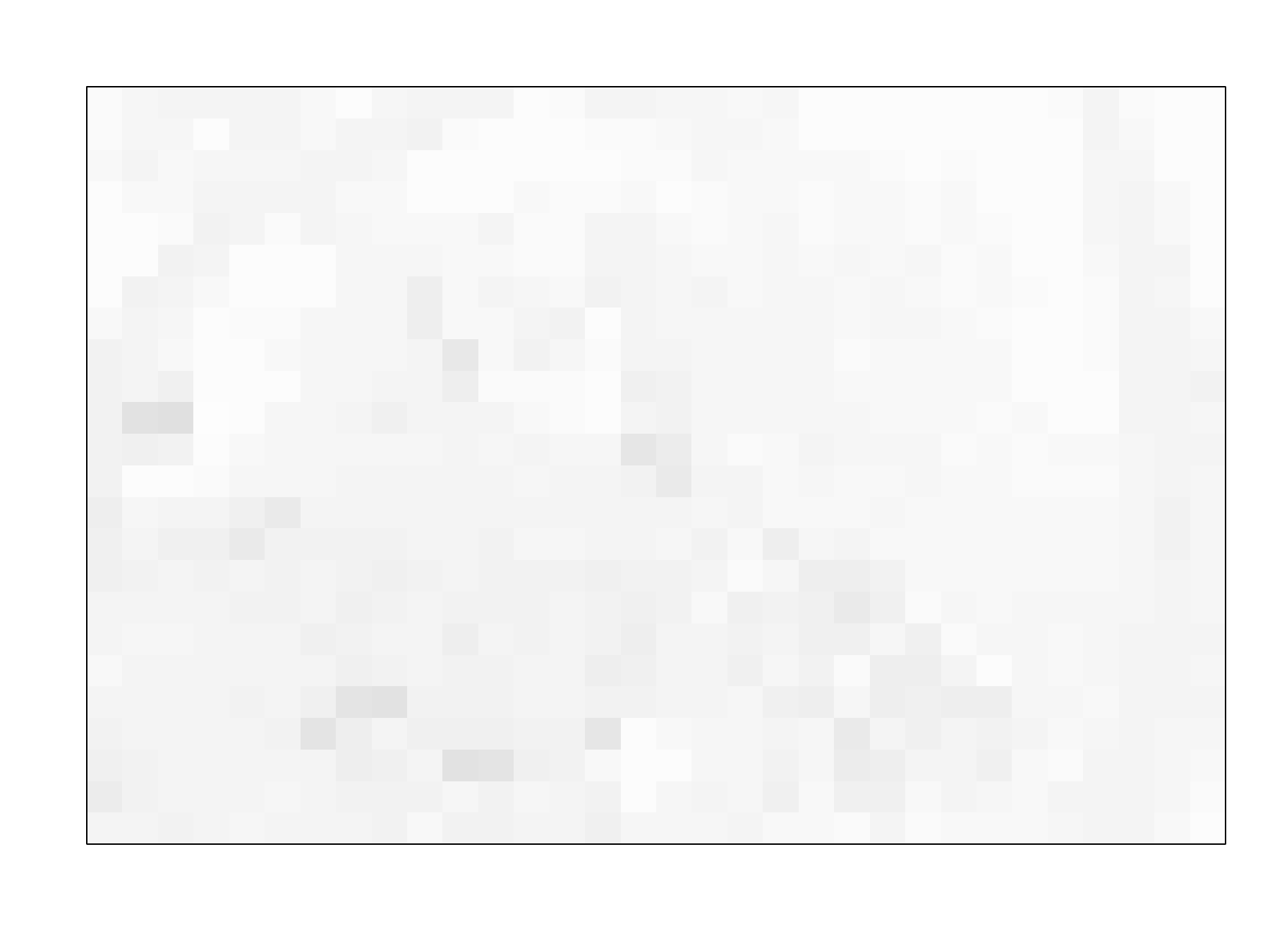}
        %\caption{fig1}
        \end{minipage}%
        }
    \end{tabular}
\caption{View of Footage}
\label{fig5}
\end{figure}

\newpage
\subsection{Additional Simulation Study}
In this simulation, we show the accuracy of recovering low rank
matrix in the first step of Algorithm \ref{algorithm1} with increasing dimension. Consider the VAR model
\begin{center}
Target model:
$
X^{(0)}_t = 
(L + S_0)
X^{(0)}_{t-1}
+
\epsilon^{
(0)}_{t}
$;
Auxiliary model:
$
X^{(k)}_t = 
(L + S_k)
X^{(k)}_{t-1}
+
\epsilon^{(k)}_{t}, \ 1\leq k\leq K,
$
    
\end{center}

where $S_0$ is one-off diagonal matrix having the following structure

\[
S_0=
\begin{bmatrix}
0  & 0.5  & 0   & \cdots  &  0 \\
0  & 0    & 0.5 & \cdots  &  0 \\
\vdots &\vdots &\vdots &\ddots & \vdots   \\
0  & 0    & 0   & \cdots  &  0.5 \\
0  & 0    & 0   & \cdots  &  0 
\end{bmatrix}_{p\times p}.
\]

$L$ is generated by $L = U  D  V^{'}$, where $D := diag(0.2, r)$ is a diagonal matrix. The dimension $p$ is set to 20, 50 and 100, respectively, while the rank $r$ is fixed to be 4 for all $p$. $S_k, k\geq 1$ is constructed by randomly replacing four entries of $S_0$. The sample size for each group is set to be 200, $n_0 = n_1=\cdots = n_K = 200$. Define $\mathcal{S}$ as the set of non-zero entries in $S_0$. Let $H$ be a random subset of $\{(i,j): 1\leq i\leq p, 1\leq j\leq p\}$ such that $|H|=4$. If $(i,j)\in H\cap \mathcal{S}$, we set $(S_k)_{ij} = 0$; if $(i,j)\in H\cap \mathcal{S}^{c}$, $(S_k)_{ij} = \eta_{ij}+0.5$, where $\eta_{ij}\sim b\cdot \mathrm{uniform}(\{+1, -1\})$. Size of auxiliary set is set to $0$, $1$, $\cdots$, $9$.

we utilize a grid search to select the optimal values of $\lambda$ and $\mu_i$. To simplify the selection procedure, we use the same $\mu_i$ for all groups. We compare estimation error of low-rank matrix for different values of $p$. The squared Frobenius norm error of estimation given by $\Vert L - \widehat{L}\Vert_F^2$ is shown in Table \ref{simu1_table1}. As we can see, our proposed algorithm gets better low-rank matrix estimator when more observations are provided. This result is consistent with Theorem \ref{th1}.

\begin{table}[h]
\caption{Squared Frobenius Norm Error of $L$}\label{simu1_table1}
    \centering
    \begin{tabular}{c|c|c|c|}
    \hline
        & $p=20$ & $p=50$ & $p=100$\\
        \hline
$K = 0$ & 0.639 & 0.639  & 0.639 \\
$K = 1$ & 0.616 & 0.631 & 0.636 \\
$K = 2$ & 0.546 & 0.612 & 0.628 \\
$K = 3$ & 0.468 & 0.579  & 0.611 \\
$K = 4$ & 0.414 & 0.529 & 0.597  \\
$K = 5$ & 0.368 & 0.484 &  0.578 \\
$K = 6$ & 0.320 & 0.451  & 0.555 \\
$K = 7$ & 0.294 & 0.429 & 0.533 \\
$K = 8$ & 0.274 & 0.407 & 0.508 \\
$K = 9$ & 0.260 & 0.389  & 0.486 \\
 \hline
    \end{tabular}
\end{table}

\subsection{Computation Time}

We provide average computation time for the proposed algorithm in Table~\ref{tab:my_label}.

    \begin{table}[h]
        \centering
        \begin{tabular}{c|c|c|c|c|c|}
        \hline
             $N$     &  200   &  400   &  600    &  800    &  1000  \\
             \hline
             $p=20$  &  0.1869 & 0.5492 &  1.115 &  2.040 &  3.351\\
             $p=40$  &  0.2011 & 0.5986 &  1.206 &  2.213 &  3.607\\
             $p=60$  &  0.2184 & 0.6418 &  1.314 &  2.378 &  3.848\\
             $p=80$  &  0.2284 & 0.6890 &  1.410 &  2.553 &  4.109\\
             $p=100$ &  0.2448 & 0.7413 &  1.499 &  2.722 &  4.381\\
             \hline
        \end{tabular}
        \caption{Computation time of algorithm 1. Each entry is the total time
of 200 replicates. The unit of computation time is second.  We use the same target model as simulation 1 in our paper and consider only one auxiliary model. The sample size of the target model is set to be 100 and sample size of auxiliary model take the value from $\{100,300,500,700,900\}$.}
        \label{tab:my_label}
    \end{table}

\subsection{Alternative Models for the Real Data}

In this part, for the real data analysis, we make a comparison with other parameterizations (results are provided in Table~\ref{table1}). As we can see from these results, the proposed algorithm coupled with low rank plus sparse model parameters outperforms all other competing methods.

    \begin{table}[h]
    \centering
    \begin{tabular}{c|c|c|c|c|}
    \hline
    & seg1 & seg2 & seg3 & seg4 \\
    \hline
Trans-lasso(L+S) & 7.205(0.015) & 0.133(0.003) & 2.177(0.012) & 0.468(0.005)  \\
Trans-lasso(S) & 7.415(0.015) & 0.289(0.004) & 3.398(0.014) & 0.673(0.006) \\
lasso(L+S) & 8.660(0.017) & 0.241(0.004) & 4.484(0.017) & 1.928(0.011)  \\
lasso(S) & 8.686(0.017)  & 0.257(0.004) & 4.579(0.017) & 1.954(0.011) \\
Low-rank & 14.090(0.021) & 1.955(0.005) & 6.854(0.018) & 5.008(0.014) \\ 
 \hline
   &  seg5 & seg6 & seg7 & seg8 \\
\hline
Trans-lasso(L+S) & 0.092(0.002) & 0.916(0.008) & 0.372(0.005) & 0.233(0.004) \\
Trans-lasso(S) & 0.251(0.004) & 1.164(0.009) & 0.920(0.008) & 0.503(0.006) \\
lasso(L+S)    & 0.450(0.005) & 2.686(0.013) & 1.653(0.010) & 1.235(0.009)\\
lasso(S) & 0.476(0.005) & 2.682(0.013) & 1.683(0.010) & 1.273(0.009) \\
Low-rank & 2.388(0.007) & 5.682(0.016) & 3.938(0.012) & 3.509(0.011) \\
 \hline
  & seg9 & seg10 & seg11 & seg12\\
 \hline
Trans-lasso(L+S) & 0.146(0.003) & 0.094(0.002) & 0.071(0.002) & 0.139(0.002)\\
Trans-lasso(S) & 0.297(0.004) & 0.209(0.003) & 0.163(0.003) & 0.212(0.002) \\
lasso(L+S)  & 0.523(0.006) & 0.189(0.004) & 0.102(0.004) & 0.153(0.002) \\
lasso(S) & 0.529(0.006) & 0.211(0.004) & 0.117(0.003) & 0.219(0.003) \\
Low-rank & 2.377(0.007) & 1.910(0.005) & 1.815(0.004) & 0.876(0.003) \\
\hline
    \end{tabular}
    \caption{Mean squared prediction error for each segment. Standard errors are shown in parentheses. Trans-lasso(L+S) and lasso(L+S) refer to methods that we model VAR(1) with low-rank plus sparse structure. Trans-lasso(S) and lasso(S) refer to methods that we model VAR(1) without low-rank component. Low-rank methods in the last row implies that we model VAR(1) with only low-rank component for each segment.}
    \label{table1}
\end{table}

\subsection{Hyperparameter selection}

Using cross-validation to select hyperparameters (tuning parameters) in our model is difficult due to the time series nature of the data. Also, it would be computationally demanding since each private sparse component corresponds to one hyperparameter. It seems difficult to compute all cases when dealing with large groups. There are three types of hyperparameters in the proposed algorithms. The first group is related to lasso penalty terms. For those, we use suggestions in the literature for tuning parameter selection \citep{li2020transfer}. Specifically, we set $\lambda_\beta = \sqrt{2\log p/(n_0+n_{\mathcal{A}_0})}$ and $\lambda_\delta = \sqrt{2\log p/(n_0)}$ . Second, there will be a tuning parameter for the low rank penalty. For that, we set $\mu = \tau * \sqrt{np}$. We performed several simulations with different $\tau$. Empirically speaking, our algorithm reaches a good result when $\tau \in (0.1, 2)$. We set $\tau = 0.1$ in all numerical analyses. Further, there will be an additional tuning parameter for the selection algorithm which is the constant $c$. We performed sensitivity analysis and given $c \in (0.01, 0.5)$, the algorithm always selects helpful groups for the next transfer learning step. We set $c = 0.01$ in all numerical analyses. Finally, for the lasso method, we set $\lambda = \sqrt{2\log p/(n_0)}$ while for the inference part we use $r_0 = \sqrt{n_0}$, $r_i = 2^{i}$ and $\mu_j = \sqrt{\frac{\log p}{2m_j}}$ for $i, j\geq 1$ \citep{deshpande2021online}.

% We use suggestions in \cite{li2020transfer} and \cite{deshpande2021online} for tuning parameter selection. Specifically, for Oracle Trans-lasso, we set $\lambda_\beta = \sqrt{2\log p/(n_0+n_{\mathcal{A}_0})}$ and $\lambda_\delta = \sqrt{2\log p/(n_0)}$ while for Trans-Lasso, we set $c=0.01$ in Algorithm \ref{algorithm2}. For the lasso method, we set $\lambda = \sqrt{2\log p/(n_0)}$. For inference part we use $r_0 = \sqrt{n_0}$, $r_i = 2^{i}$ and $\mu_j = \sqrt{\frac{\log p}{2m_j}}$ for $i, j\geq 1$.

\section{Additional Details on Numerical Experiments}\label{sec:6}
\textbf{Computer Information}:\\
Processor:   Intel(R) Core(TM) i7-9750H CPU @ 2.60GHz   2.59 GHz\\
Installed RAM: 16.0 GB (15.9 GB usable)\\
System type: Windows 10 Home 64-bit operating system, x64-based processor\\
Time of execution: 3h(one computer)

\end{document}